\documentclass[nohyperref]{article}

\usepackage{microtype}
\usepackage{graphicx}
\usepackage{subfigure}
\usepackage{booktabs} 
\usepackage{bbm}

\usepackage{hyperref}

\usepackage[accepted]{icml2022}

\usepackage{amsmath}
\usepackage{amssymb}
\usepackage{mathtools}
\usepackage{amsthm}

\usepackage[capitalize,noabbrev]{cleveref}

\theoremstyle{plain}
\newtheorem{theorem}{Theorem}[section]
\newtheorem{proposition}[theorem]{Proposition}
\newtheorem{lemma}[theorem]{Lemma}
\newtheorem{corollary}[theorem]{Corollary}
\theoremstyle{definition}

\theoremstyle{remark}
\newtheorem{remark}[theorem]{Remark}

\usepackage[textsize=tiny]{todonotes}

\usepackage{thm-restate}

\newtheoremstyle{theoremdd}
  {\topsep}
  {\topsep}
  {\itshape}
  {0pt}
  {\bfseries}
  {. }
  { }
  {\thmname{#1}\thmnumber{ #2}\textnormal{\thmnote{ (#3)}}}
\theoremstyle{theoremdd}

\newcommand\Exp{\mathbb{E}}
\newcommand\R{\mathbb{R}}
\newcommand\Ls{L}

\DeclareMathOperator{\tr}{Tr}
\DeclareMathOperator{\KL}{KL}
\usepackage{dblfloatfix}

\newcommand{\Cov}{\operatorname{cov}}

\usepackage{enumitem}

\icmltitlerunning{Anticorrelated Noise Injection for Improved Generalization}

\begin{document}

\twocolumn[
\icmltitle{Anticorrelated Noise Injection for Improved Generalization}



\icmlsetsymbol{equal}{*}

\begin{icmlauthorlist}
\icmlauthor{Antonio Orvieto}{equal,xxx}
\icmlauthor{Hans Kersting}{equal,yyy}
\icmlauthor{Frank Proske}{zzz}
\icmlauthor{Francis Bach}{yyy}
\icmlauthor{Aurelien Lucchi}{www}
\end{icmlauthorlist}

\icmlaffiliation{yyy}{INRIA -- Ecole Normale Sup\'erieure -- PSL Research University, Paris, France}
\icmlaffiliation{zzz}{Department of Mathematics, University of Oslo, Norway}
\icmlaffiliation{xxx}{Department of Computer Science, ETH Zurich, Switzerland}
\icmlaffiliation{www}{Department of Mathematics and Computer Science, University of Basel, Switzerland}

\icmlcorrespondingauthor{Antonio Orvieto}{antonio.orvieto@inf.ethz.ch}
\icmlcorrespondingauthor{Hans Kersting}{hans.kersting@inria.fr}

\icmlkeywords{Machine Learning, ICML}

\vskip 0.3in
]



\printAffiliationsAndNotice{\icmlEqualContribution} 

\begin{abstract} 
Injecting artificial noise into gradient descent (GD) is commonly employed to improve the performance of machine learning models.
Usually, uncorrelated noise is used in such \emph{perturbed gradient descent} (PGD) methods.
It is, however, not known if this is optimal or whether other types of noise could provide better generalization performance.
In this paper, we zoom in on the problem of \emph{correlating} the perturbations of consecutive PGD steps.
We consider a variety of objective functions for which we find that GD with \emph{anticorrelated} perturbations (``Anti-PGD'') generalizes significantly better than GD and standard (uncorrelated) PGD. 
To support these experimental findings, we also derive a theoretical analysis that demonstrates that Anti-PGD moves to wider minima, while GD and PGD remain stuck in suboptimal regions or even diverge. 
This new connection between anticorrelated noise and generalization opens the field to novel ways to exploit noise for training machine learning models.
\end{abstract}

\begin{figure}[ht]
    \centering
    \vspace{-2mm}
    \includegraphics[width=0.94\linewidth]{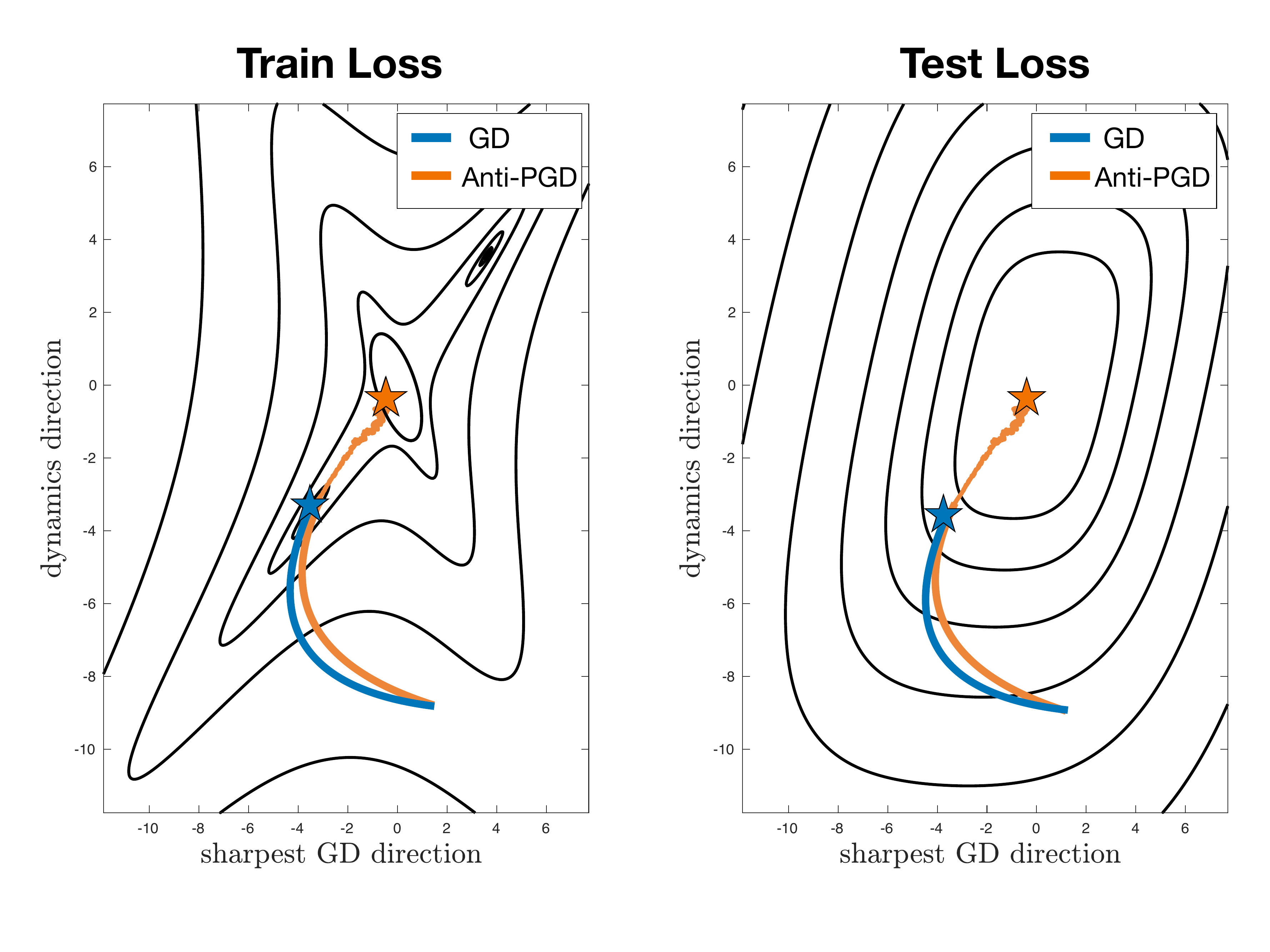}
    \vspace{-5mm}
    \caption{\small GD and Anti-PGD~(GD with anticorrelated noise injection) on a quadratically parametrized model~(details in \S\ref{sec:exp}) in $100$ dimensions, with only $5$ data points. The projection along two relevant directions is plotted. The train loss is most stable to sampling artifacts at flat minima. Indeed, the spurious minima on the south-west and north-east (left plot) are sharp. Flat minima often yield lower test losses. We prove in Thms.~\ref{thm:implicit_reg}~\&~\ref{thm:main_tube} that Anti-PGD is biased towards convergence to flat minima.}
    \label{fig:page1}
    \vspace{-3mm}
\end{figure}

\begin{figure*}[ht]
    \centering
    \includegraphics[height=0.25\textwidth]{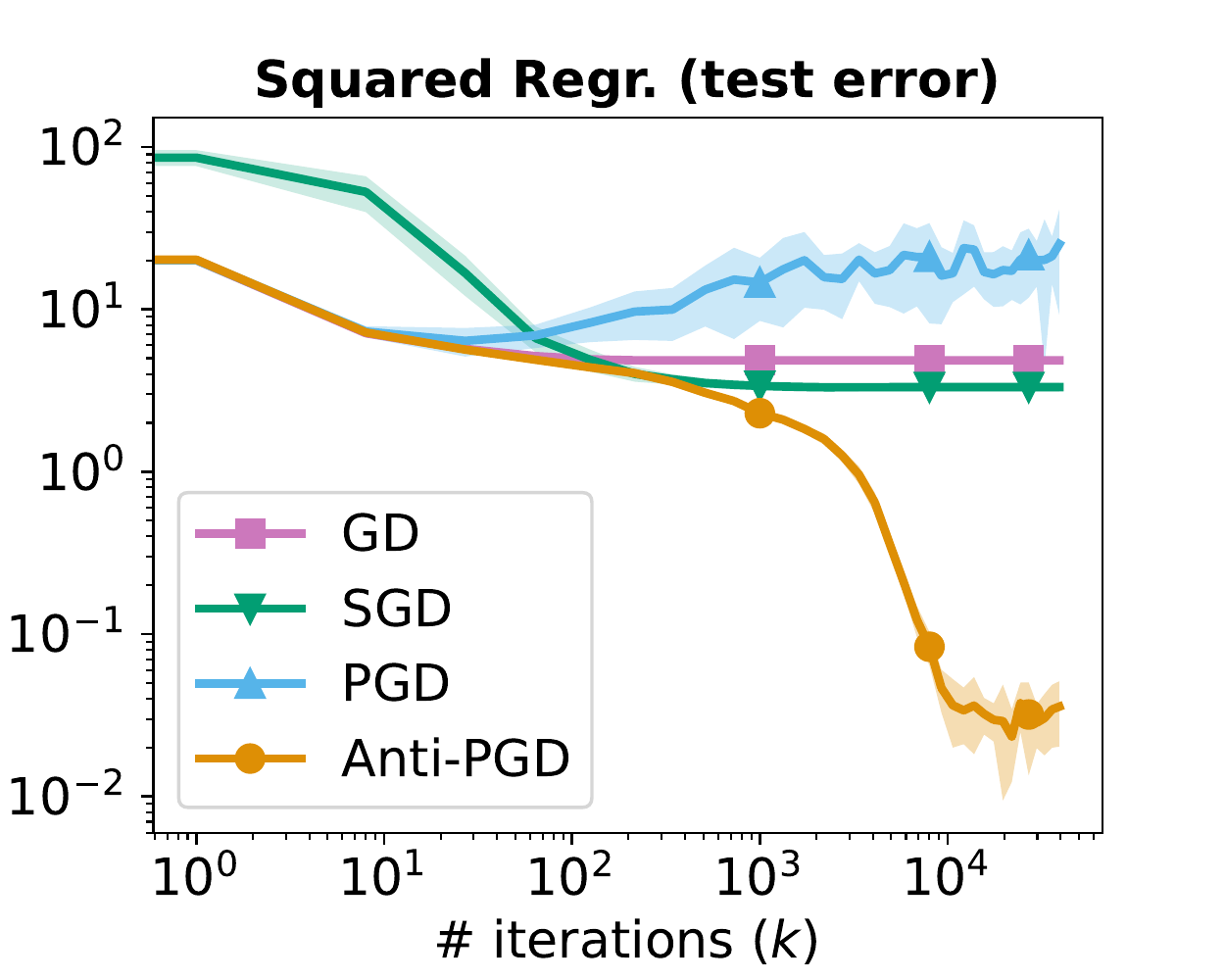}
    \includegraphics[height=0.25\textwidth]{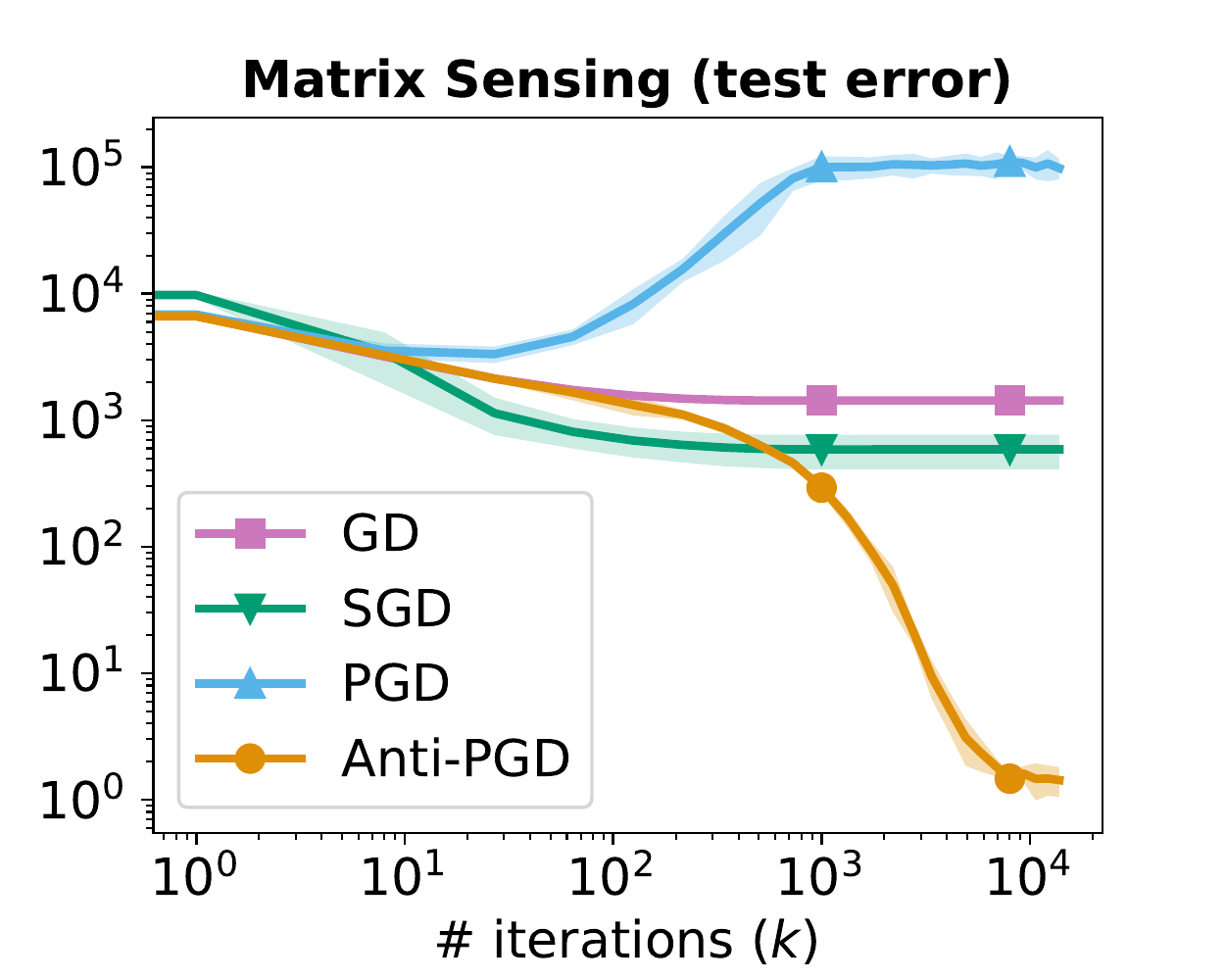}
    \includegraphics[height=0.25\textwidth]{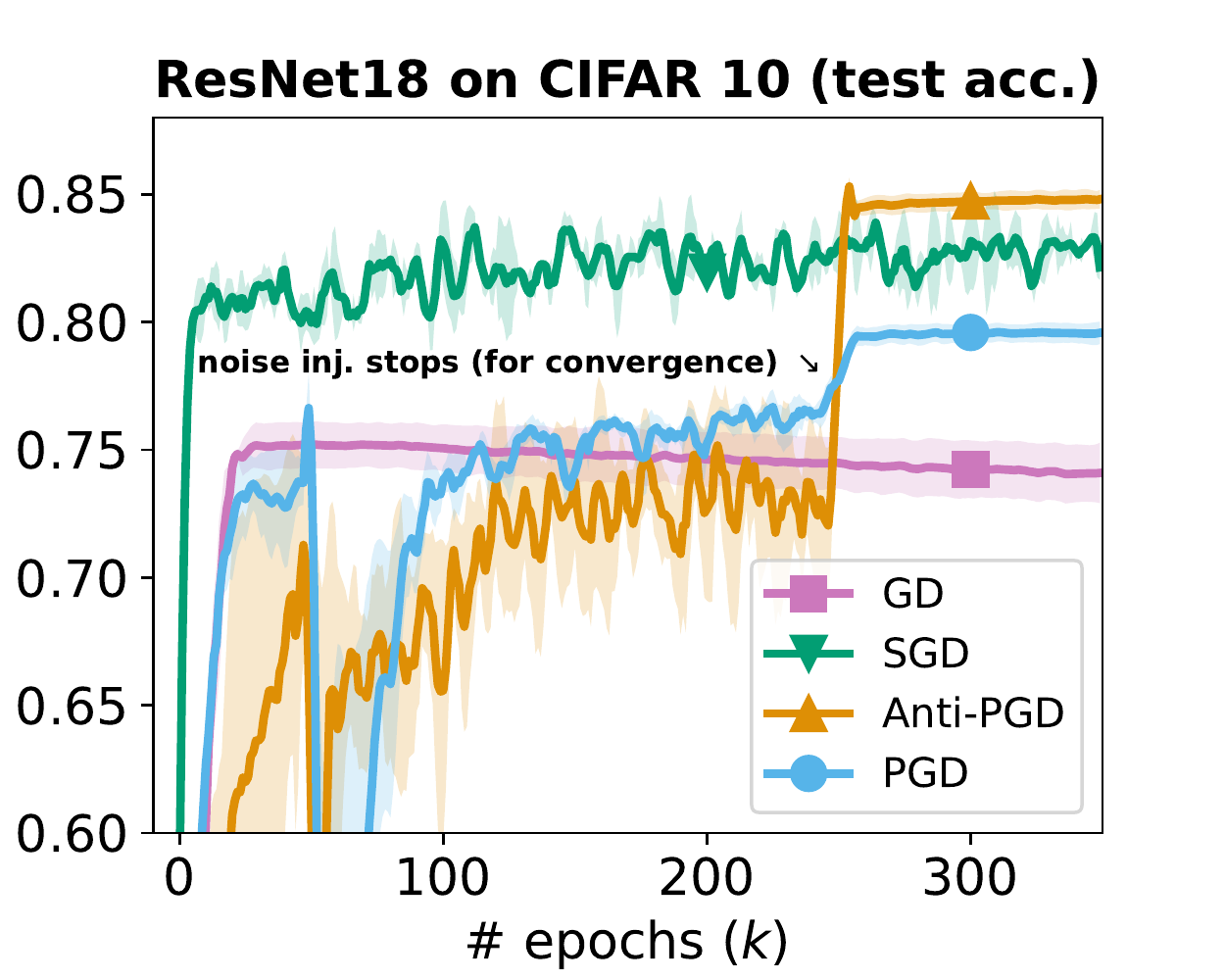}
    \\
    \includegraphics[height=0.25\textwidth]{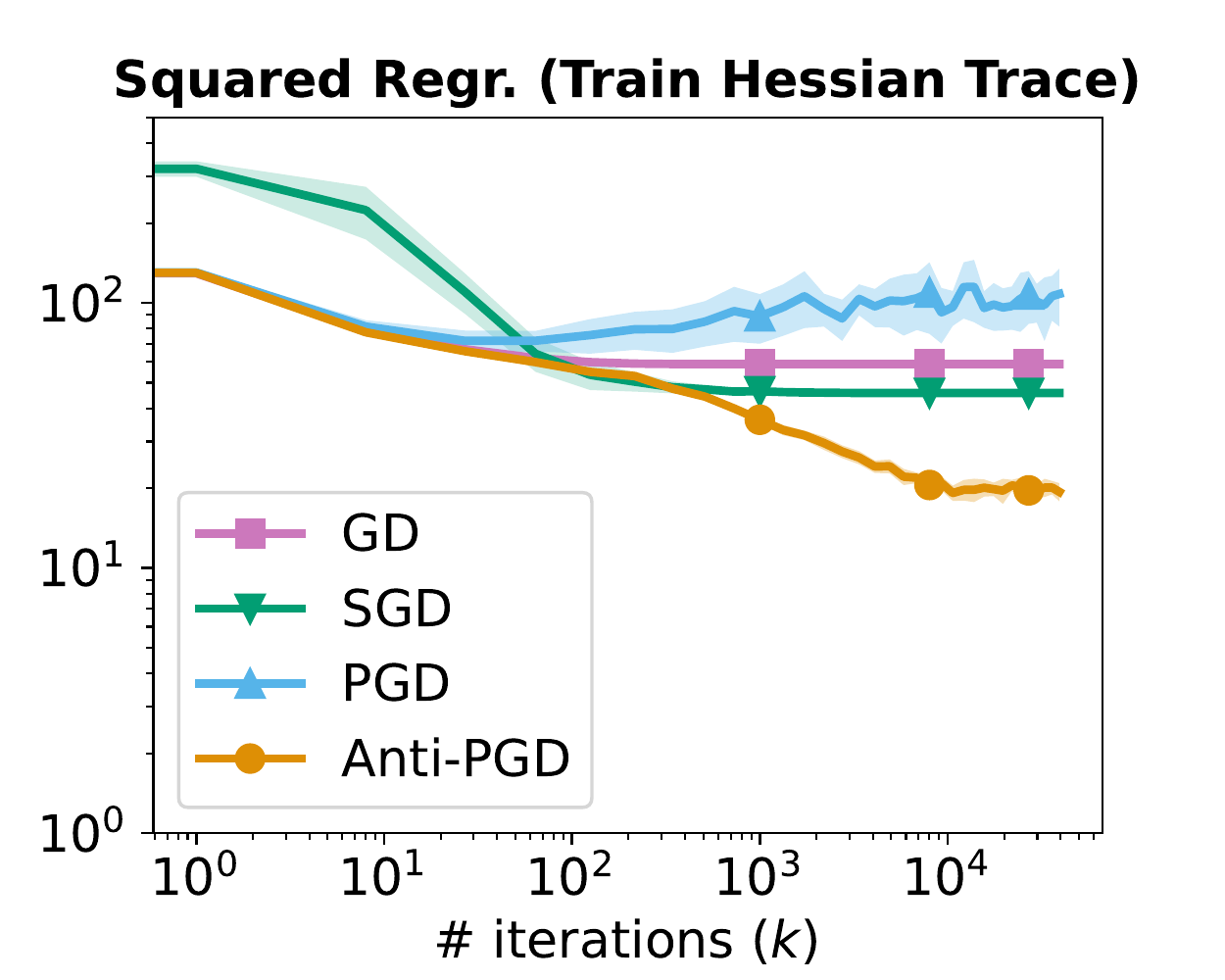}
    \includegraphics[height=0.25\textwidth]{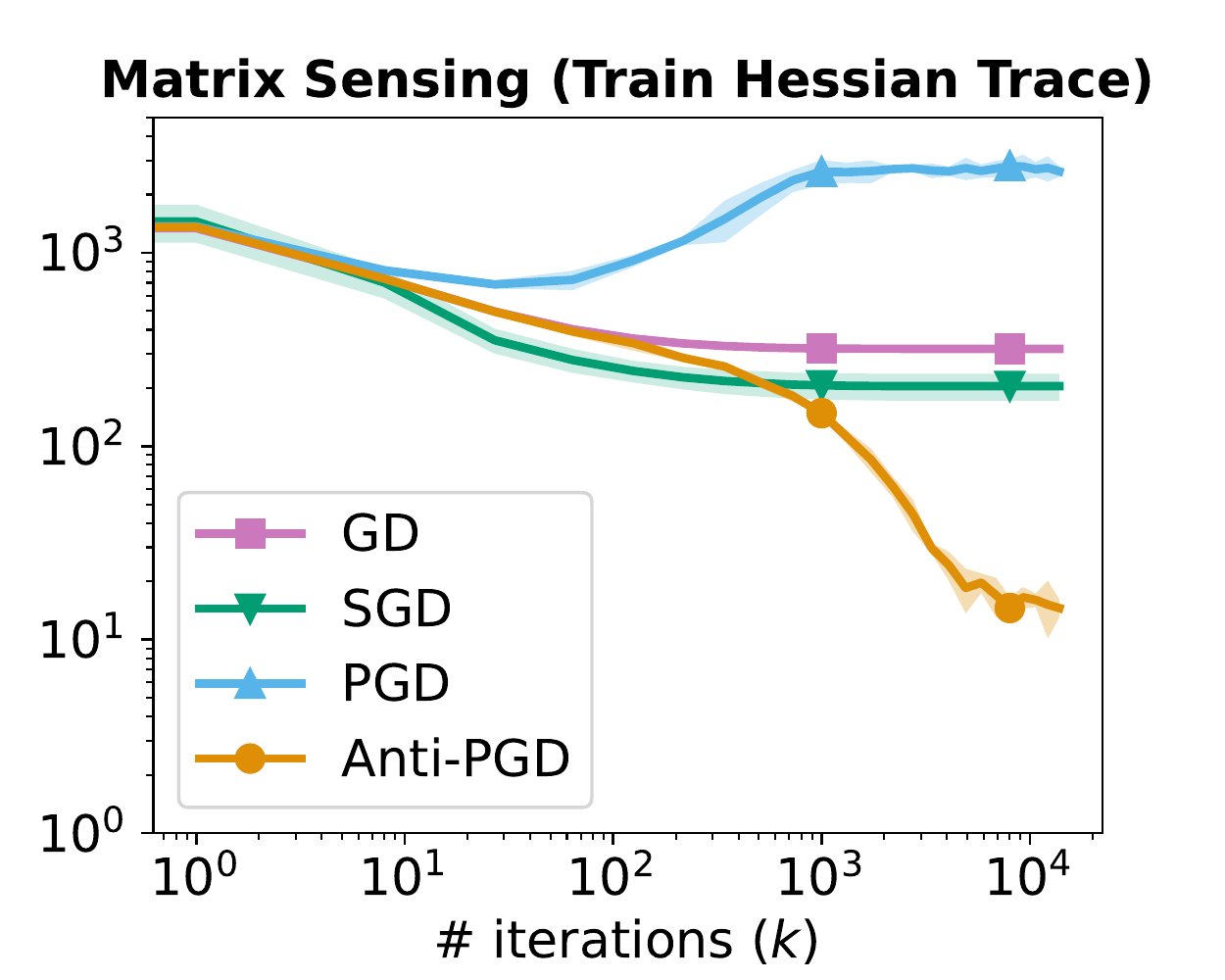}
    \includegraphics[height=0.25\textwidth]{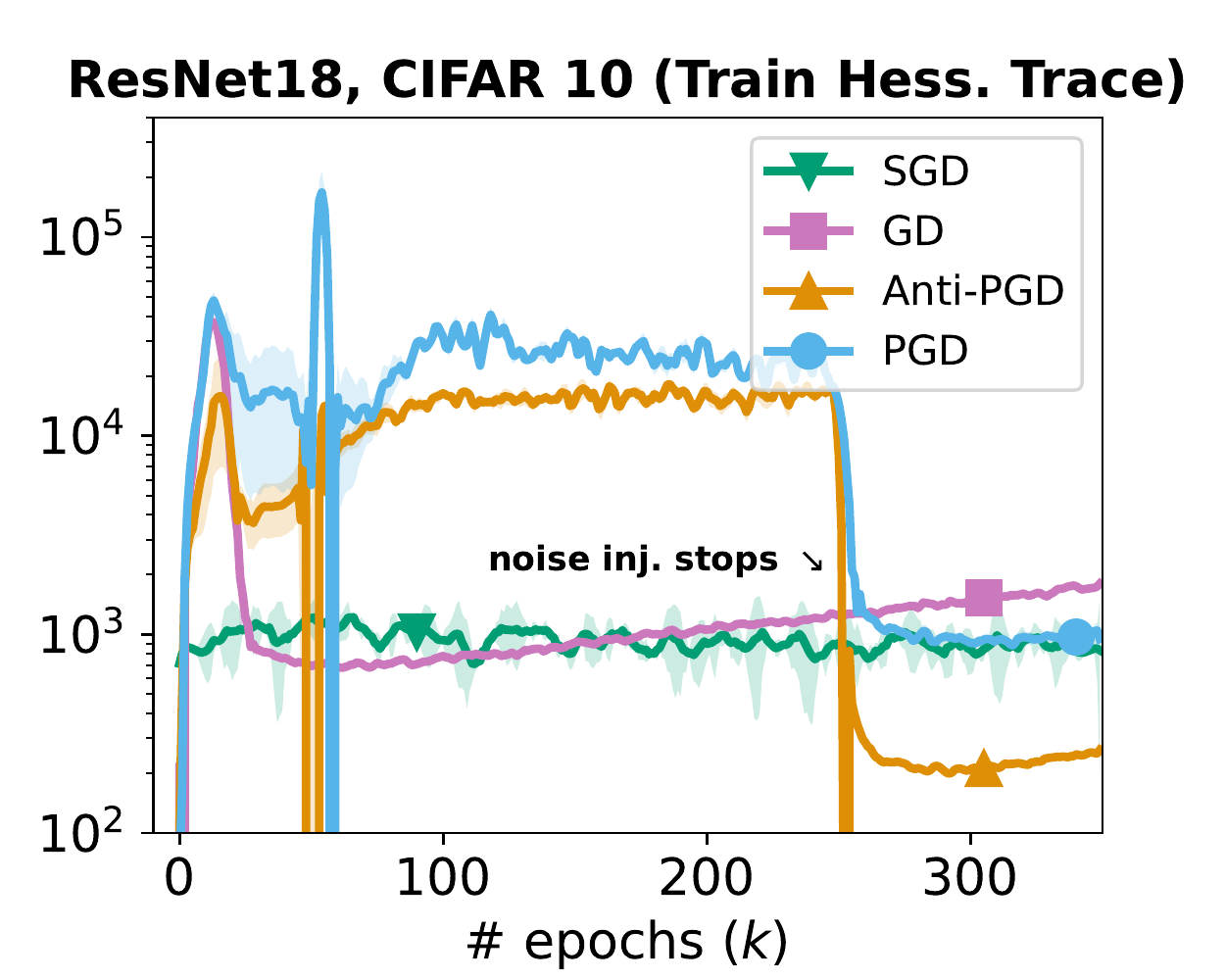}
    \vspace{-2.5mm}
    \caption{\small Effect of uncorrelated~(PGD) and anticorrelated~(Anti-PGD) noise injection on learning with gradient descent. Experiments are conducted on three non-convex machine learning problems with increasing complexity~(details in \S \ref{sec:exp}). These experiments are inspired by recent literature on label noise~\cite{blanc2020implicit,haochen2021shape}. Shown is the mean and standard deviation over several runs~(5 for the first two problems, 3 for the last). Findings are robust to hyperparameter tuning~(see Appendix~\ref{sec:exp_app}). All experiments suggest that Anti-PGD allows convergence to a flat minimizer~(lower Hessian trace), improving generalization. In the ResNet18 experiments, the high dimensionality makes it hard to evaluate metrics under noise injection~--~since we converge to a neighborhood with big (dimension dependent) radius. Hence, we evaluate the accuracy and the Hessian trace after stopping noise injection, to allow exact convergence to the nearest minimizer. Note that, while SGD can temporarily be better than Anti-PGD for a small number of iterations, Anti-PGD ultimately outperforms SGD in all experiments.For more details and further investigations, see~\S\ref{sec:exp}.}
    \label{fig:front_experiments}
    \vspace{-3mm}
\end{figure*}

\vspace{-3mm}
\section{Introduction}
It is widely believed that \emph{flat minima} generalize better than sharp minima in loss \emph{landscapes} of overparametrized models such as deep neural networks (DNNs).
This idea goes back to \citet{Hochreiter1995FlatMinima,Hochreiter1997FlatMinima} who observed that, in flat minima, it is sufficient to determine weights with low precision and conjectured that this correlates with a small generalization gap.
Although it has not been proved conclusively and the debate continues, this hypothesis has been supported by increasing empirical evidence \citep{keskar2016large,Chaudhari2017EntropySGD,jiang2019fantastic}.

On the other hand, much work has analyzed which \emph{optimization algorithms} yield good performance on test data.
For standard stochastic gradient descent (SGD), it is a common finding (both empirically and theoretically) that its stochastic noise tends to guide the optimizer towards flat minima; see  \citet{smith2020generalization} and the references therein.
Going beyond SGD, several papers have proposed to perturb (stochastic) gradient descent methods by injecting artificial noise.
So far, such perturbed gradient descent (PGD) methods have proved beneficial to quickly escape spurious local minima \citep{zhou2019pgd} and saddle points \citep{Jin2021pgd}.
Based on the findings made in prior work, it seems natural to ask about the role of noise injection on the generalization performance of a model. The exact question we investigate is \emph{whether stochastic noise can be designed to match (or even outperform) the favorable generalization properties of vanilla SGD}.


\vspace{-3mm}
\paragraph{Contribution.}
We start from the observation that prior PGD methods rely on independent (uncorrelated) perturbations. We question whether this choice is optimal and instead study whether (anti-)correlated perturbations are more suitable in terms of generalization. We introduce a new perturbed gradient descent method we name ``Anti-PGD'' whose perturbations at two consecutive steps are \emph{anticorrelated}.
We motivate this design by showing that Anti-PGD drifts (on average) to flat parts of the loss landscape. 
We conduct an extensive set of experiments~--~ranging from shallow neural networks to deep architectures with real data (e.g.~CIFAR 10)~--~and we demonstrate that Anti-PGD, indeed, reliably finds minima that are both flatter and generalize better than the ones found by standard GD or PGD.
We explain this observation with two theorems.
Firstly, we show that Anti-PGD minimizes the trace of the Hessian~--~in the sense of converging (in expectation) to a minimum of a regularized loss to which the trace of the Hessian is added. 
Secondly, we show in the simplest possible toy model (the ``widening valley'') that Anti-PGD converges to the flattest minimum~--~while GD gets stuck in sharper minima and standard (uncorrelated) PGD diverges.
In summary, these findings lead us to postulate that anticorrelated noise can be employed to improve generalization.

\vspace{-2mm}
\subsection{Related Work}

The following lines of work are closely connected to our work.
In particular, the connections with PAC-Bayes bounds and label noise are particularly relevant and will be further discussed later on.

\vspace{-3mm}
\paragraph{Generalization measures and flat minima.} 
Generalization measures are quantities that monotonically capture the generalization ability of a model. 
For example, \citet{keskar2016large} and \citet{Chaudhari2017EntropySGD} conducted an extensive set of experiments demonstrating that the spectrum of the Hessian of the loss $\nabla^2 L(w^{\ast})$ computed at a minimum $w^{\ast}$ is related to the generalization performance~--~in the sense that low eigenvalues of $\nabla^2 L(w^{\ast})$ tend to indicate good generalization performance.
To capture this phenomenon, several flatness (a.k.a.~sharpness) measures have been proposed as generalization measures. 
Notably,~\citet{jiang2019fantastic} conducted a large-scale comparison of many popular generalization measures and concluded that some flatness measures are among the best performing measures.
Recently, \citet{petzka2021relative} connected flatness to generalization via the notion of `relative flatness'.

We note, however, that the superiority of flat minima is contested: \citet{Dinh2017SharpMinima} demonstrated that sharp minima can also generalize well and point out that flatness is not invariant to reparametrization. 
Hence, the empirical finding of correlation between flatness and good generalization should not necessarily be regarded as a causal relationship.

\vspace{-3mm}
\paragraph{PAC-Bayes bounds.}
The generalization ability of a model can in theory be captured by upper bounding the generalization gap, as done by classical VC or Rademacher bounds, as well as PAC-Bayes bounds such as~\citet{LangfordCaruna2002PacBayes}, \citet{Neyshabur2017ExploringGeneralization,Neyshabur2018PacBayes} and \citet{tsuzuku2020normalized}. However, characterizing the generalization ability of deep learning models has proven to be a challenging task. Most classical bounds are vacuous when computed on modern over-parametrized networks. Encouragingly, empirical evidence~--~see, e.g.,  \citet{Dziugaite2017PacBayes} and follow-up works~--~has shown that PAC-Bayes bounds can be optimized to yield practically useful results. As discussed in~\citet{yang2019fast}, PAC-Bayes bounds can also be related to flatness, and more precisely to the trace of the Hessian. The latter quantity will be key in our analysis and we will explain this connection in detail in \S\ref{sec:conection_pacbayes}.
\vspace{-3mm}
\paragraph{The scale of noise in SGD.}
SGD tends to find minima that generalize surprisingly well in overparametrized models. This phenomenon has been explained from different perspectives in the literature.
Focusing on the intrinsic noise of SGD,~\citet{zhang2019algorithmic} and \citet{smith2020generalization} empirically showed that SGD generalizes well by converging to flat minima. Alternatively,~\citet{Bradley2021IncreasedRandomness} characterized the stationary distribution of SGD, demonstrating a connection between increased levels of noise (i.e.~smaller batch size or larger learning rate)  and convergence to flat minima.
Particularly related to our work is \citet{wei2019noise} who showed that, in some settings, SGD decreases the trace of the Hessian in expectation.

\vspace{-3mm}
\paragraph{The shape of noise in SGD.}
The distribution of the noise of SGD is often a subject of debate in the literature.
In this regard, \citet{simsekli2019tailindex} challenged the default assumption that SGD noise is Gaussian. In some particular settings, their work showed empirically that a heavy-tailed distribution is observed. This type of noise was then shown in~\citet{nguyen2019first} to yield faster exit from sharp to flat minima.
While the universality of the finding of~\citet{simsekli2019tailindex} is debated in the community~\citep{panigrahi2019non,xie2020diffusion}, the tail index of SGD has a drastic influence on its diffusion properties. For instance,  heavy-tail noise leads to faster escape from \textit{sharp minima}; see, e.g.,~Thm.~1 by~\citet{simsekli2019tailindex}. 
Moreover, heavy-tail noise is provably found in simple models~(e.g., linear regression on isotropic data), and in the regime of high-learning rates~\cite{gurbuzbalaban2021heavy} as an effect of multiplicative noise~\cite{hodgkinson2021multiplicative}.
Recently, \citet{wang2022eliminating} demonstrated that truncated heavy-tailed noise can eliminate sharp minima in SGD.

\vspace{-3mm}
\paragraph{Perturbed Gradient Descent (PGD).}
PGD is a version of (stochastic) gradient descent where artificial noise is added to the parameters after every step. 
Multiple PGD methods have been shown to help quickly escape spurious local minima \citep{zhou2019pgd} and saddle points \citep{Jin2021pgd}. These methods differ from our Anti-PGD in that they inject \emph{uncorrelated} perturbations. \\
Instead of perturbing the parameter, one can alternatively add noise to the gradient which can improve learning for very deep networks \citep{Neelakantan2015AddingGradientNoise,Deng2021Shrinking}.

\vspace{-3mm}
\paragraph{Label noise and implicit bias.}
Another way to add perturbations to SGD is to add noise to the labels of the data used for training.
Recent work has demonstrated that such perturbations are indeed beneficial for generalization by implicitly regularizing the loss \citep{blanc2020implicit,haochen2021shape,damian2021label}.
This alternative noise injection perturbs the labels \emph{before} computing the gradient~--~instead of the parameter \emph{after} a gradient-descent step, as in PGD and our Anti-PGD.
(Nonetheless these approaches are closely connected, as we will further explain in \S\ref{subsec:Connection_with_label_noise}.) \\
For the limit case of small SGD learning rate $\eta \to 0$, \citet{li2022what} introduced a general SDE framework to analyze the implicit bias in relation to flatness.

\section{Finding Flat Minima by Anti-PGD} \label{sec:finding_flat_minima}

After introducing our problem setting, we provide a detailed description of \emph{Anti-PGD} and explain how it is designed to find flat minima.

\vspace{-3mm}
\paragraph{Problem setting.} Let $\{(x^{(i)}, y^{(i)})\}_{i=1}^M$ denote a data set of $M$ input-output pairs with $x^{(i)}\in\R^{d_{\rm in}}$ and $y^{(i)}\in\R$.
We consider a machine learning model $f_w:\R^{d_{\rm in}}\to\R$, with parameters $w\in\R^d$, whose parameters are trained using empirical risk minimization. 
Let $\Ls^{(i)}:\R^d\to\R$ be the loss associated with the $i$-th data point $(x^{(i)}, y^{(i)})$.
We denote by $\Ls(w):=\frac{1}{M}\sum_{i=1}^M \Ls^{(i)}(w)$ the (full-batch) training loss, which we optimize to find the best parameters.
\vspace{-3mm}
\paragraph{Anti-PGD.} \emph{Gradient descent} (GD) iteratively optimizes the loss $L(w)$ by computing a sequence of weights $\{w_n\}_{n=0}^N$ where $w_{n+1} = w_n - \eta \nabla L(w_n)$ with step size (a.k.a.~learning rate) $\eta >0$. 
\emph{Perturbed gradient descent} (PGD) simply adds an i.i.d.~perturbation to each step, i.e.
\begin{equation} \label{eq:def_PGD}
    w_{n+1} = w_n - \eta \nabla L(w_n) + \xi_{n+1},
\end{equation}
where $\{\xi_n\}_{n=0}^{N}$ is a set of centered i.i.d.~random variables with variance $\sigma^2 I$.
Similarly, we define \emph{anticorrelated perturbed gradient descent} (Anti-PGD) as
\begin{equation} \label{eq:def_antiPGD}
    w_{n+1} = w_n - \eta \nabla L(w_n) + (\xi_{n+1} - \xi_{n}).
\end{equation}
In other words, Anti-PGD replaces the i.i.d.~perturbations $\{\xi_n\}_{n=0}^{N}$ in PGD \eqref{eq:def_PGD} with their increments $\{\xi_{n+1} - \xi_{n}\}_{n=0}^{N-1}$.
The name Anti-PGD comes from the fact that consecutive perturbations are anticorrelated:
\begin{equation*}
    \frac{\mathbb{E} \left[ (\xi_{n+1}-\xi_{n})(\xi_{n}-\xi_{n-1})^\top \right]}{2\sigma^2}
    \overset{\text{(iid)}}{=} - \frac{\Cov \left ( \xi_{0} \right)}{2\sigma^{2}} = - \frac 12 I.
\end{equation*} 

\subsection{Regularization in Anti-PGD}

While Anti-PGD \eqref{eq:def_antiPGD} is defined as a modification of PGD \eqref{eq:def_PGD}, it can alternatively be viewed as a regularization (smoothing) of the loss landscape $L$.
To see this, note that, after a change of variables $z_n := w_n - \xi_n$, the Anti-PGD step becomes
\begin{equation} \label{eq:Anti-PGD_z_formulation}
    z_{n+1}  =  z_n - \eta \nabla{L}(z_n + \xi_n).
\end{equation}
The corresponding loss $L(\cdot + \xi_n)$ can, in expectation, be regarded as a convolved (or smoothed) version of the original $L$.
To see in which direction the gradients of this loss (and thus Anti-PGD) are biased, we perform a Taylor expansion of $\partial_i L(\cdot)$ around $z_n$:
\begin{multline}\label{eq:z_n_dyn_by_dim}
    z^{i}_{n+1} = z^{i}_{n} - \eta \partial_i L(z_{n}) -\eta \sum_{j}\partial^2_{ij} L(z_{n})\xi_{n}^j\\ -  \underbrace{\frac{\eta}{2}\sum_{j,k} \partial^3_{ijk} L(z_{n})\xi^{j}_n\xi^{k}_n}_{= \frac{\eta}{2} \partial_i \sum_{jk} \partial^2_{jk} L(z_{n})\xi^{j}_n\xi^{k}_n } + O(\eta\|\xi_n\|^3),
\end{multline}
where the term under the brace is due to Clairaut's theorem (assuming that $L$ has continuous fourth-order partial derivatives).
By exploiting that $\xi_n$ has mean zero and covariance $\sigma^2 I$, we can express the conditional expectation of each step as
\begin{equation} \label{eq:implicit_bias}
    \mathbb{E}\left[ z_{n+1} \vert z_n  \right]
    =
    z_n - \eta \nabla \tilde{L}(z_n) + O\left(\eta \mathbb{E}[\|\xi_n\|^3] \right),
\end{equation}
where the \emph{modified loss} $\tilde{L}$ is given by
\begin{equation} \label{def:L_tilde}
    \tilde{L}(z) := L(z) + \frac{\sigma^2}{2} \tr(\nabla^2 L(z)),
\end{equation}
where $\tr(A)$ denotes the trace of a square matrix $A$.
(In Appendix~\ref{app:cond_var}, we also compute the conditional variance of Anti-PGD.)
The conditional mean, Eq.~\eqref{eq:implicit_bias}, highlights the \emph{motivation} for Anti-PGD:
When expressed in terms of the variable $z_n$, Anti-PGD in expectation (modulo the impact of the third moment of the noise) takes steps in the direction of a loss which is regularized by adding the trace of the Hessian.
The higher the noise variance $\sigma^2$, the stronger is the influence of the (trace of the) Hessian on Anti-PGD.
This is related to how stochastic gradient noise smoothes the loss in standard SGD \citep{kleinberg2018alternativesgd}, with the difference that, here, we inject artificial noise that \emph{explicitly} regularizes the trace of the Hessian. A discussion on the smoothing literature is postponed to \S\ref{sec:smoothing}.

In the next theorem, we analyze the case where the noise $\xi_n$ follows a symmetric Bernoulli distribution.
We find that, indeed, Anti-PGD (on average) minimizes the regularized loss $\tilde{L}$~--~in the sense that the regularized gradient converges.

\begin{restatable}[Convergence of the regularized gradients]{thm}{generalproof}
Let $L:\mathbb{R}^d\to\mathbb{R}$ be lower bounded with continuous fourth-order partial derivatives and $\beta$-Lipschitz continuous third-order partial derivatives, for some constant $\beta > 0$.
Consider the iterates $\{z_n\}_{n=0}^{N-1}$ computed by Anti-PGD as in \eqref{eq:Anti-PGD_z_formulation}, where for each $n$ the noise coordinate $\xi^n_i$ follows a symmetric centered Bernoulli distribution with variance $\sigma^2$ (i.e., $\sigma$ and $-\sigma$ have probability $1/2$).
If we set $\epsilon>0$ small enough such that $\eta = \Theta(\epsilon/\sigma^2)<\frac{1}{\beta}$ and let $N = \Omega(\sigma^2\epsilon^{-2})$, then it holds true\footnote{The first version of this paper had a mistake in the proof of this result and did not include the $O(\sigma^3)$ error. We note that this correction comes with minor changes in the theorem assumptions.} that
\begin{equation}
    \Exp\left[\frac{1}{N}\sum_{n=0}^{N-1}\|\nabla \tilde L(z_{n})\|^2\right]\le O(\epsilon) + O(\sigma^3).
\end{equation}
\label{thm:implicit_reg}
\end{restatable}
\vspace{-2mm}
For a proof, see Appendix~\ref{app:thm_implicit_reg}. Now that we have seen how exactly anticorrelated perturbations lead to a reduction of the trace of the Hessian appearing in \eqref{def:L_tilde}, we connect this finding to previous work relating to regularizing by the trace of the Hessian.

\subsection{Connection with PAC-Bayes Bounds} \label{sec:conection_pacbayes}

PAC-Bayes bounds can be interpreted as bounds on the average loss over a posterior distribution $Q$. These bounds connect to the curvature of the loss through the concept of expected sharpness. The following theorem makes this connection precise.

\begin{restatable}[\cite{Neyshabur2017ExploringGeneralization,tsuzuku2020normalized}]{thm}{pac}
\label{thm:pac}
Let $Q(w|w^*)$ be any distribution over the parameters, centered at the solution $w^*$ found by a gradient-based method. For any non-negative real number $\lambda$, with probability at least
$1-\delta$ one has
\vspace{-2mm}
\begin{align*}
    &\Ls_{\rm true}(Q(w|w^*)) \le \Ls(w^*) +\frac{\lambda}{2M} +\frac{1}{\lambda}\ln\left(\frac{1}{\delta}\right)\\& + \underbrace{\Ls(Q(w|w^*))-\Ls(w^*)}_{\text{expected sharpness}} +\frac{1}{\lambda} \KL[Q(w|w^*) || P(w)],
\end{align*}
where $\Ls_{\rm true}$ is the generalization loss; $L(Q) := \mathbb{E}_{w \sim Q} L(w)$ and $L_{\rm true}(Q) := \mathbb{E}_{w \sim Q} L_{\rm true}(w)$; $P$ is a distribution over parameters; and $\KL$ denotes the Kullback-Leibler divergence.
\end{restatable}
For a proof, see~\citet{tsuzuku2020normalized}. 
In the setting of this theorem, by picking $Q$ to be Gaussian with variance $s^2$, one obtains the following approximation of the expected sharpness
\vspace{-1mm}
\begin{equation}
\label{eq:trace_hess_pac}
    \Ls(Q(w | w^*),w^*)-\Ls(w^*) \approx \frac{s^2}{2} \tr(\nabla^2\Ls(w^*)).
\end{equation}
Thus, by minimizing the trace of the Hessian, Anti-PGD is expected to also reduce the PAC-Bayes bound from Thm.~\ref{thm:pac}.
In fact, the reasoning behind the bound in Thm.~\ref{thm:pac} has motivated researchers to find an explicit link between stochastic gradient noise and the trace of the Hessian at the solution found by SGD. Empirically, these quantities have a high correlation in many settings~\citep{yao2020pyhessian,smith2021origin}: usually, the lower the trace~(i.e., the flatter the minima), the higher is the test accuracy.
Similar bounds involving the trace of the Hessian are also discussed by~\cite{dziugaite2018data,wang2018identifying}.

\begin{figure*}
    \centering
    \includegraphics[width =0.35\linewidth]{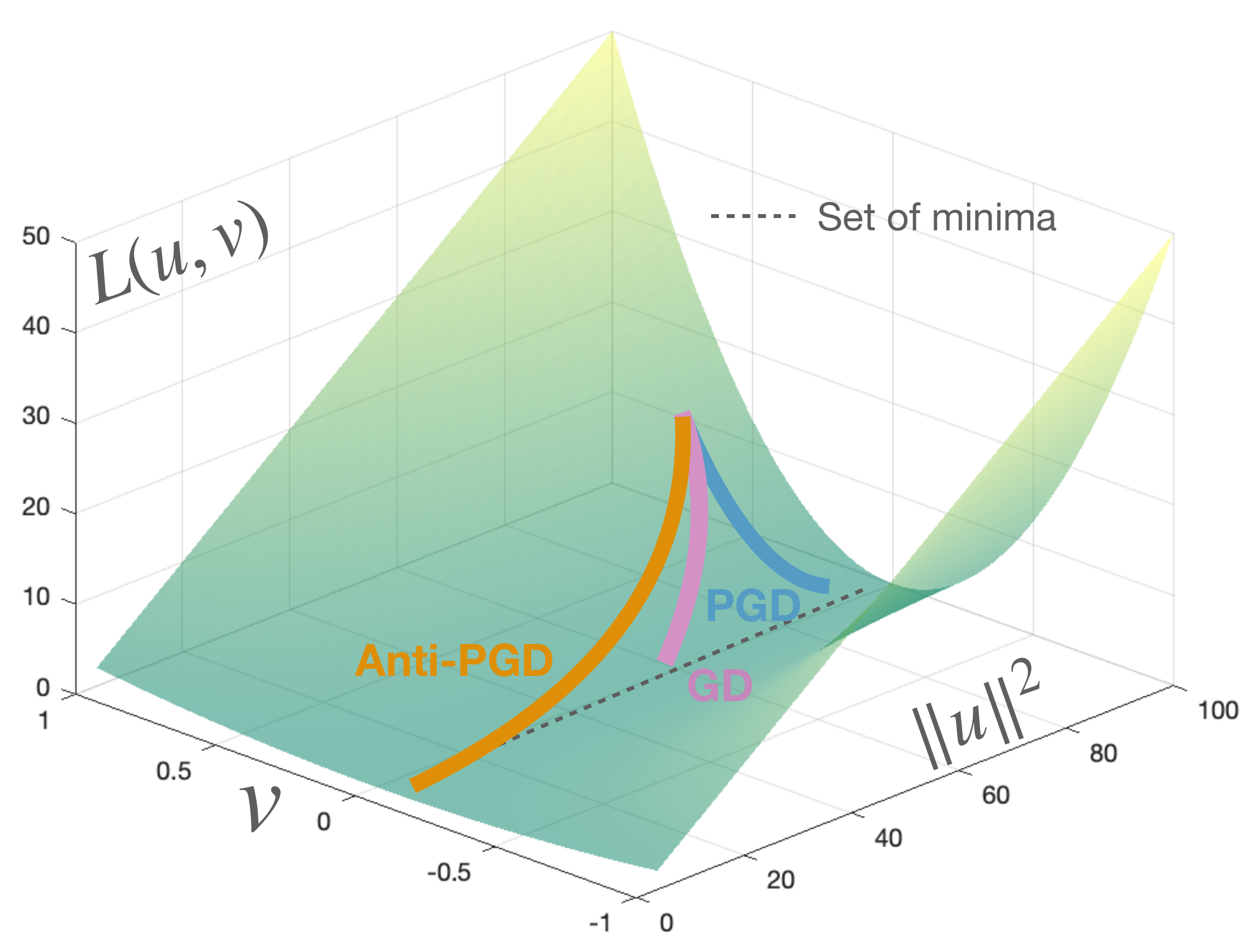}
    \hspace{1.3cm}
    \includegraphics[width=0.43\linewidth]{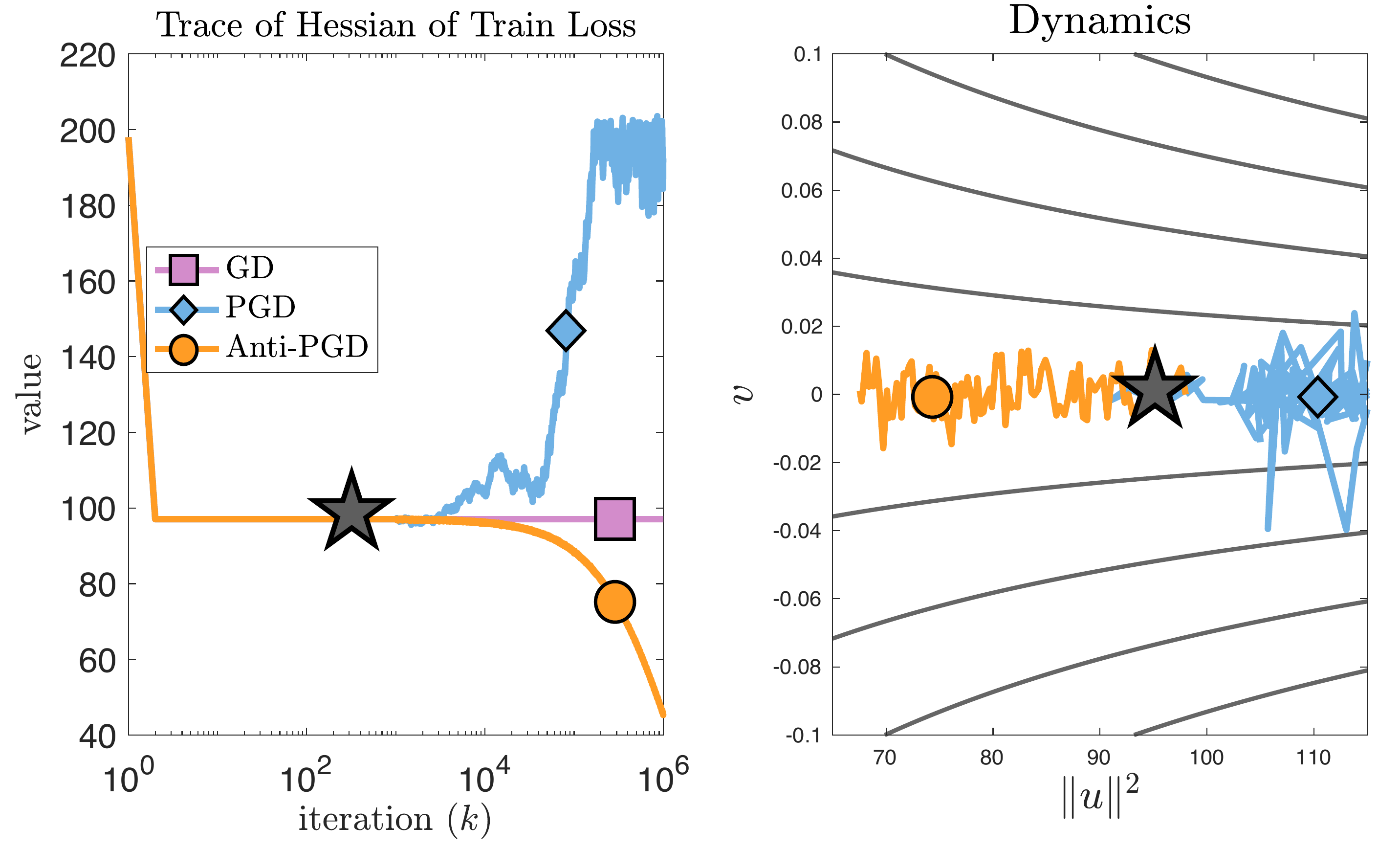}
    \vspace{-3mm}
    \caption{\textbf{Left}: Illustration of the \emph{widening valley} loss $L$, Eq.~\eqref{eq:def_wv}. A valley of minima with loss $L(u,v)=0$ for all $(u,v)$ with $v=0$; the smaller $\|u\|$, the flatter the minimum. GD gets stuck where it first touches this valley. PGD diverges to sharp regions (with high $\|u\|$). Anti-PGD converges to a flat minimum (with small $\|u\|$). \textbf{Right}: Simulation of the considered algorithms on the widening valley. After convergence of GD~(black star), we start injecting uncorrelated and anticorrelated noise. We choose $\eta = 0.01$, and $\sigma = 0.005$~--~yet the findings generalize to all sets of stable parameters. The observed behavior is supported by Thm.~\ref{thm:main_tube}. The plot looks similar for both Gaussian and Bernoulli noise injection.}
    \vspace{-2mm}
    \label{fig:tube_pic}
\end{figure*}

\subsection{Comparison with Label Noise}\label{subsec:Connection_with_label_noise}

Instead of perturbing $w$ as in PGD, label-noise methods perturb the label $y^{(i)}$ of the data.
If we denote $f_w(x)$ as the output of our model for input $x$, the label-noise update in the full-batch setting with squared loss is $w_{n+1} = w_n -\eta \nabla \bar L(w_n)$, with $\bar L(w) = \frac{1}{2}\sum_{i=1}^M\left[f_w(x^{(i)}) - y^{(i)} + \xi_{n+1} \right]^2$, for a set of random perturbations $\{\xi_n\}_{n=0}^N$.  It is instructive to compare the label-noise loss $\bar{L}$ with the Anti-PGD loss $L(\cdot + \xi_n)$ from Eq.~\eqref{eq:Anti-PGD_z_formulation}.
The above formula gives the gradient as
\begin{equation}
    \nabla \bar L(w) = \nabla  L(w) +\sum_{i=1}^M \nabla f_w(x^{(i)}) \xi_{n+1}.
\end{equation}
Hence, while label noise was observed to yield an improvement in terms of generalization~\citep{blanc2020implicit,haochen2021shape,damian2021label}, its effect (in general) is highly dependent on the model and on the data. 
Instead, the noise injection we propose is both data and model independent, as can be seen from the regularization in Eq.~\eqref{eq:implicit_bias}.

\subsection{Connection to Smoothing}
\label{sec:smoothing}
Eq.~\eqref{def:L_tilde} shows that Anti-PGD amounts to optimizing a regularized loss, which we can also interpret as a smoothing of the original objective function. Smoothing is of course not a new concept in the field of optimization as it is often used to regularize non-differentiable functions in order to compute approximate derivatives~\citep{nesterov2017random}, or to obtain faster rates of convergence~\citep{lin2018catalyst}.

In the context of deep learning, noise injection~(or even stochastic gradient noise) is often linked to smoothing \citep{kleinberg2018alternativesgd, stichescaping, bisla2022low}. As we saw in Eq.~\eqref{eq:Anti-PGD_z_formulation}, anticorrelated noise injection is equivalent to smoothing after a change of variables~--~this property was crucial in deriving the trace regularizer. We are not aware of any similar explicit regularization result in the smoothing literature~(most work focuses on the resulting landscape properties and convergence guarantees). Even though Anti-PGD is linked to smoothing, it is much more convenient to analyze: $\nabla f(x+\xi)$ follows a data-dependent distribution that is complex to characterize. Instead, in Anti-PGD, the smoothing effect comes from adding --- this is a \emph{linear} operation --- anticorrelated random variables. This is very convenient and will be leveraged in the proof for the next result, Thm.~\ref{thm:main_tube}.

\section{Convergence in Widening Valleys}  \label{sec:widevalley}

We have seen above that Anti-PGD acts as a regularizer on the trace of the Hessian. In this section, we will analyze the dynamics of Anti-PGD in more detail on the ``widening valley''~--~the simplest possible loss landscape with a changing trace of the Hessian. In the following subsections, we will introduce this model (\S\ref{subsec:wv_landscape}), demonstrate with experiments that Anti-PGD successfully finds flat minima in this model (\S\ref{subsec:wv_empirical}), prove this behaviour theoretically (\S\ref{subsec:wv_theoretical}), and explain how the widening valley relates to more realistic problems like sparse regression (\S\ref{subsec:wv_sparse}).

\subsection{The Widening Valley Landscape} \label{subsec:wv_landscape}

The \emph{widening valley} is defined as the loss function
\vspace{-1mm}
\begin{equation} \label{eq:def_wv}
    L(u,v) = \tfrac{1}{2} v^2 \|u\|^2,
\end{equation}
where $\|\cdot\|$ is the Euclidean norm, $v\in\mathbb{R}$, and $u\in\mathbb{R}^d$; see Fig.~\ref{fig:tube_pic}. 
The gradient and Hessian of $L$ are given by
\begin{equation} \label{eq:gradientHessianWV}
    \nabla L(u,v) = \begin{bmatrix} v^2\cdot u   \\ \| u \|^2 v   \end{bmatrix}, \quad    
    \nabla^2 L(u,v) = \begin{bmatrix} v^2 I_d & 2vu \\ 2vu^{\top} & \| u \| \end{bmatrix}.
\end{equation}
The trace of the Hessian is thus 
\begin{equation} \label{eq:trace_of_hessian_wv}
    \tr(\nabla^2 L(u,v))
    =  d v^2 + \| u \|^2. 
\end{equation}
We consider $L$ as a suitable problem to analyze the dynamics of GD and Anti-PGD as it has a relatively simple structure consisting of a valley of minima with monotonously changing flatness (as measured by the trace of the Hessian):
All $(u,v)$ with $v = 0$ are minima, but we also require $\| u \|$ to be minimized as well to get a small trace of the Hessian.

The widening valley can also be seen as a simplified local model of the landscape close to a minimizer. Indeed,~\citet{draxler2018essentially} showed that minimizers in neural networks are often connected by a path where the loss is exactly zero: no jumping is required for an optimizer to gradually increase the solution flatness.
While these valleys are not straight in general and the flatness might not change monotonously, our straight valley \eqref{eq:def_wv} with monotonously changing flatness serves as a first simplified model.
We will link it to a more realistic regression problem in \S \ref{subsec:wv_sparse}.

\subsection{Empirical Demonstration}  \label{subsec:wv_empirical}

When optimizing the widening valley \eqref{eq:def_wv}, GD will get stuck in any of the global minima $(u,v=0)$, regardless of their flatness. In particular, if the dimension $d\gg 1$, the path of GD will be biased towards making $v$ small and not optimizing $u$~(since the direction along $v$ is the most curved). As a result, the final Hessian trace will be $\|u_0\|^2$. Improving this by injecting noise is challenging: when adding stochastic perturbations, one has to balance perturbing $v$ away from zero~--~to get a gradient \eqref{eq:gradientHessianWV} to reduce $\|u\|$~--~while preventing $\|u\|$ from growing too much.

We find empirically that Anti-PGD succeeds to do this and moves to flat parts of the valley, while PGD does not; see Fig.~\ref{fig:tube_pic}.
This means that Anti-PGD converges to flat parts of the valley, while PGD diverges to sharper regions; see Fig.~\ref{fig:tube_verification}.

\subsection{Theoretical Analysis}  \label{subsec:wv_theoretical} 

\begin{figure}
    \centering
    \includegraphics[width = 0.49\linewidth]{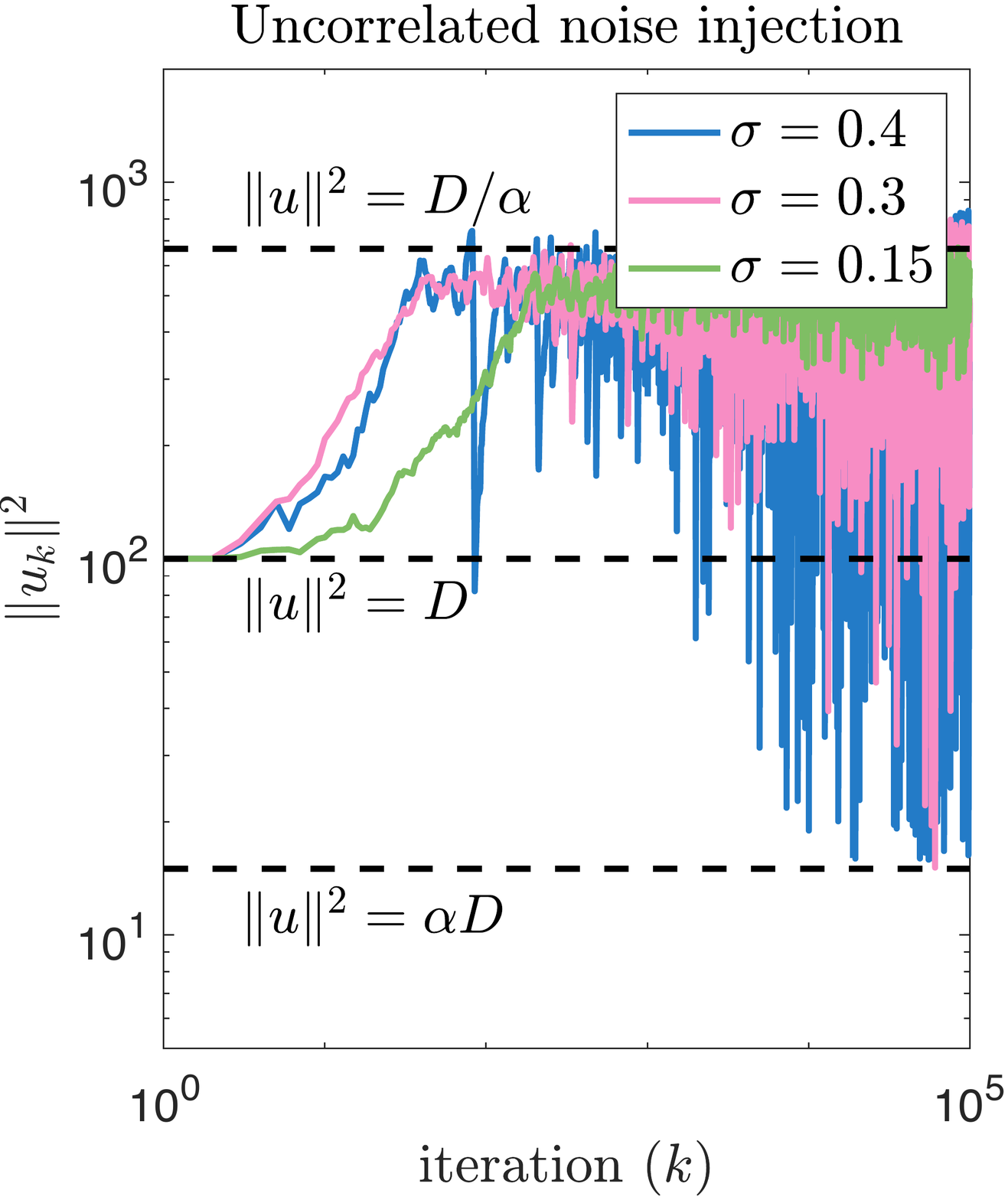}
    \includegraphics[width = 0.49\linewidth]{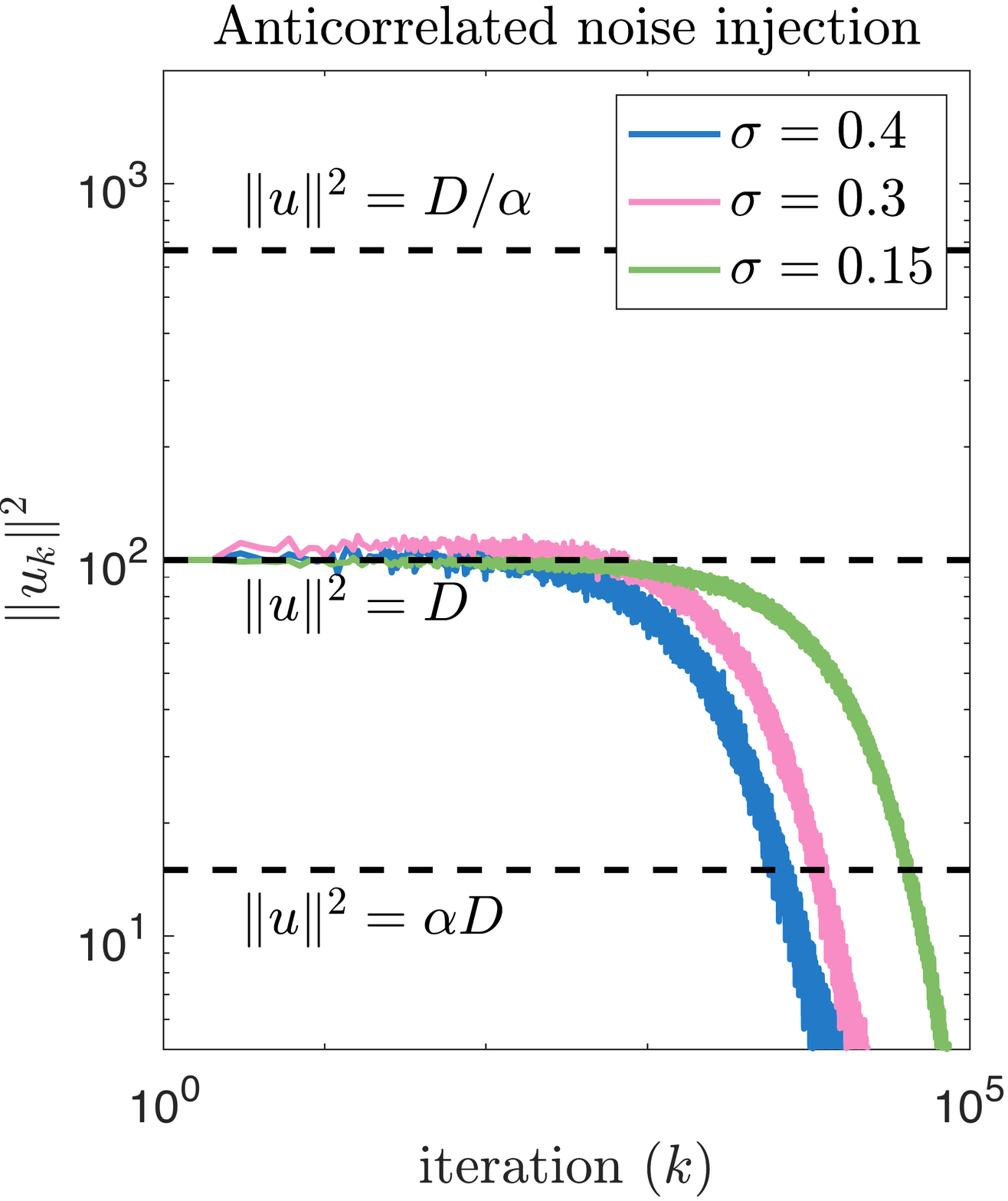}
    \vspace{-5mm}
    \caption{\small Numerical illustration and verification of Thm.~\ref{thm:main_tube}. Performance of PGD~(\textbf{left}) and Anti-PGD~(\textbf{right}) on the widening valley in Eq.~\eqref{eq:def_wv}. The setting and the notation is as described in Thm.~\ref{thm:main_tube}, and the simulation confirms the result: that is, Anti-PGD effectively decreases $\|u\|^2$ below $\alpha D$, where for this plot we consider $\alpha = 0.25$, $\eta = \alpha/D$ and $d=100$. Instead, the high problem dimensionality $d\ge 2/\alpha^2 = 32$ induces an increase in $\|u\|^2$ for standard PGD, which gets bigger than $D/\alpha$.}
    \vspace{-2mm}
    \label{fig:tube_verification}
\end{figure}

The following theorem proves what we empirically demonstrated in the preceding section.
\begin{restatable}[Widening Valley]{thm}{tube}
Let $L:\mathbb{R}^{d+1}\to\mathbb{R}$ be the \emph{widening valley} loss from Eq.~\eqref{eq:def_wv}.
We start optimizing from a point $w_0=(u_0,0)$, where $\|u_0\|^2=D\gg 1$~(e.g. the solution found by gradient descent), around which we consider the domain $\mathcal{D}_\alpha :=\{(u,v) \in \R^{d+1}:  \|u\|^2\in(\alpha D,D/\alpha)\}$ for some fixed $\alpha \in (0,1)$.
We want to compare the long-term stochastic dynamics of PGD and Anti-PGD, as defined in Eqs.~\eqref{eq:def_PGD} and \eqref{eq:def_antiPGD}, in terms of where they exit $\mathcal{D}_\alpha$.
As a noise model, we assume that the i.i.d.~perturbations $\xi_n$ are distributed according to a symmetric centered Bernoulli distribution  (i.e., $\sigma$ and $-\sigma$ have probability $1/2$) whose variance $\sigma^2$ is upper bounded by $\sigma^2 \in\left(0, \min\left\{\tfrac{\alpha^3 D}{2},  \tfrac{D}{8\alpha}\right\} \right]$. 
As a step size, we set $\eta = \frac{\alpha}{2D}$ which, for both methods, leads to stable dynamics inside of $\mathcal{D}_\alpha$. We find that (on average) PGD and Anti-PGD exit through different sides of $\mathcal{D}_\alpha$:
\vspace{-2mm}
\begin{enumerate}
    \item {\bfseries In high dimensions, PGD diverges away from zero.} If $d \geq \frac{2}{\alpha^2}$, then it holds for any admissible  $\sigma^2$ that
    \begin{equation}  \label{eq:thm_wv_PGD}
        \lim_{n \to \infty} \Exp\left[ \| u_n \|^2 \right] \geq D/\alpha,
    \end{equation}
    where $u_n$ are the first $d$ coordinates of $w_n$ computed by PGD as in \eqref{eq:def_PGD}. 
    \item {\bfseries Independent of dimensions, Anti-PGD goes to zero.} For  any $d \in \mathbb{N}$, if we choose any admissible $\sigma^2$ such that $\sigma^2 \leq \frac{\alpha D}{2d}$, then
    \begin{equation} \label{eq:thm_wv_antiPGD}
        \lim_{n \to \infty} \Exp\left[ \| u_n \|^2 \right] \leq \alpha D,
    \end{equation}
    where $u_n$ are the first $d$ coordinates of $w_n$, computed by Anti-PGD as in \eqref{eq:def_antiPGD}. 
\end{enumerate}
\label{thm:main_tube}
\end{restatable}

As expected, this theorem implies that, as $n \to \infty$, Anti-PGD reduces the trace of Hessian while PGD increases it. 
For a proof, see Appendix~\ref{app:wv_proof}.

\begin{restatable}[The trace of the Hessian in the widening valley]{corollary}{cortube} \label{corollary:main_tube}
In the same setting as Thm.~\ref{thm:main_tube}, let $\eta = \frac{\alpha}{2D}$, $\sigma^2 \in\left(0, \min\left\{\tfrac{\alpha^3 D}{2},  \tfrac{D}{8\alpha},\frac{\alpha D}{2d}\right\} \right]$ and $d\ge\frac{2}{\alpha^2}$. If $\alpha\ll 1$, then
\begin{align*}
    &\lim_{n \to \infty} \Exp[\tr(\nabla^2 L(w^{\rm anti}_n)))]
    \leq
    16 \alpha D \ll \Exp[\tr(\nabla^2 L(w_0)))]\\
    &\lim_{n \to \infty} \Exp[\tr(\nabla^2 L(w^{\text{un}}_n)))]\ge  \ D/\alpha \ \gg
    \Exp[\tr(\nabla^2 L(w_0)))],
\end{align*}
where $w^{\text{un}}_n=(u_n,v_n)$ and $w^{\rm anti}_n=(u_n,v_n)$ are the weights computed by Anti-PGD and PGD respectively.
\end{restatable}

\subsection{Relation to Linear Networks with One Hidden Unit}  \label{subsec:wv_sparse}

In this section, we explain how widening valleys, similar to our model \eqref{eq:def_wv}, might appear in more realistic learning problems.
To this end, consider sparse regression with $M$ input-output pairs $\{(x^i,y^i)\}_{i=1}^M$, where $x^i\in \mathbb{R}^{m+d}$, $d,m>1$, and $y^i\in\mathbb{R}$ for all $i\in[M]$.
To induce sparseness, we consider the setting where only the first $m$ features of each $x^i$, i.e., $(x^i_1,x^i_2,\dots,x^i_m)$ are relevant predictors, while the other features $(x^i_{m+1},x^i_{m+2},\dots,x^i_{m+d})$ are uncorrelated from the target.
Further, we assume that the input has isotropic standardised distribution. 
As predictor, we consider a neural network with one hidden neuron and standard square loss 
\begin{equation}
    L(u,v) = \frac{1}{2M}\sum_{i=1}^M \left ( y^i - v\cdot u^\top x^i \right )^2.
\end{equation}
\begin{figure}
    \centering
    \includegraphics[width = 0.2\textwidth]{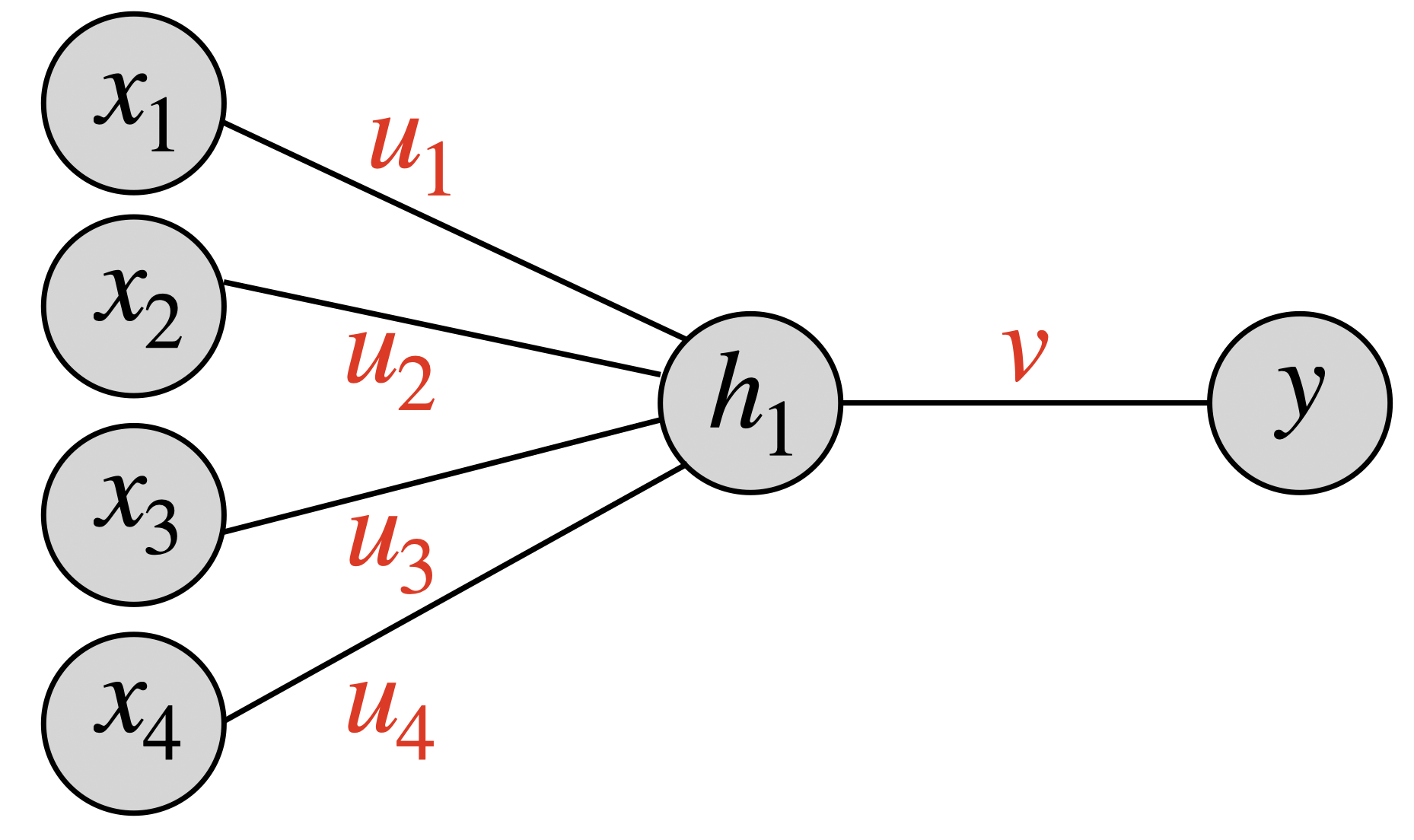}
    \vspace{-2mm}
    \caption{Pictorial illustration of the network~(linear activations, one hidden unit) we study in~\S\ref{subsec:wv_sparse}. The associated loss function, Eq.~\eqref{eq:sparse_loss}, has striking similarities to the widening valley, Eq.~\eqref{eq:def_wv}.}
    \vspace{-3mm}
\end{figure}
Such a loss is highly non-convex, due to the non-linear interaction between $v$ and $u$. By expanding the square, we obtain
\begin{equation*}
    2 L(u,v) = \Exp_i[ (y^i)^2 ] - 2v\cdot u^\top \Exp[y^ix^i] + v^2\Exp_i[ (u^\top x^i)^2].
\end{equation*}
We can drop the first term since it is irrelevant for optimization. Further, the last term can be written as
\begin{equation}
    v^2\Exp_i [ (u^\top x^i)^2 ] 
    = v^2\Exp_i [ \tr(u^\top x^i (x^i)^\top u) ].
\end{equation}
Using the cyclic property of the trace and the assumption $\Exp_i [x^i (x^i)^\top] = I$, we get
\begin{equation}
    v^2\Exp_i[(u^\top x)^2] = v^2\tr( u u^\top) = v^2\|u\|^2.
\end{equation}
Therefore, we obtain $L(u,v) = \frac{1}{2}v^2\|u\|^2- 2v\cdot u^\top \Exp[y^ix^i]$. We now use the sparseness assumption: since the features $(x^i_{m+1},x^i_{m+2},\dots,x^i_{m+d})$ are uncorrelated from the target, we have 
\begin{equation}\label{eq:sparse_loss}
    L(u,v) = \frac{1}{2}v^2\|u\|^2- 2v\cdot u_{1:m}^\top \Exp[y^ix^i_{1:m}].
\end{equation}
This loss \eqref{eq:sparse_loss} is very similar to the widening valley \eqref{eq:def_wv}:
For good generalization, the weights relative to the spurious coordinates $(u_{m+1}, u_{m+2},\dots, u_{m+d})$ have to be set to zero. Unfortunately, the solution of gradient descent~(without further regularization), in general does not have this property~(see \S\ref{subsec:wv_empirical}). 
This fact motivates us to look at the dynamics in the space $(u_{m+1}, u_{m+2},\dots, u_{m+d},v)$, ignoring the dynamics on the space $(u_{1}, u_{2},\dots, u_{m})$. In the spurious subspace of the parameter space, the last term in the last equation is a constant, and therefore the effective loss becomes $L(u,v) = \frac{1}{2}v^2\|u\|^2$, where $u\in\mathbb{R}^d$ denotes the vector $(u_{m+1}, \dots, u_{m+d})$.

\begin{figure*}[ht]
    \centering
    \includegraphics[height = 0.28\textwidth]{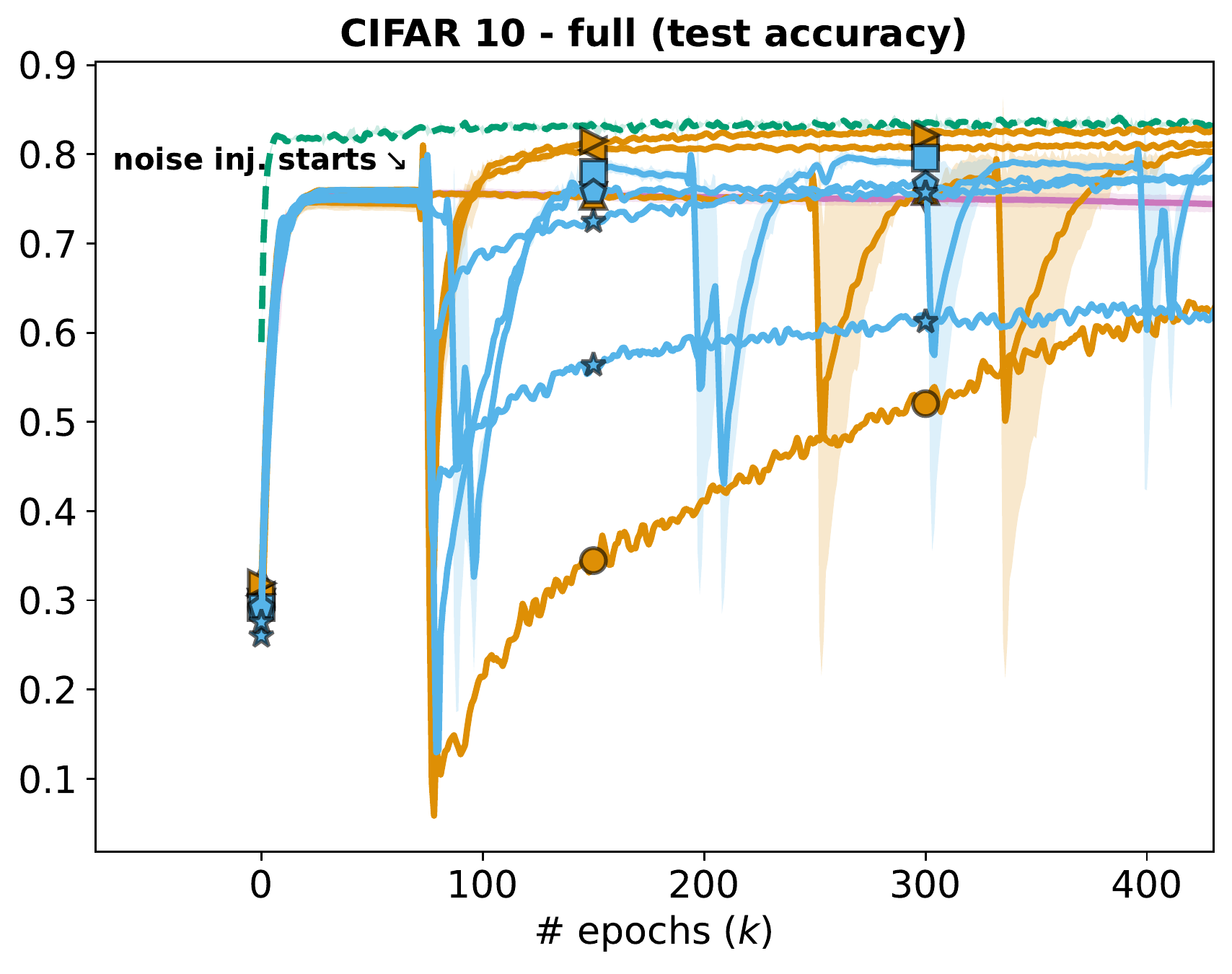}
        \includegraphics[height = 0.28\textwidth]{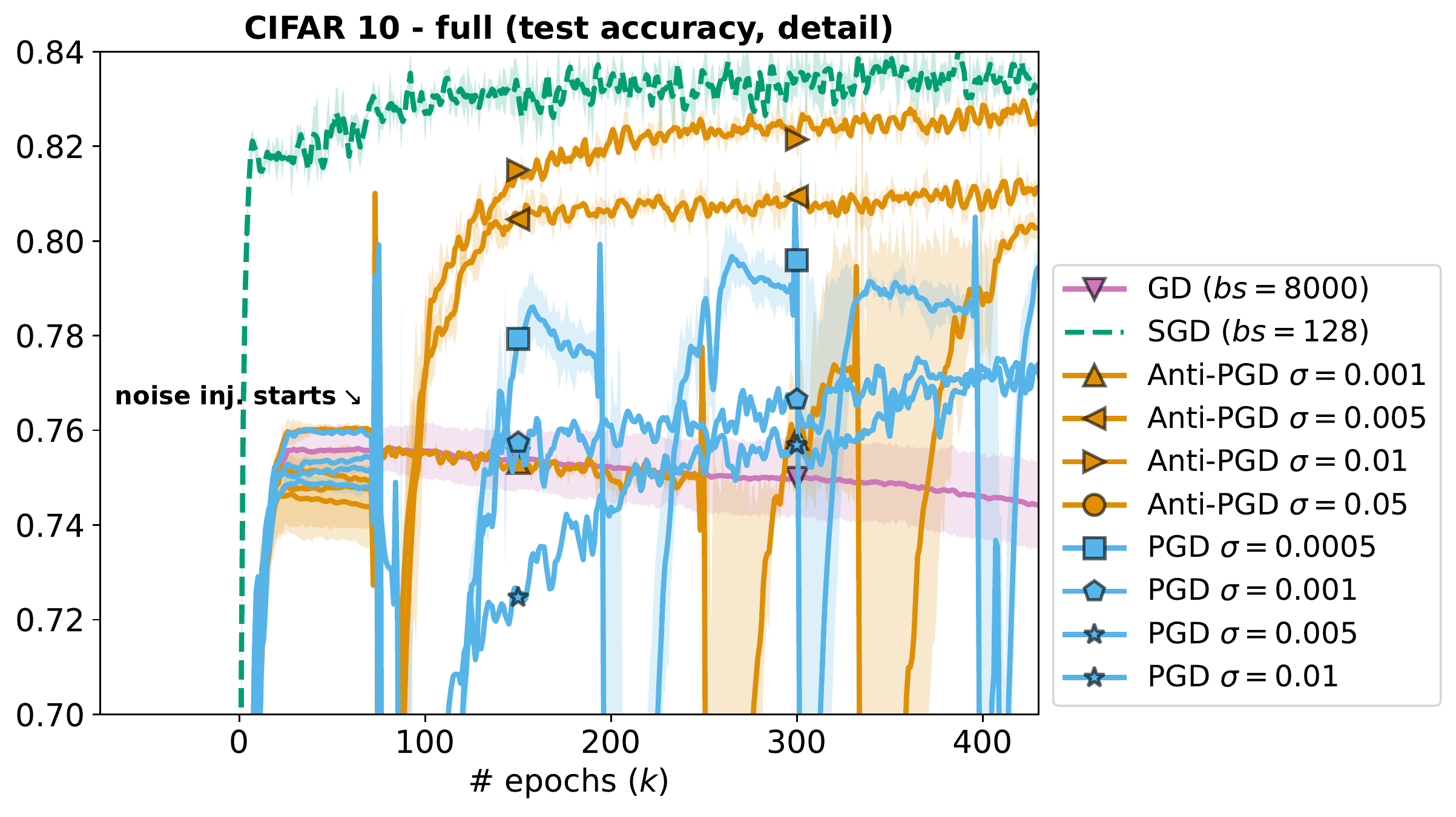}
    \vspace{-2mm}
    \caption{\small Anti-PGD, PGD and GD on CIFAR 10, using a ResNet18-like architecture~(details and further plots in the Appendix). Mean and one standard deviation are plotted. We train all algorithms with a step size of $0.05$, which leads to fast convergence of both GD and SGD~(batch size 128) in 70 iteration. Compared to Fig.~\ref{fig:page1}, we start the noise injection at epoch 75. Results are discussed in the main text. }
    \vspace{-3mm}
    \label{fig:cifar-full-inj-init}
\end{figure*}

\vspace{-2mm}
\section{Additional Experiments and Details}
\label{sec:exp}
We demonstrated the validity of our theoretical findings in Fig.~\ref{fig:front_experiments} by showcasing the performance of Anti-PGD on the three different problems (see the next three paragraphs).
Finally, in another experiment on CIFAR 10, we will show that Anti-PGD can recover from a sharp minimum. 
\vspace{-3mm}
\paragraph{Quadratically-parametrized linear regression.} For a data matrix $X\in\mathbb{R}^{n\times d}$~($d=100, n=40$) and targets $y\in\R^d$, this loss is $L(w) = \frac{1}{4n}\|X (w\odot w) - y\|^2$, where $\odot$ denotes the element-wise product. While the expressive power of the underlying model is limited, it has a few interesting features which make it a compelling case study~(see discussion by~\citet{haochen2021shape}). First, note that the nonlinear parametrization makes the loss non-convex: to see this simply note that changing the sign in any weight does not change the loss. Furthermore, inspecting the Hessian~(see \S\ref{sec:squared_regr}) one can easily see that minima with different curvature exist. On this problem, Anti-PGD~(with Gaussian perturbations) is able to find a flat minimum; a pictorial illustration of the corresponding dynamics can be found in Fig.~\ref{fig:page1}. In \S\ref{sec:squared_regr} we show that this good performance is robust to different choices of hyperparameters: we found no setting where well-tuned PGD~(with Gaussian perturbations) outperforms Anti-PGD.
\vspace{-3mm}
\paragraph{Matrix sensing.} The corresponding loss function has some similarities with quadratically-parametrized linear regression, and was considered by~\citet{blanc2020implicit} to study label noise. All the findings of the above paragraph hold true in this setting as well. Details on the experimental setup and hyperparameter tuning can be found in~\S\ref{sec:matrix_sensing}.
\vspace{-3mm}
\paragraph{CIFAR10 on ResNet 18.} We consider training a ResNet18-like architecture~\cite{he2016deep} with batch normalization. Architecture details are provided in~\S\ref{sec:cifar_app}. The performance on this network greatly depends on careful hyperparameter tuning, algorithmic choices~(e.g., adaptive step sizes), schedulers, etc. Here, to keep things simple, we train with a simple SGD optimizer (with momentum $0.9$), and select a learning rate of $0.05$. To approximate full-batch gradient descent we use a very large batch size of $7500$ samples~(i.e. until saturation of 5 GPUs). To isolate the effect of noise injection, without mixing it with mini-batch noise, we also run PGD and Anti-PGD~(with Gaussian perturbations) in this high batch regime~($1/7$ of the dataset). For SGD, we instead select a batch size of $128$, and keep the learning rate at $0.05$. For convergence of the test accuracy and the Hessian trace, it is convenient to kill the noise injection after 250 epochs~--~so that the optimizer converges to the nearest minimum. Again, we see that (compared with PGD) Anti-PGD find a flatter minimum that generalizes better. Crucially, we show in Appendix~\S\ref{sec:cifar_app} that~--~as in the problems above~--~no tuning of $\sigma$~(noise injection level) can help PGD to reach the generalization performance of Anti-PGD. This is confirmed by the results in the next paragraph.

\vspace{-3mm}
\paragraph{Recovering from a sharp minimum.} In this experiment, we keep the parameter settings as in the last paragraph, but instead consider injecting noise only after 75 epochs~--~i.e., after convergence of full-batch gradient descent. As a result, Anti-PGD and PGD are trapped in the minimum found by GD, until noise injection starts. The behavior we observe~(Fig.~\ref{fig:cifar-full-inj-init}) after noise injection resembles our widening valley model in Fig.~\ref{fig:tube_pic}: noise injection makes the dynamics suddenly unstable, and we observe several points during training where the algorithm is~(probably) switching between minima. Note that this behavior is very different from the one observed in Fig.~\ref{fig:front_experiments}, where we start noise injection from the very beginning~--~and then kill the noise injection at the end. In contrast, we here do the reverse: we first initialize at a bad minimum and then inject noise to recover. We observe that Anti-PGD is able to recover from the bad initialization much better than PGD. Interestingly, while stopping noise injection at the end of training was needed for good accuracy in Fig.~\ref{fig:front_experiments}, here no such step is needed for Anti-PGD to get an accuracy close to SGD~(i.e., we directly recover from a bad minimum). We postulate that this difference comes from the different landscape properties when (a) close to initialization or (b) close to a local minimizer. For the last two experiments, we provide further plots~(e.g. test-train loss) in Appendix~\ref{sec:cifar_app}.
In summary, we found that Anti-PGD reliably finds flat minima that generalize well~--~as predicted by the theory in \S \ref{sec:finding_flat_minima} and \S \ref{sec:widevalley}.

\vspace{-3mm}
\paragraph{Anti-SGD.} In the appendix, we show that a combination of mini-batch noise and anti-correlated noise is able to further improve on the performance~(see Fig.~\ref{fig:anti-sgd}). 



\vspace{-3mm}

\section{Conclusion and Future Work}

Motivated by recent findings on the correlation of the flatness of minima with their generalization performance, we demonstrated that anticorrelated noise injection can improve the generalization of machine learning models~--~by biasing the optimization toward flatter minima.
To this end, we replaced the i.i.d.~perturbations in perturbed gradient descent with anticorrelated ones.
We proved that the resulting method \emph{Anti-PGD} regularizes the trace of the Hessian, a common measure of flatness.
In order to provide further theoretical justification, we introduced the \emph{widening valley} model and proved that Anti-PGD converges to the flat part of the valley~--~while GD and standard PGD remain in sharper regions.
In realistic experiments with real data (e.g.~CIFAR 10), we likewise observed that Anti-PGD converges to flat minima that generalize well (compared with GD and standard PGD).

\vspace{-1mm}
These discoveries lead us to hypothesize that anticorrelated noise can improve the generalization performance of a model~--~which opens up several directions to investigate.
First of all, the range from uncorrelated to anticorrelated perturbations should be explored further, and combined with different noise distributions. 
Since uncorrelated noise can help to quickly exit saddle points \citep{Jin2021pgd}, a compromise (or adaptive schedule) between uncorrelated and anticorrelated might be beneficial.

\vspace{-1mm}
Moreover, it seems worthwhile to explore the implications of our findings for standard SGD. Unlike common noise injection techniques, the noise of SGD is data-dependent and its magnitude is determined by the ratio of step size and batch size. One could, however, modify the selection of the batches to  (negatively) correlate the stochastic gradient noise of subsequent steps.
One could also add anticorrelated noise on top of the existing noise in SGD, or inject it only after the test loss of SGD (or another optimizer) plateaus. Both theoretical and empirical results are likely to provide novel insights about the importance of noise in optimization.

\paragraph{Acknowledgement.} We would like to thank Nacira
Agram and Bernt Øksendal for some helpful discussions in the early stage of this project, as well as the reviewers for their feedback that greatly helped us improve this manuscript.
Frank Proske acknowledges the financial support of the Center for International Cooperation in Education, project No CPEA-LT-2016/10139, Norway.
Hans Kersting and Francis Bach acknowledge support from the French government under the management of the Agence Nationale de la Recherche as part of the “Investissements d’avenir” program, reference ANR-19-P3IA-0001 (PRAIRIE 3IA Institute), as well as from the European Research Council (grant SEQUOIA 724063). Aurelien Lucchi acknowledges the financial support of the Swiss National Foundation, SNF grant No 207392.

\bibliography{paper}
\bibliographystyle{icml2022}

\newpage

\appendix
\onecolumn


\section{Computation of Conditional Variance}\label{app:cond_var}

We know from Eq.~\eqref{eq:implicit_bias} that the \emph{conditional mean} of an Anti-PGD step is
\begin{equation}
    \mathbb{E}\left[ z_{n+1} \vert z_n  \right]  =
    z_n - \eta \nabla \tilde{L}(z_n) + O\left(\eta \mathbb{E}[\|\xi_n\|^3] \right),  
\end{equation}
where the \emph{modified loss} $\tilde{L}$ is given by
\begin{equation}
    \tilde{L}(z) := L(z) + \frac{\sigma^2}{2} \tr(\nabla^2 L(z)).
\end{equation}
This stands in contrast to a standard PGD step (Eq.~\eqref{eq:def_PGD}) whose conditional mean coincides with gradient descent (i.e.~includes no implicit bias):
\begin{equation}
    \mathbb{E}\left[ w_{n+1} \vert w_n  \right]  =
    w_n - \eta \nabla L(w_n).
\end{equation}

In this section, we additionally compute the conditional variance of Anti-PGD and PGD.
We start with PGD.
There, we obtain by a second-order Taylor expansion of $\partial_i L(\cdot)$ around $z_n$ in Eq.~\eqref{eq:def_PGD} that
\begin{equation}
    w_{n+1}^i  =
    w_n^i - \eta \partial_i L(w_n) - \eta \sum_{j=1} \partial_{ij}^2 L(w_n) - \frac{\eta}{2} \partial_i \sum_{j,k} \partial_{j,k}^2 L(w_n) + \xi^i_{n+1}.
\end{equation}
Hence, the conditional variance for PGD is simply
\begin{equation}
    \operatorname{var}\left[w^{(i)}_{n+1} \vert w_n\right] = \operatorname{var}\left[ \xi^i_{n+1} \right] = \sigma^2.
    \label{eq:conditional_variance_vanilla}
\end{equation}
For Anti-PGD on the other hand, we compute 
\begin{align*}
    \operatorname{var}\left[z^{(i)}_{n+1} \vert z_n\right]  &\overset{Eq.~\eqref{eq:z_n_dyn_by_dim}}{=} 
    \mathbb{E}\Big[ \Big[ \eta \partial_i L(z_{n}) + \eta \sum_{j}\partial^2_{ij} L(z_{n})\xi_{n}^j  + \frac{\eta}{2}\sum_{j,k} \partial^3_{ijk} L(z_{n})\xi^{j}_n\xi^{k}_n + O(\eta\|\xi_n\|^3)\Big]^2 \vert z_n \Big]  \\ &= 
    \eta^2 [\partial_i L(z_n)]^2 + \eta^2 \sigma^2 \sum_{j=1}^d [\partial_{ij}^2 L(z_n)]^2 + \frac{\eta^2 \sigma^4}{4} \sum_{j \neq k} \partial^3_{ijk} L(z_n) + \frac{\eta^2}{4} \sum_{j=1}^d \partial_{ijj} L(z_n) \mathbb{E} [(\xi_n^j)^4] \\ 
    &\phantom{=}\ + \frac{\eta^2 \sigma^4}{4} \sum_{j \neq k} [\partial_{ijj} L(z_n)] \cdot [\partial_{ikk} L(z_n)] + \eta^2 O(\mathbb{E}[\|\xi_n\|^6]) \\ &\phantom{=}\,
    +2 \Bigg[ \Big(0 + 0 + \frac{\eta^2 \sigma^2}{2} \sum_{j=1}^d \partial_{ijj} L(z_n) + \eta^2 \partial_i L(z_n) O(\mathbb{E}[\|\xi_n\|^3])  \Big) + \Big( 0 + 0 + 0 + 0 \Big) + \Big( 0 + 0 \Big) \\ &\phantom{=\ 2 \Bigg[} \,  + \frac{\eta^2 \sigma^2}{2} \left[ \sum_{j=1}^d \partial_{ijj} L(z_n) \right] O(\mathbb{E}[\|\xi_n\|^3]) \Bigg]
\end{align*}
By rearranging the summands, we obtain
\begin{align}
    &\operatorname{var}\left[z^{(i)}_{n+1} \vert z_n\right]
    = \notag \\
    &\quad \eta^2 [\partial_i L(z_n)]^2 
    + \eta^2 \sigma^2 \sum_{j=1}^d [\partial_{ij}^2 L(z_n)]^2 
    + \frac{\eta^2 \sigma^4}{4} \sum_{j \neq k} \left( \partial^3_{ijk} L(z_n) + [\partial_{ijj} L(z_n)] \cdot [\partial_{ikk} L(z_n)] \right)
    + \eta^2 \sigma^2 \sum_{j=1}^d \partial_{ijj} L(z_n) \notag \\ &\quad
    + \left[\frac{1}{2}\partial_i L(z_n) + \frac{\sigma^2}{2} \sum_{j=1}^d \partial_{ijj} L(z_n)\right]O(\mathbb{E}[\eta^2 \|\xi_n\|^3])
    + \left[\frac{1}{4} \sum_{j=1}^d \partial_{ijj} L(z_n) \right]O(\mathbb{E}[\eta^2 \|\xi_n\|^4])
    + O(\eta^2 \mathbb{E}[\|\xi_n\|^6]).
    \label{eq:conditional_variance_anti}
\end{align}
The above expression is the asymptotic (second order) expansion of the conditional variance of Anti-PGD, as $\sigma \to 0$.
It consists of two parts: Its first line contains summands which come from the first two moments of the distribution of $\xi$. Its second line contains the summands from the third, fourth, and sixth moment of the distribution of $\xi$. 
The precise size of the conditional variance will therefore depend on the first six moments of the noise distribution.
Without assuming more on the noise distribution, it is thus difficult to make further statements about the conditional variance.
By eyeballing Eq.~\eqref{eq:conditional_variance_anti}, it however seems likely that this variance is larger than the one from PGD; cf.~\eqref{eq:conditional_variance_vanilla}.

\section{Proof of Theorem~\ref{thm:implicit_reg}} \label{app:thm_implicit_reg}

\generalproof*

\begin{proof} 
First, we observe that Eq.~\eqref{eq:z_n_dyn_by_dim} implies~--~in the considered case of Bernoulli perturbations $\xi_n$ with $(\xi_i^n)^2 = \sigma^2$ a.s.~--~that
\begin{equation}
    z^{i}_{n+1} - z^{i}_{n} = - \underbrace{\partial_i \eta L(z_{n}) - \partial_i  \eta\frac{\sigma^2}{2}\sum_{j}\partial^2_{jj} L(z_{n})}_{= \eta \partial_i \tilde{L}(z_n)\ \text{(Regularized Gradient)}} -\underbrace{\eta\sum_{j}\partial^2_{ij} L(z_{n})\xi_{n}^j - \frac{\eta}{2} \partial_i \sum_{j\ne k}\partial^2_{jk} L(z_{n})\xi_{n}^j\xi_{n}^k}_{\text{Mean-zero Perturbation}} + \underbrace{O(\eta\sigma^3)}_{\text{Expansion error (small)}}. \label{eq:writing_out_z_n_for_Ber}
\end{equation}
Since we assumed that $L$ has $\beta$-Lipschitz continuous third-order partial derivatives, Theorem~2.1.5 from \citet{nesterov2018book} implies that $\tilde L$ is $\beta$-smooth. For some tensor $B$ which captures all summands from the mean-zero perturbation of Eq.~\eqref{eq:writing_out_z_n_for_Ber} we have
\begin{align*}
    \tilde L(z_{n+1}) &\le \tilde L(z_{n}) +\langle\nabla \tilde L(z_{n}), z_{n+1} - z_{n}\rangle + \frac{\beta}{2}\|z_{n+1} - z_{n}\|^2\\
    &\overset{\eqref{eq:writing_out_z_n_for_Ber}}{=} \tilde L(z_{n}) -\eta\langle\nabla \tilde L(z_{n}),  \nabla \tilde L(z_n) + B(z_n)\odot \xi_n + O(\sigma^3)\rangle + \frac{\beta\eta^2}{2}\|\nabla \tilde L(z_n) + B(z_n) \odot \xi_n + O(\sigma^3)\|^2\\
    &= \tilde L(z_{n}) -\eta\|\nabla \tilde L(z_{n})\|^2 + \frac{\beta\eta^2}{2}\|\nabla \tilde L(z_n)\|^2\\
    &\quad -\eta\langle\nabla \tilde L(z_{n}), B(z_n)\odot \xi_n\rangle + \eta\langle\nabla \tilde L(z_{n}), O(\sigma^3)\rangle\\
    &  \quad + \frac{\beta\eta^2}{2}\left[\|B(z_n) \odot \xi_n\|^2 +O(\sigma^6)+ 2 \langle \nabla \tilde L(z_n) + O(\sigma^3), B(z_n) \odot \xi_n\rangle + 2 \langle \nabla \tilde L(z_n), O(\sigma^3)\rangle\right]
\end{align*}

Taking the expectation, most of the terms cancel out and we get
\begin{align}
   \Exp[\tilde L(z_{n+1})] \le
& \Exp\left[\tilde L(z_{n}) -\eta\|\nabla \tilde L(z_{n})\|^2 + \frac{\beta\eta^2}{2}\|\nabla \tilde L(z_n)\|^2\right]\\
&\quad + O(\eta\sigma^3)\\
&  \quad + \frac{\beta\eta^2}{2}\left[O(\sigma^2) +O(\sigma^6)+ O(\sigma^3)\right]\\   
&=\Exp[\tilde L(z_{n})] -\left(\eta-\frac{\beta\eta^2}{2}\right)\Exp[\|\nabla \tilde L(z_{n})\|^2] + O(\eta^2\sigma^2) +  O(\eta\sigma^3).
\end{align}
Note that, under the assumption of positive $\eta$,  $\left(\eta-\frac{\beta\eta^2}{2}\right)>0$ implies $0<\eta\le \frac{2}{\beta}$. If we further require  $\left(\eta-\frac{\beta\eta^2}{2}\right)\ge\frac{\eta}{2}$, we get the condition $0<\eta\le \frac{1}{\beta}$. 

Under this condition, the following upper bound holds\footnote{Outlined in blue color below is the updated version of the proof of the theorem.}:

\begin{equation}
   \Exp[\tilde L(z_{n+1})] \le \Exp[\tilde L(z_{n})] -\frac{\eta}{2}\Exp[\|\nabla \tilde L(z_{n})\|^2] + O(\eta^2\sigma^2) +  O(\eta\sigma^3).
\end{equation}

Rearranging the terms we get

\begin{equation}
    \frac{\eta}{2}\Exp[\|\nabla \tilde L(z_{n})\|^2] \le \Exp[\tilde L(z_{n})] - \Exp[\tilde L(z_{n+1})] + O(\eta^2\sigma^2) +  O(\eta\sigma^3).
\end{equation}
After a division by $\eta/2>0$ we arrive at

\begin{equation}
    \Exp[\|\nabla \tilde L(z_{n})\|^2] \le \frac{2}{\eta} \left(\Exp[\tilde L(z_{n})] - \Exp[\tilde L(z_{n+1})]\right) + O(\eta\sigma^2) +  O(\sigma^3).
\end{equation}

Summing all terms, we arrive at the following:

\begin{equation}
    \Exp\left[\sum_{n=0}^{N-1}\|\nabla \tilde L(z_{n})\|^2\right] \le \frac{2}{\eta} \left(\Exp[\tilde L(z_{0})] - \Exp[\tilde L(z_{N})]\right) + O(N\eta\sigma^2) +  O(N\sigma^3).
\end{equation}

After dividing by $N$ we get

\begin{equation}
    \Exp\left[\frac{1}{N}\sum_{n=0}^{N-1}\|\nabla \tilde L(z_{n})\|^2\right] \le \frac{2}{\eta N} \left(\Exp[\tilde L(z_{0})] - \Exp[\tilde L(z_{N})]\right) + O(\eta\sigma^2) +  O(\sigma^3).
\end{equation}

Note that $\tilde L(z_{N})\ge\tilde L^* = \min_z \tilde L(z) > -\infty$. This is because $L$ is lower bounded and also $tr(\nabla^2 L)$ is lower bounded since we assumed $L$ has Lipschitz gradients. Therefore, $\Exp[\tilde L(z_{0})] - \Exp[\tilde L(z_{N})] = O(1)$ and we conclude that

\begin{equation}
    \Exp\left[\frac{1}{N}\sum_{n=0}^{N-1}\|\nabla \tilde L(z_{n})\|^2\right] \le O(\eta^{-1} N^{-1}) + O(\eta\sigma^2) +  O(\sigma^3).
\end{equation}


If we set $\eta = O(\epsilon/\sigma^2)$, then $O(\eta\sigma^2) = O(\epsilon)$. Further, under the choice $\eta = \Theta(\epsilon/\sigma^2)$, we also have $O(\eta^{-1} N^{-1}) = O(\sigma^2 \epsilon^{-1}N^{-1})$. By setting $N = \Omega(\sigma^2\epsilon^{-2})$, we ensure that $O(\eta^{-1} N^{-1}) = O(\epsilon)$. 

All in all, we get $\Exp\left[\frac{1}{N}\sum_{n=0}^{N-1}\|\nabla \tilde L(z_{n})\|^2\right] \le O(\epsilon) + O(\sigma^3)$.

\end{proof}

\section{Proof of Theorem~\ref{thm:main_tube}}  \label{app:wv_proof}

To proof Theorem~\ref{thm:main_tube}, we first need some preliminary preparation; in doing so, we will also provide some intuition for the reader.
The main proof follows afterwards, in \S\ref{app:proof_of_main_result}.

Consider the problem of minimizing the cost function
\begin{equation}
    L(u,v) = \frac{1}{2} v^2 \|u\|^2,
\end{equation}
where $\|\cdot\|$ is the Euclidean norm, $v\in\mathbb{R}$, and $u\in\mathbb{R}^d$. 
To minimize the loss, we use perturbed gradient descent~(PGD), with noise injection. Note that any point where $v=0$ or $\|u\|^2=0$ minimizes the loss. By the considerations in the section above, we want to find a solution $(u,v)$ where $\|u\|$ is small~(i.e. a solution with low curvature). We show that, while standard noise injection does not necessarily induce this bias on the dynamics, injection of anticorrelated noise does. This show that anticorrelated noise effectively minimizes the trace of the Hessian:
\begin{equation}
    \tr(\nabla^2 L(u,v)) = d v^2 + \|u\|^2.
\end{equation}

\paragraph{Preliminary considerations.} Let us start by writing down the update in discrete-time. Recall that the gradient is $(v^2 u, \|u\|^2 v)$, hence:
\begin{align}
    &u_{k+1} = (1 - \eta v_k^2) \cdot u_k +  \varepsilon_{k}^u
 \label{eq:eq_tube_w}\\
    &v_{k+1} = (1 - \eta \|u_k\|^2) \cdot v_k +  \varepsilon_{k}^v
    \label{eq:eq_tube_v}
\end{align}
where $ \varepsilon_{k}^u\in\R^{d}$ and $ \varepsilon_{k}^v\in\R$ are the noise variables.

\begin{enumerate}
    \item For stability~(in the noiseless setting), we need $\eta\le\frac{2}{\max\{v_k^2,\|u_k\|^2\}}$.
    \item Starting from a big $\|u\|$ and any $v$, under noiseless GD, since $d\gg 1$, we converge to $(u_0,0)$, with $\|u_0\|:=D\gg 1$ 
\end{enumerate}

\noindent The key to the proof of effectiveness of anticorrelated noise, compared to uncorrelated noise, relies on the following observation:
\ \\

\noindent\textbf{Empirical Observation:} for the widening valley $L(u,v) = \frac{1}{2}v^2 \|u\|^2$, if we only perturb the $v$ coordinate with \textit{any noise} then we get to a wide minimum.
\ \\

\noindent\textbf{Why? Intuition behind the proof.} Well, of course this is the case! If one knows is advance which direction to move in order to pick up a signal, then \textit{les jeux sont faits}. The problem is that in order to perturb this direction --- we need to perturb \textit{all directions}, and this leads to ``getting lost'' if the noise is not controlled (i.e. does not have an attraction force to the origin). This also motivates why the effect gets more intense as the dimension $d$ increases: there is a lot of bias added, which drives us away from good minima.
\ \\

\noindent\textbf{Plan:} To show the result, we follow the following procedure:
\begin{enumerate}
    \item Starting from $(u_0,0)$, we start injecting noise and want to reach $(\tilde u, \tilde v)$ such that $\|\tilde u\|^2 = \alpha D$ and $0<\alpha<1$. We want to show here a difference in behavior under different noise correlation.
    \item We proceed by contradiction: starting from $(\tilde u, \tilde v)$, we assume that $\|u_k\|^2 \ge \alpha D$ for all $k\ge 0$~($\alpha\in(0,1)$). Under injection of anticorrelated noise, we show that \textit{this leads to a contradiction} --- i.e. that the dynamics substantially decreases the trace of the Hessian: $\|u_\infty\|^2 < \alpha D$ (worst-case upper bound). Crucially, we also show that the hypothesis does not lead to any contradiction under standard noise injection --- i.e. without anticorrelation we do not significantly decrease the trace of the Hessian. More specifically, we show that $\lim_{n \to \infty} \Exp\left[ \| u_n \|^2 \right] \geq D/\alpha$ under uncorrelated noise injection (worst-case lower bound).
    \item To simplify the computations, assume that coordinate-wise the noise is a result of a Bernoulli$(1/2)$ perturbation $(\xi_k)_i \in \{-\sigma,\sigma\}$. The injected noise is then either $\varepsilon_k =\xi_k$ or $\varepsilon_k =\xi_k-\xi_{k-1}$, for PGD and Anti-PGD respectively.
\end{enumerate}

\subsection{Some Useful Lemmata}

This section is pretty technical, hence the reader can skip the proof on a first read. \textit{The meaning behind the bounds we derive and a numerical verification can be found in Figure~\ref{fig:tube_lemma_verification}}.

\noindent We start by recalling the variation of constants formula, which we will heavily use along the proof. We also extend this to the anticorrelated setting.

\begin{lemma}[Variations of constants formula]
\label{lemma:variations}
Let $w\in\mathbb{R}^d$ evolve with time-varying linear dynamics $w_{k+1} = A_k w_k + \varepsilon_k$, where $A_k\in\mathbb{R}^{d\times d}$ and $\varepsilon_k\in \mathbb{R}^{d}$ for all $k$. Then, with the convention that $\prod_{j=k+1}^{k} A_j = 1$, 
\begin{equation}
    w_{k+1} = \left(\prod_{j=0}^{k}A_j\right)w_0 +\sum_{i=0}^{k}\left(\prod_{j=i+1}^{k} A_j\right)\varepsilon_i.
\end{equation}
\end{lemma}
\begin{proof}
For $k=1$ we get $w_1 = A_0 w_0 + \varepsilon_0$. The induction step yields
\begin{align}
    w_{k+1} &= A_k\left(\left( \prod_{j=0}^{k-1}A_j\right)w_0 +\sum_{i=0}^{k-1}\left(\prod_{j=i+1}^{k-1} A_j\right)\varepsilon_i\right) + \varepsilon_k.\\
    &= \left( \prod_{j=0}^{k}A_j\right)w_0 +\sum_{i=0}^{k-1}A_k\left(\prod_{j=i+1}^{k-1} A_j\right)\varepsilon_i + \varepsilon_k.\\
    &= \left( \prod_{j=0}^{k}A_j\right)w_0 +\sum_{i=0}^{k-1}\left(\prod_{j=i+1}^{k} A_j\right)\varepsilon_i + \left(\prod_{j=k+1}^{k} A_j\right)\varepsilon_k.\\
    &= \left( \prod_{j=0}^{k}A_j\right)w_0 +\sum_{i=0}^{k}\left(\prod_{j=i+1}^{k} A_j\right)\varepsilon_i.
\end{align}
This completes the proof of the variations of constants formula.
\end{proof}

\noindent We extend this formula to the anticorrelated case, where $\varepsilon_k$ has some additional structure.

\begin{corollary}[Anticorrelated variations of constants formula]
\label{cor:anti_variations}
Under the same setting of Lemma~\ref{lemma:variations}, if there exist a family of vectors $\{\xi_k\}$ such that $\varepsilon_0 = \xi_0$ and $\varepsilon_k = \xi_k-\xi_{k-1}$, then
\begin{equation}
    w_{k+1} = \left(\prod_{j=0}^{k}A_j\right)w_0 +\xi_k  +\sum_{i=0}^{k-1}(A_{i+1}-I)\left(\prod_{j=i+2}^{k}  A_j\right)\xi_{i}.
\end{equation}
\end{corollary}

\begin{proof}
We have, by direct computation:
\begin{align}
    w_{k+1} &= \left(\prod_{j=0}^{k}A_j\right)w_0 + \left(\prod_{j=1}^{k} A_j\right)\xi_0 +\sum_{i=1}^{k}\left(\prod_{j=i+1}^{k} A_j\right)\xi_i - \sum_{i=1}^{k}\left(\prod_{j=i+1}^{k}  A_j\right)\xi_{i-1}\\
    &= \left(\prod_{j=0}^{k}A_j\right)w_0+ \left(\prod_{j=1}^{k} A_j\right)\xi_0   +\sum_{i=1}^{k-1}\left(\prod_{j=i+1}^{k} A_j\right)\xi_i  +\xi_k  - \sum_{i=0}^{k-1}\left(\prod_{j=i+2}^{k}  A_j\right)\xi_{i}\\
    &= \left(\prod_{j=0}^{k}A_j\right)w_0 +\xi_k  +\sum_{i=0}^{k-1}\left(\prod_{j=i+1}^{k} A_j\right)\xi_i  - \sum_{i=0}^{k-1}\left(\prod_{j=i+2}^{k}  A_j\right)\xi_{i}\\
    &= \left(\prod_{j=0}^{k}A_j\right)w_0 +\xi_k  +\sum_{i=0}^{k-1}\left[\left(\prod_{j=i+1}^{k} A_j\right)-\left(\prod_{j=i+2}^{k}  A_j\right)\right]\xi_{i}.
\end{align}
\end{proof}
\begin{remark}
If $A_i=I$ for all $i$, then the last summand is zero. This showcases the effect of anticorrelation: noise cancellation under noise accumulation.
\end{remark}

\subsubsection[]{Expectation Quantities under Deterministic $\rho_k$}
Using the variation of constants formula, we can write the dynamics of the second moment of stochastic linear time-varying dynamical systems, with either standard or anticorrelated noise.

\begin{proposition}[An It{\^o}-like formula]
Let $w\in\mathbb{R}^d$ evolve with time-varying linear dynamics $w_{k+1} = A_k w_k + \varepsilon_k$, where $A_k\in\mathbb{R}^{d\times d}$ and $\varepsilon_k\in \mathbb{R}^{d}$ for all $k$. Let $\{\xi_k\}$ be a family of uncorrelated zero-mean $d$-dimensional random variables with variance $\mathbb{E}[\|\xi_k\|^2 ]= d\sigma^2$~(dependency on the dimension because additivity of squared norm). Consider $\varepsilon_0 = \xi_0$ and $\varepsilon_k = \xi_k-\xi_{k-1}$ for all $k\ge 1$. Further, assume that $A_k = \rho_k I$ for all $k$ (i.e. $A_k$ is a multiple of the identity), with $\rho_k\in\mathbb{R}$ a deterministic quantity. Then, with the convention that $\prod_{j=k+1}^{k} A_j = 1$, we have
\begin{equation}
    \mathbb{E}[\|w_{k+1}\|^2] = \left(\prod_{j=0}^{k}\rho_j^2\right) \|w_0\|^2 + \left(1+\sum_{i=0}^{k-1}\left[ (1-\rho_{i+1})^2 \prod_{j=i+2}^k \rho_j^2\right]\right)d \sigma^2. 
\end{equation}
Instead, if $\varepsilon_k = \xi_k$ for all $k$~(standard noise injection) we have
\begin{equation}
    \mathbb{E}[\|w_{k+1}\|^2] = \left(\prod_{j=0}^{k}\rho_j^2\right) \|w_0\|^2 + \sum_{i=0}^{k}\left(\prod_{j=i+1}^k \rho_j^2\right)d \sigma^2. 
\end{equation}
\label{prop:ito_discrete}
\end{proposition}

\begin{proof}
    Using independence of the $\{\xi_k\}$ family, we obtain for the anticorrelated case:
    
    \begin{align}
    \mathbb{E}[w_{k+1}^\top w_{k+1}] &= \left(\prod_{j=0}^{k}\rho_j\right)^2\|w_0\|^2 +\mathbb{E}[\|\xi_k\|^2]  +\sum_{i=0}^{k-1}(\rho_{i+1}-1)^2\left(\prod_{j=i+2}^{k}  \rho_j\right)^2\mathbb{E}[\|\xi_{i}\|^2]\\
    &= \left(\prod_{j=0}^{k}\rho_j\right)^2\|w_0\|^2 +\sigma^2  +\sum_{i=0}^{k-1}(\rho_{i+1}-1)^2\left(\prod_{j=i+2}^{k}  \rho_j^2\right)\sigma^2,
    \end{align}
    where we used the fact that the $\xi_k$ are not correlated. The case $\varepsilon_k = \xi_k$ is similar and therefore left to the reader.
\end{proof}

\begin{corollary}
In the setting of Proposition~\ref{prop:ito_discrete}, assume $\rho_j = \rho \in (0,1)$ is constant for all $j$. If $\varepsilon_0 = \xi_0$ and $\varepsilon_k = \xi_k-\xi_{k-1}$ for all $k\ge 1$ then

\begin{equation}
    \mathbb{E}[\|w_{k+1}\|^2] = \rho^{2(k+1)} \|w_0\|^2 + \left(1+\frac{(1-\rho)^2(1-\rho^{2 (k+1)})}{1-\rho^2}\right)d \sigma^2 \quad\overset{\infty}{\longrightarrow}\quad  \frac{2}{1+\rho} d \sigma^2. 
\end{equation}

\noindent Instead, if $\varepsilon_k = \xi_k$ for all $k$~(standard noise injection) we have

\begin{equation}
    \mathbb{E}[\|w_{k+1}\|^2] = \rho^{2(k+1)} \|w_0\|^2 + \frac{1-\rho^{2 (k+1)}}{1-\rho^2} d \sigma^2\quad \overset{\infty}{\longrightarrow}\quad  \frac{1}{1-\rho^2} d \sigma^2.
\end{equation}
\label{cor:ito_discrete_const}
\end{corollary}
\begin{proof}
Simple application of the formula for geometric series. Numerical verification in Figure~\ref{fig:tube_lemma_verification}.
\end{proof}

\begin{remark}
\label{rmk:reversed trend}
Note that the corollary has a clear interpretation: if $\rho$ is between zero and one, we experience striking difference between uncorrelated and anticorrelated noise. If $\rho$ increases, the total accumulated anticorrelated noise decreases.\footnote{$\frac{2}{1+\rho}$ is a decreasing function or $\rho$, while $1/(1-\rho^2)$ is increasing.} This trend is reversed for normal noise injection: as $\rho\to 1$ the total accumulated variance explodes. Numerical verification can be found in Figure~\ref{fig:tube_lemma_verification}.
\end{remark}

\begin{figure}[ht]
    \centering
    \includegraphics[width=0.32\textwidth]{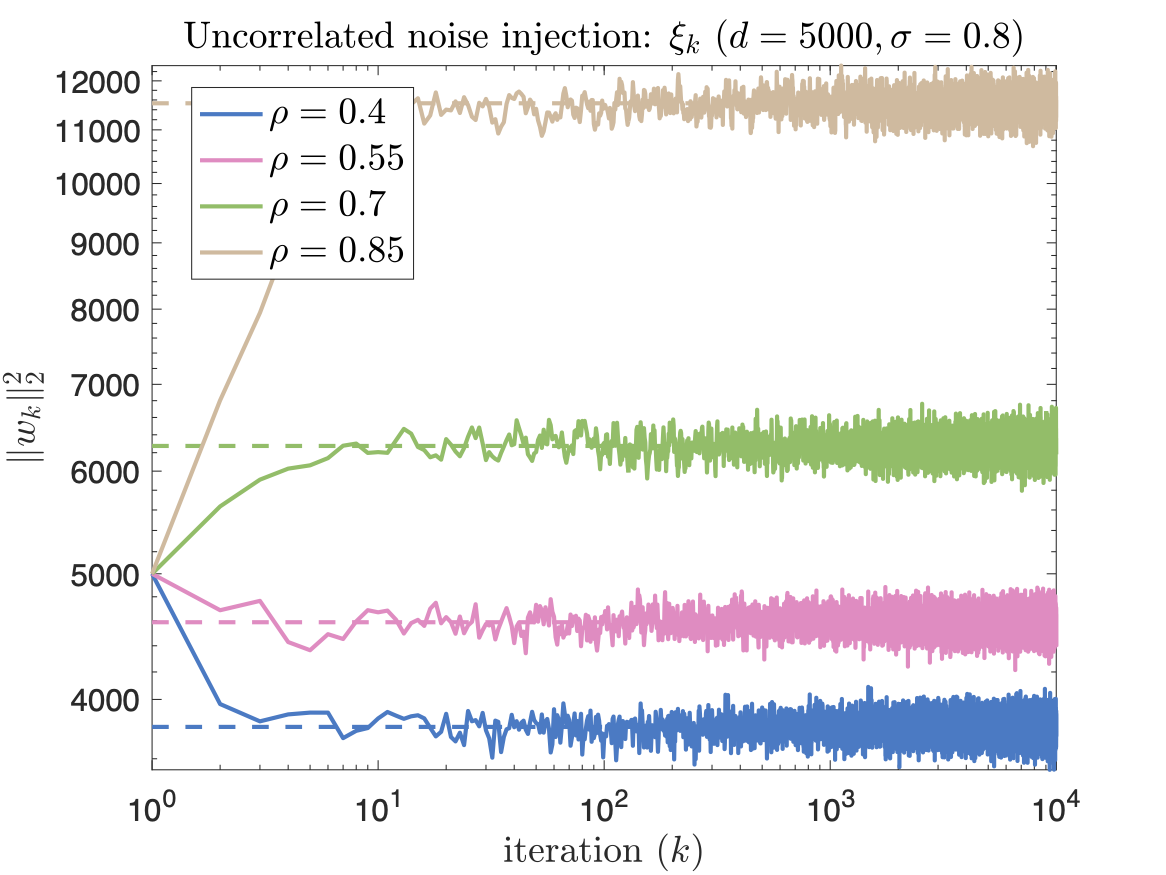}
    \includegraphics[width=0.32\textwidth]{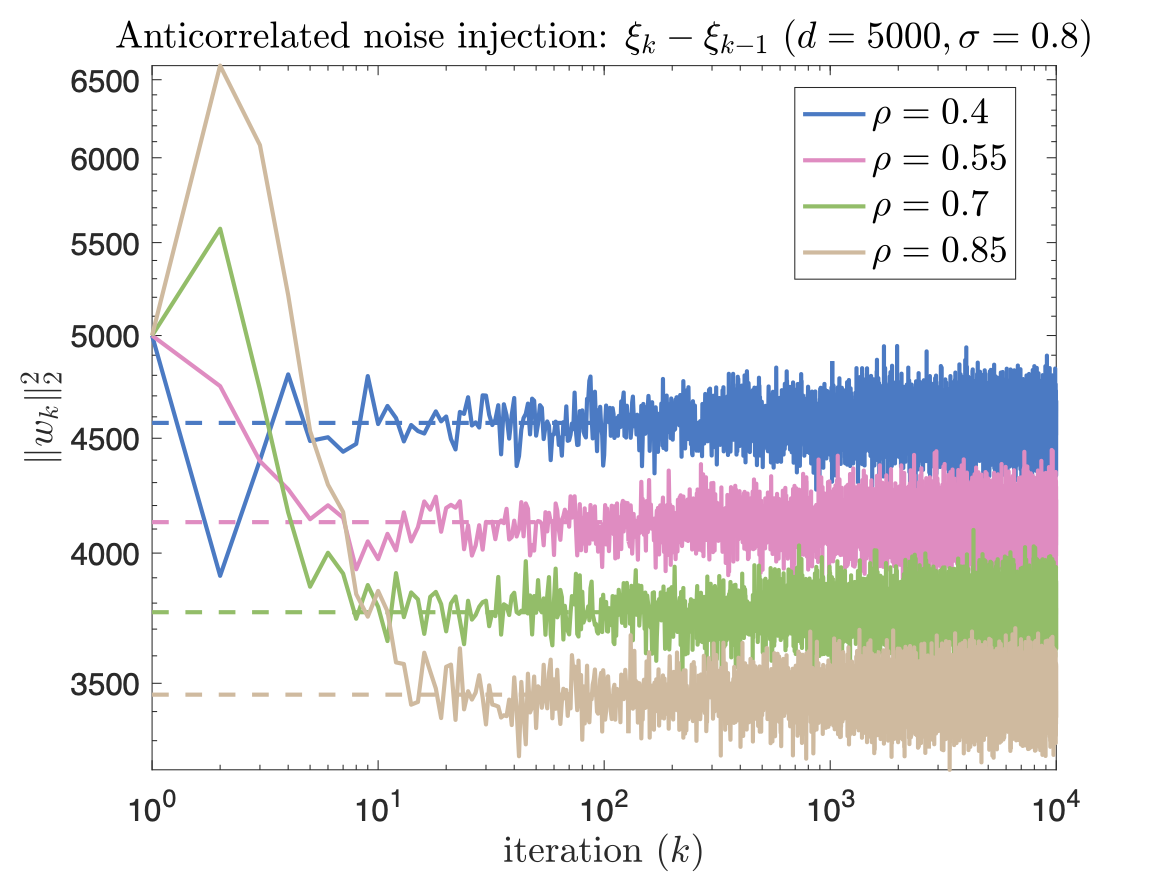}
    \includegraphics[width=0.32\textwidth]{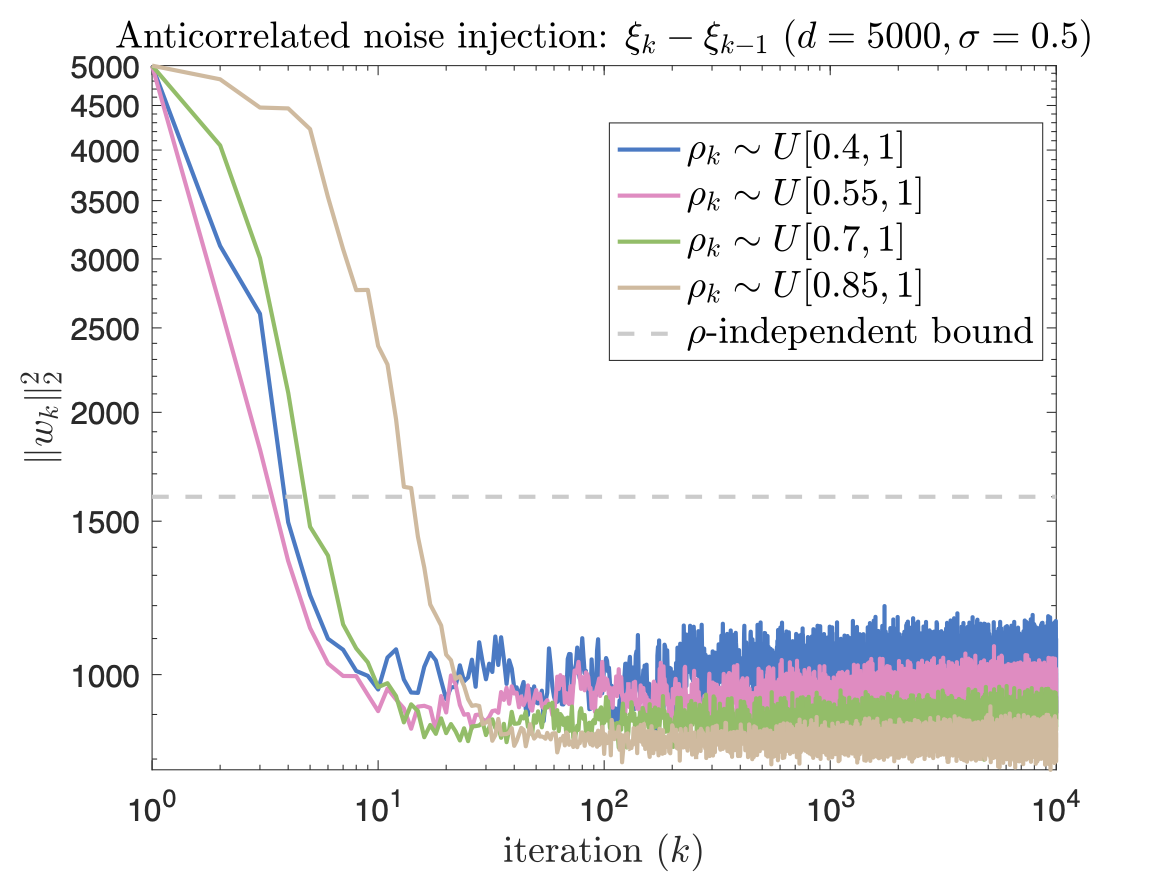}
    \vspace{-3mm}
    \caption{Numerical verification of our final result that we will use in the proof, i.e. Corollary~\ref{cor:ito_discrete_const}~(first and second panel) and Proposition~\ref{prop:ito_discrete_all_rho}~(last panel). The dashed lines indicate our predicted value~(in expectation) by the theory. In the right-most plot, we sample $\rho_k$ at each iteration uniformly on an interval.}
    \label{fig:tube_lemma_verification}
\end{figure}

\subsubsection[]{Dealing with Potential Stochasticity in the $\rho_k$}
For the proof in the next subsection, we need to deal with stochastic $\rho_k$, which are only specified up to an interval.

\begin{proposition}[Limit bound on second moment for anticorrelated noise]
\label{prop:ito_discrete_all_rho}
Let $w\in\mathbb{R}^d$ evolve with time-varying linear dynamics $w_{k+1} = A_k w_k + \varepsilon_k$, where $A_k\in\mathbb{R}^{d\times d}$ and $\varepsilon_k\in \mathbb{R}^{d}$ for all $k$. Let $\{\xi_k\}$ be a family of uncorrelated zero-mean $d$-dimensional random variables with variance $\mathbb{E}[\|\xi_k\|^2 ]= d\sigma^2$~(dependency on the dimension because additivity of squared norm). Consider $\varepsilon_0 = \xi_0$ and $\varepsilon_k = \xi_k-\xi_{k-1}$ for all $k\ge 1$. Further, assume that $A_k = \rho_k I$ for all $k$ (i.e. $A_k$ is a multiple of the identity) and that $\rho_k\in[0,1]$ for all $k$. Assume that the probability of $\rho_k<1$ is non-zero, i.e. that $\rho_k\ne 1$ with non-vanishing probability. Then, we have
\begin{equation}
\lim_{k\to\infty}\mathbb{E}[\|w_{k+1}\|^2]\le 2d\sigma^2.
\end{equation}
\end{proposition}

\begin{proof}
The proof is based on an induction argument, starting from the equation in Proposition~\ref{prop:ito_discrete}:
\begin{equation}
    \mathbb{E}[\|w_{k+1}\|^2] = \left(\prod_{j=0}^{k}\rho_j^2\right) \|w_0\|^2 + \left(1+\sum_{i=0}^{k-1}\left[ (1-\rho_{i+1})^2 \prod_{j=i+2}^k \rho_j^2\right]\right)d \sigma^2. 
\end{equation}
First, note that by assumption on $\rho_k$ the first term vanishes as $k\to\infty$. We just have to deal with the second term. Specifically, we want to show that for whatever sequence $\rho_k\in(0,1)$ we have
\begin{equation}
    \nu_k = \sum_{i=0}^{k-1} (1-\rho_{i+1})^2 \left(\prod_{j=i+2}^k \rho_j^2\right)\le1, \ \quad \forall k\ge 0.
\end{equation}

A fundamental observation, is that the term can be written in a recursive form. Indeed,
\begin{align}
    \nu_k &= \sum_{i=0}^{k-1} \left[(1-\rho_{i+1})^2 \left(\prod_{j=i+2}^k \rho_j^2\right)\right]\\
    &= \sum_{i=0}^{k-2} \left[(1-\rho_{i+1})^2 \left(\prod_{j=i+2}^k \rho_j^2\right)\right] + (1-\rho_{k})^2\\
    &= \rho_k \sum_{i=0}^{k-2} \left[(1-\rho_{i+1})^2 \left(\prod_{j=i+2}^{k-1} \rho_j^2\right)\right] + (1-\rho_{k})^2\\
    & = \rho_k^2\nu_{k-1} + (1-\rho_k)^2,
\end{align}
where the second equality follows from the fact that, as previously noted, our notation implies $\prod_{j=k+1}^k \rho_k^2 = 1$, for all $k$. Let us now proceed again by induction to show that $v_k\in(0,1)$ for all $k\ge 0$. Note that trivially $v_0=0$. Let's proceed with the inductive step:
\begin{equation}
    \nu_k = \rho_k^2\nu_{k-1} + (1-\rho_k)^2 = (\nu_{k-1}+1)\rho_k^2 - 2\rho_k + 1.
\end{equation}
This quantity is less then one if and only if
\begin{equation}
 (\nu_{k-1}+1)\rho_k^2 \le 2\rho_k.
\end{equation}
Note that this is satisfied since $\nu_{k-1}+1\le2$, and $\rho_k^2 \le \rho_k$ since $\rho_k\in(0,1)$. The result follows. 
\end{proof}
A numerical verification of this result can be found in Figure~\ref{fig:tube_lemma_verification}.

\subsection{Proof of the Main Result} \label{app:proof_of_main_result}
Using the results from the last subsection, we are now ready to show the main theorem for optimization of the widening valley under noise injection.

\tube*

\begin{figure}[ht]
    \centering
    \includegraphics[width=0.4\textwidth]{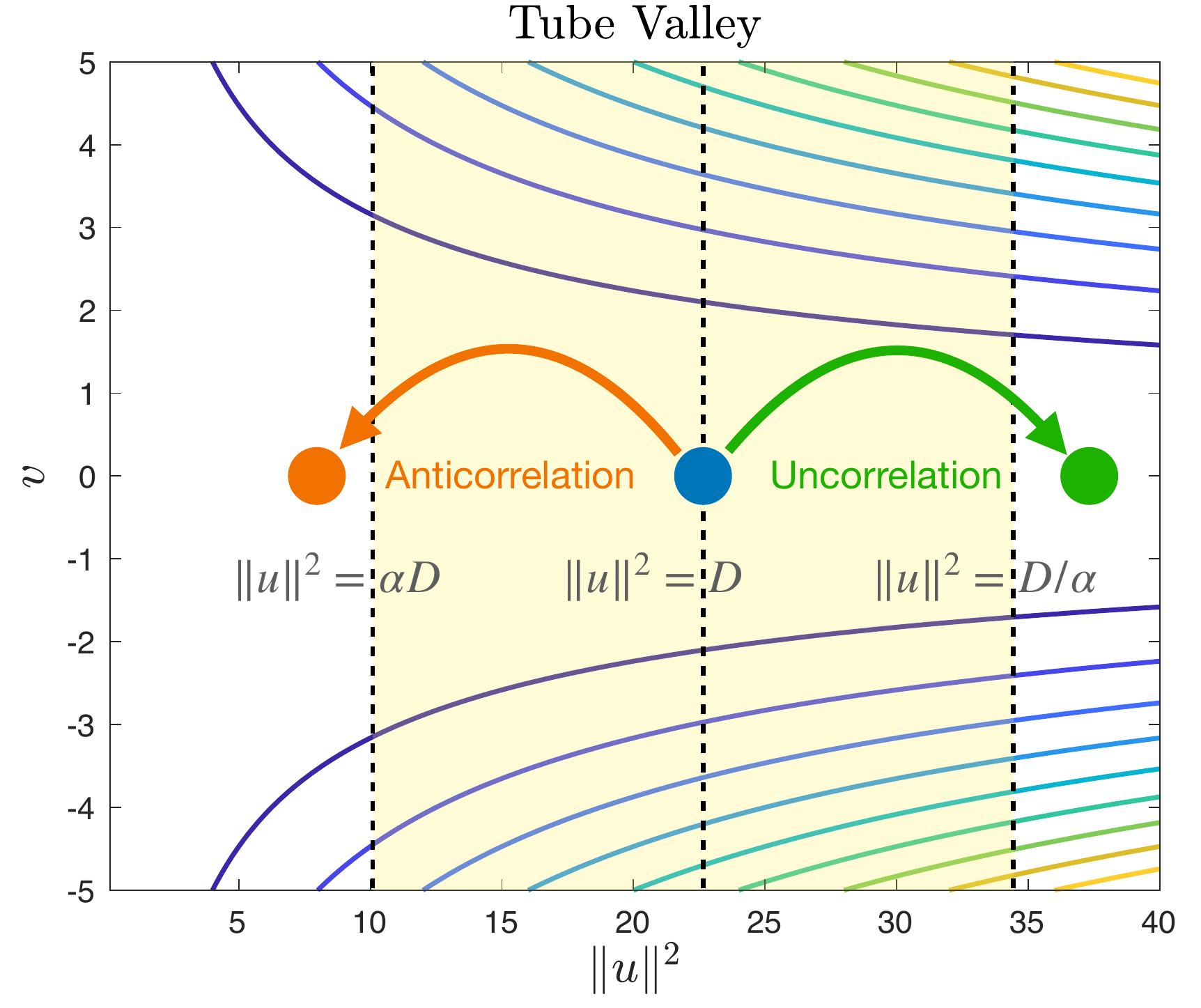}
    \vspace{-3mm}
    \caption{The sketch on the left illustrates the intuition behind the result in~\ref{thm:main_tube}.}
    \label{fig:sketch_proof}
\end{figure}

\begin{proof}
As above, we denote the perturbations by $\varepsilon_k$; i.e., $\varepsilon_k = \xi_k$ for PGD and $\varepsilon_k = \xi_{k+1} - \xi_k$ for Anti-PGD.
Let us start by inspecting the equation
\begin{equation}
    u_{k+1} = (1 - \eta v_k^2) \cdot u_k + \varepsilon_{k}^u,
\end{equation}
where $\varepsilon_{k}^u\in\mathbb{R}^d$ is the projection of the noise $\varepsilon_{k}$ to the first $d$ coordinates.
It is clear that the optimal strategy of making $\|w\|$ small is to increase $|v|$, so to sample nearby points and pick up the gradient. The greater $v$ is in norm, the better. We can increase the norm of $v$ by heavy noise injection~(second equation). However, too much noise also increases $\varepsilon_{k}^u$, which acts adversarially to the decrease of $\|w\|$~(error accumulation increases the Euclidean norm in expectation).

\paragraph{Choice of stepsize and operating region.} We start by motivating the choice of stepsize $\eta =\frac{\alpha}{2D}$. Starting from the point $(u_0,0)$ with $\|u_0\|^2=D>0$, we consider the operating landscape region $\alpha D<\|u_k\|^2<D/\alpha$, with $\alpha\in(0,1)$. We want to show that while standard noise injection makes the process exit the region from the right~($D/\alpha$ side, see Figure~\ref{fig:sketch_proof}), anticorrelated noise injection makes the process exit the region from the left~($\alpha D$ side). In this region, named $\mathcal{D}_\alpha$, the maximal allowed learning rate is $\eta\le\frac{2}{\max_{\mathcal{D}_\alpha}\{v^2,\|u\|^2\}}$. Since $v$ stays small~(we are going to check this later in great detail), we select the stepsize $\eta\le\frac{1}{2\max_{\mathcal{D}_\alpha}\|u\|^2}=\frac{\alpha}{2D}$ --- which guarantees stability in expectaction, i.e. without noise injection~(even allowing for some slack).

\paragraph[]{Lower bound for uncorrelated noise~$(\boldsymbol{\varepsilon_k =\xi_k})$.} For this case, we have to show that standard noise injection cannot possibly work for reaching $\alpha D$, therefore we have to put ourselves in \textit{the best case scenario} for PGD: that is, we have to provide an uniform upper bound for $v_k$ under the second equation~(i.e. the equation for $v$) and show that this is not enough for a substantial decrease in $\|u\|$. In the next paragraph~(anticorrelated noise), we instead have to put ourselves in the \textit{worst case scenario} --- i.e. a lower bound for $|v|$--- and show that this is still enough for anticorrelated noise to yield a substantial decrease in $\|u\|$.

To start, let us then look at the second equation:
\begin{equation}
v_{k+1} = (1 - \eta \|u_k\|^2) \cdot v_k + \varepsilon^v_k,
\end{equation}
where $\varepsilon^v_k$ is the $(d+1)$-th component of $\varepsilon_k$. Since we start from $v_0=0$, the equation is completely dominated by noise, and is strongly mean reverting~(i.e. $v$ is effectively bounded). Indeed, since $\eta = \frac{\alpha}{2D}$ and $\|u_k\|^2\in\left(\alpha D,D/\alpha\right)$ by assumption, we have

\begin{equation}
|v_{k+1}| \le \max\left\{1 - \frac{\alpha^2}{2},\frac{1}{2}\right\} \cdot |v_k| + \sigma = \left(1 - \frac{\alpha^2}{2}\right) |v_k| +\sigma.
\end{equation}
where we used the fact that $|\varepsilon^v_k|=\sigma$ and that $\alpha^2\in(0,1)$. By induction, the last inequality yields that, starting from $v_0=0$, we have
\begin{equation}
    |v_k|\le v_{\max}:=\frac{2\sigma}{\alpha^2},\quad \forall k\ge 0.
\end{equation}
Hence, we found the ``best case scenario'' for the $w$ equation: $|v_k| = \frac{2\sigma}{\alpha^2}$, for all $k$. This gives $w$ the best decrease rate possible.\footnote{Note that noise injection in $u$ is independent of $v$, therefore to minimize $\|u\|$ we need the shrinking factor to be as large as possible. We note that using this bound is precise: with probability one we are in the best scenario~(we are finding a lower bound).}~However, we need to check this value $v_{\max}$ is such that the equation for $w$ is indeed stable~(we promised this to the reader in the last paragraph). We recall that this equation is $w_{k+1} = (1 - \eta v_k^2) \cdot w_k + \varepsilon_{k}^u$. Let us require $(1 - \eta v_k^2)\in(0,1)$, for this we need $1/v_{\max}^2>\eta = \frac{\alpha}{2D}$. Therefore, we need
\begin{equation}
    \frac{1}{v_{\max}^2} = \frac{\alpha^4}{4\sigma^2}\ge \eta = \frac{\alpha}{2D}\quad\implies\quad \sigma^2\le \alpha^3 D/2.
\end{equation}
This is guaranteed by assumption. To proceed, we substitute $v_{\max}$ into the first equation to get

\begin{equation}
    u_{k+1} = (1 - \eta \frac{4\sigma^2}{\alpha^4}) \cdot u_k + \varepsilon_{k}^u = \frac{\alpha^3 D- 2\sigma^2}{\alpha^3 D} u_k + \varepsilon_{k}^u. 
\end{equation}

Let us call $\rho := \left(\frac{\alpha^3 D- 2\sigma^2}{\alpha^3 D}\right)\in(0,1)$ the (best case) shrinking factor. Since $\varepsilon_{i}^u$ is zero-mean, computing the expected value of $\|u_{k}\|^2$ leads to the following limit by Corollary~\ref{cor:ito_discrete_const}:

\begin{equation}
    \lim_{k\to\infty} \mathbb{E}[\|u_{k}\|^2] = \frac{d\sigma^2}{1-\rho^2} = \frac{d D^2 \alpha^6 }{2 D \alpha^3-2\sigma^2}.
\end{equation}
where we assumed $\sigma^2$ strictly positive. Note that this limit is a monotonically increasing function of $\sigma^2\in(0,D \alpha^3/2)$. Hence, we get

\begin{equation}
    \lim_{k\to\infty} \mathbb{E}[\|u_{k}\|^2] \in \left(\frac{d D \alpha^3}{2}, d D \alpha^3\right)
    \label{eq:bound_tube_bm}
\end{equation}
\begin{remark}[Phase transition]
Note that, for $\sigma$ exactly 0, the limit is instead $\|u_0\|^2 = D$. Instead, for any small noise the process will grow up until at least $dD\alpha^2/2$. This might seem weird at first --- but recall that there is an interaction between noise scale and our best-case scenario bound for $v$: they both depend on $\sigma$. This causes a cancellation effect and a transition in behavior at $\sigma=0$.
\end{remark}
 Last, we need to show that this lower bound on $\lim_{k\to\infty} \mathbb{E}[\|u_{k}\|^2]$ coincides with~(or is bigger than) the right boundary of the operating region in Figure~\ref{fig:sketch_proof}. To do this we set:
\begin{equation}
    \frac{d D \alpha^3}{2} \ge \frac{D}{\alpha}\quad \implies \quad d\ge \frac{2}{\alpha^4}.
\end{equation}
This concludes the proof of Eq.~\eqref{eq:thm_wv_PGD}.

\paragraph[]{Upper bound for anticorrelated noise~$(\boldsymbol{\varepsilon_k =\xi_k-\xi_{k-1}})$.}
We consider anticorrelated noise injection $(\xi_k)_i \in \{-\sigma,\sigma\}$, and $\varepsilon_k = (\varepsilon_k^u,\varepsilon_k^v) =\xi_k-\xi_{k-1}$. Again, let us first look at the second equation:
\begin{equation}
v_{k+1} = (1 - \eta \|u_k\|^2) \cdot v_k + \varepsilon^v_k.
\end{equation}
Since by hypothesis  $\eta = \frac{\alpha}{2D}$ and $\alpha D \le \|u_k\|^2\le D/\alpha$, we have $(1 - \eta \|u_k\|^2)\in\left(1-\frac{\alpha^2}{2}, \frac{1}{2}\right)$. Clearly, we have that, for any $k\ge 1$ $v_{k}^2$, is non-zero with a non-vanishing probability. For ~(noiseless) stability, we also need an upper bound on $|v_k|$. An easy~(yet absolutely not tight) upper bound is the following:
\begin{equation} \label{eq:for_ind_vk}
|v_{k+1}| \le \frac{1}{2} |v_k| + 2\sigma \varepsilon^v_k.
\end{equation}
where we simply used the absolute value subadditivity and the fact that $|\varepsilon^v_k|\le |\xi_k|+|\xi_{k-1}| =2\sigma$. Note that the equation directly yields by induction $|v_k|\le 4\sigma$ for all $k\ge 0$.

Let us now deal with the equation for $w$.
\begin{equation}
    u_{k+1} = (1 - \eta v_k^2) \cdot u_k + \varepsilon_{k}^u 
\end{equation}
For this equation, we would want all the coefficients $\rho_k :=1 - \eta v_k^2$ to be between $0$ and $1$ --- i.e. we need to check that $v$ is indeed not too big. Since $|v_k|\le 4\sigma$ for all $k\ge 0$, we have the requirement $1 - \frac{\alpha}{2D} 16\sigma^2>0$, which implies $\sigma^2\le \frac{D}{8\alpha}$ --- that satisfies our hypothesis. 

So, to sum it up, we are in operating regime of Proposition~\ref{prop:ito_discrete_all_rho}: anticorrelated noise, $\rho_k<1$ with non-vanishing probability and $\rho_k$ always between $0$ and $1$. Hence, we get that 
\begin{equation}
    \lim_{k\to\infty}\mathbb{E}[\|u_{k}\|^2]\le 2 \sigma^2 d.
\end{equation}
Hence, for $\sigma^2$ small enough, the value $\alpha D$ is reached.
This directly implies the missing Eq.~\eqref{eq:thm_wv_antiPGD}. The proof is thereby complete.
\end{proof}


\subsection{Proof of Corollary~\ref{corollary:main_tube}} \label{app:corollary_proof}

\cortube*
\begin{proof}
Recall from Eq.~\eqref{eq:trace_of_hessian_wv} that $\tr(\nabla^2 L(u,v)) = d v^2 + \|u\|^2$. 

For the first two inequalities in the corollary, recall from Eq.~\eqref{eq:for_ind_vk} of the proof of Theorem~\ref{thm:main_tube} that $|v_n|\le 4\sigma$ for all $n$, almost surely.
The first two inequalities in the Corollary follow now by Eq.~\eqref{eq:thm_wv_antiPGD}. 

For the last two inequalities in the Corollary, we lower bound the trace by $\|u\|^2$ and make use of Eq.~\eqref{eq:thm_wv_PGD}. 
\end{proof}


\section{Additional Experimental Evidence}
\label{sec:exp_app}
\subsection{Details for Figure~\ref{fig:front_experiments}}
\begin{itemize}
    \item \textbf{Squared regression}: Problem definition, loss function, gradient and Hessian provided in~\S \ref{sec:squared_regr}. We run GD, PGD and anti-PGD with full batch~(40 datapoints in 100 dimensions), while for SGD we select a batch size of $1$. All algorithms except SGD run with a constant learning rate of $\eta=0.1$. For SGD, to improve generalization at such a small batch size, we instead select a slightly smaller learning rate $\eta=0.01$.  Perturbations in PGD and anti-PGD have parameter $\sigma=0.05$. Findings are robust to changing these hyperparameters, as shown in \S~\ref{sec:squared_regr}.  All plots also show one standard deviation for all measures.
    
    \item \textbf{Matrix sensing}: Problem definition, loss function, gradient and Hessian provided in~\S\ref{sec:matrix_sensing}. We run GD, PGD and anti-PGD with full batch~(100 datapoints in 400 dimensions), while for SGD we select a batch size of $10$. All algorithms run with a constant learning rate of $\eta=0.001$. Perturbations in PGD and anti-PGD have parameter $\sigma=0.1$. Findings are robust to changing these hyperparameters, as shown in \S~\ref{sec:matrix_sensing}. For better visualization, here all plots also show two standard deviation for all measures.
    
    \item \textbf{ResNet on CIFAR10.} Details in~\S\ref{sec:exp}. Further supporting experiments in \S\ref{sec:cifar_app}.
\end{itemize}

\subsection{Quadratically Parametrized Model}
\label{sec:squared_regr}
\paragraph{Problem Definition.} 
Consider the standard linear regression setting in $d$ dimensions with $M$ datapoints. The design matrix is $X\in\R^{M\times d}$. We assume there exist sparse~(we also study the effect of sparseness) vector $w\in\mathbb{R}^d$ such that targets $y$ are perfectly predicted as $y = Xw^{\odot 2}$, where $w^{\odot 2}\in\mathbb{R}^d$ is the element-wise product of $w$. This parametrization, also studied in~\cite{haochen2021shape,blanc2020implicit} in the context of label noise, makes the landscape highly non-linear, with minimizers that achieve different generalization properties.

For this loss we have
\begin{align}
    &L(w) = \frac{1}{4n}\|X w^{\odot 2} - y\|^2,\quad  \nabla L(w) = \frac{1}{n} \cdot [X^T  (X w^{\odot 2} - y)] \odot w. \\
    &\nabla^2 L(w) = \frac{2}{n} \cdot \textrm{diag}(w) \cdot X^TX \cdot \textrm{diag}(w) + \frac{1}{n} \cdot \textrm{diag}(X^T(X w^{\odot 2} - y ))
\end{align}

\paragraph{Precise setting.} As in~\cite{haochen2021shape}, we consider the case $d=100$ and $M=40$, and generate $X_{\text{train}}$ at random. We consider $w^* = (1,1,\dots,1,0,0,\dots, 0)$ --- where only $10$ elements are non-zero, and generate $y = X_{\text{train}}(w^*\odot w^*)$. For testing, we use instead $100$ datapoints. We test three different learning rates in Figure~\ref{fig:square_tuning_low},~\ref{fig:square_tuning_mod},~\ref{fig:square_tuning_high}, and for each 3 values of noise injection variance. For SGD, since we choose a very small~(i.e. unit) batch size, the learning rate is scaled down by a factor of 10, to provide stability. 

\paragraph{Findings.} We found that anti-PGD always provides the best test accuracy, and minimizes the trace of the Hessian as well. This finding is quite robust in terms of hyperparameter tuning. Further, we found that the performance is always drastically different from the one of PGD. An explanation of this phenomenon is provided in Theorem~\ref{thm:main_tube}, in the main paper. Mini-batch SGD improves the final test loss if the stepsize is small enough. Larger stepsizes are unstable for SGD.

\begin{figure}[ht!]
    \centering
    
    \includegraphics[height=0.275\textwidth]{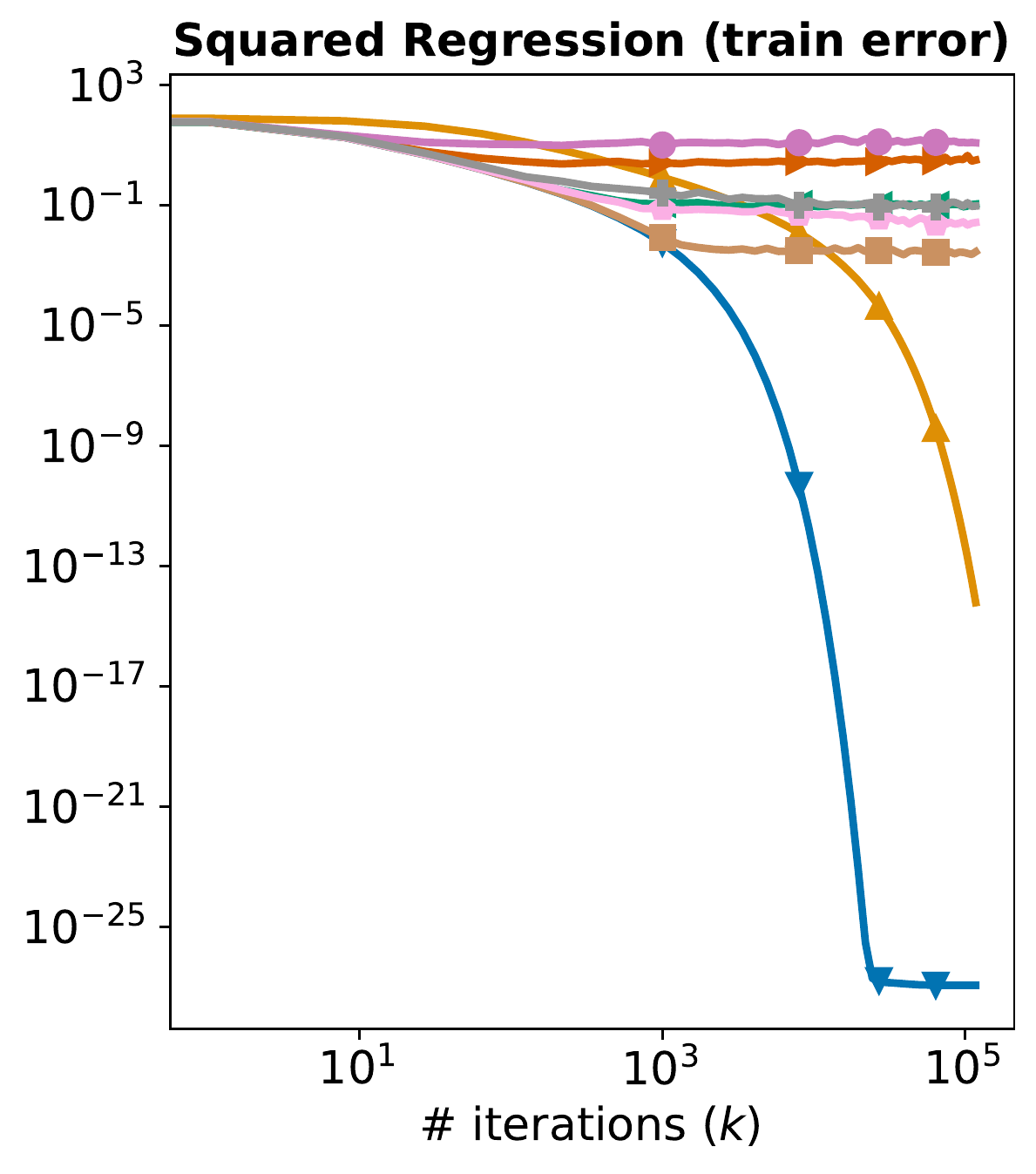}
    \includegraphics[height=0.275\textwidth]{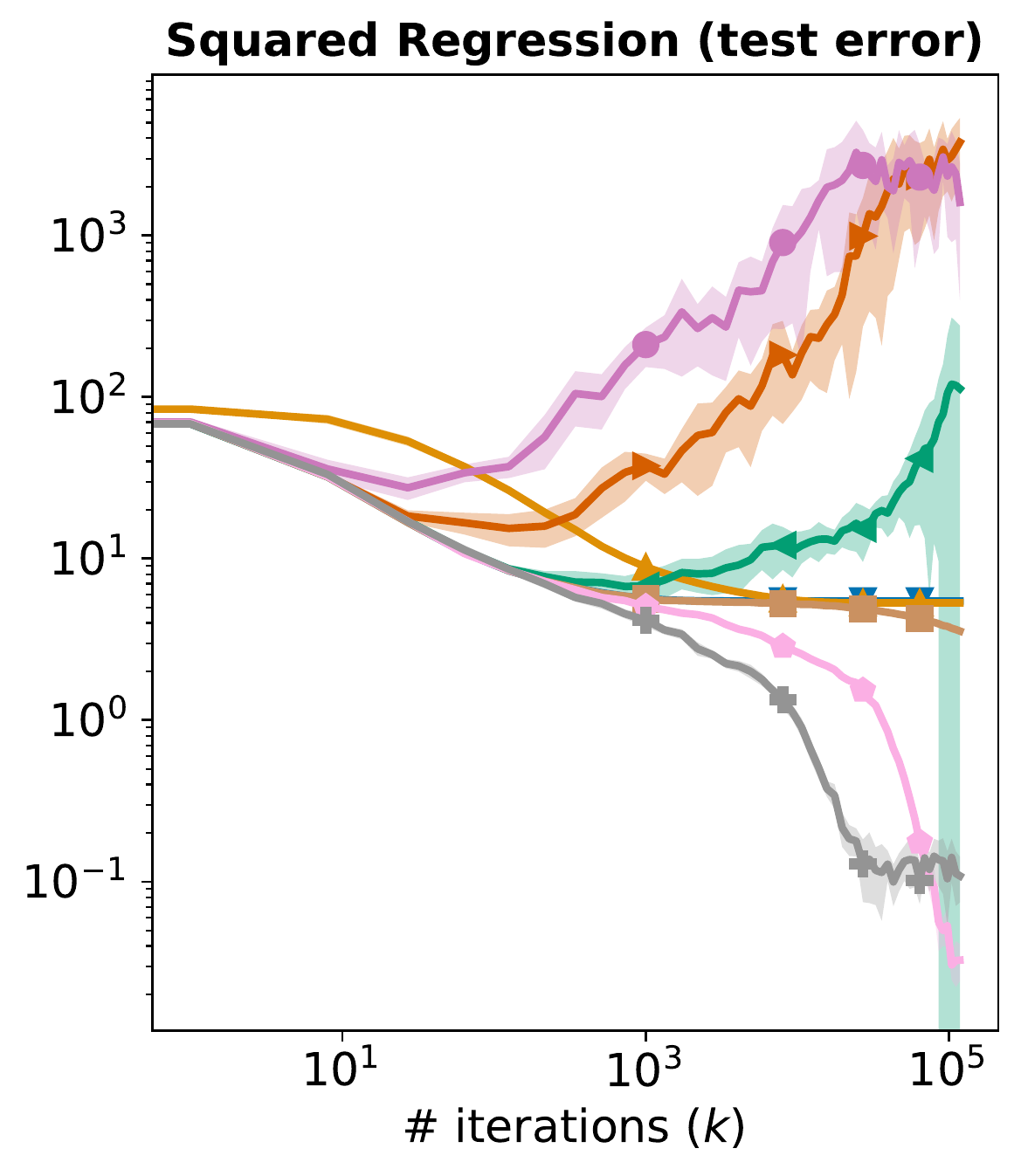}
    \includegraphics[height=0.275\textwidth]{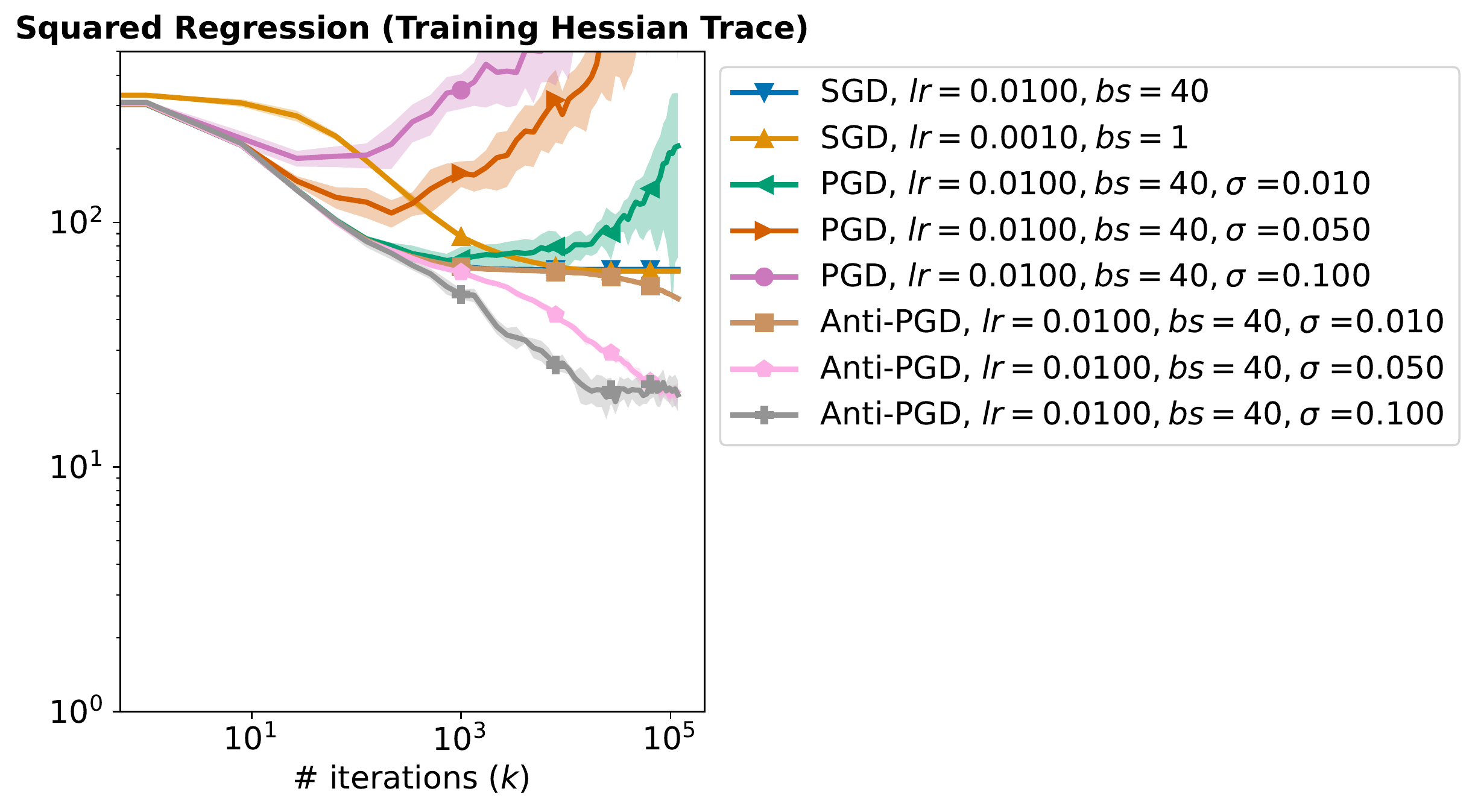}
    
    \vspace{-3mm}
    \caption{Performance of anti-PGD on \textbf{quadratically parametrized linear regression}, for \textbf{low learning rate} and different values of noise injection standard deviation. Plotted is also the error bar relative to 1 standard deviation~(10 runs).}
    
    \label{fig:square_tuning_low}
\end{figure}

\begin{figure}[ht!]
    \centering
    
    \includegraphics[height=0.275\textwidth]{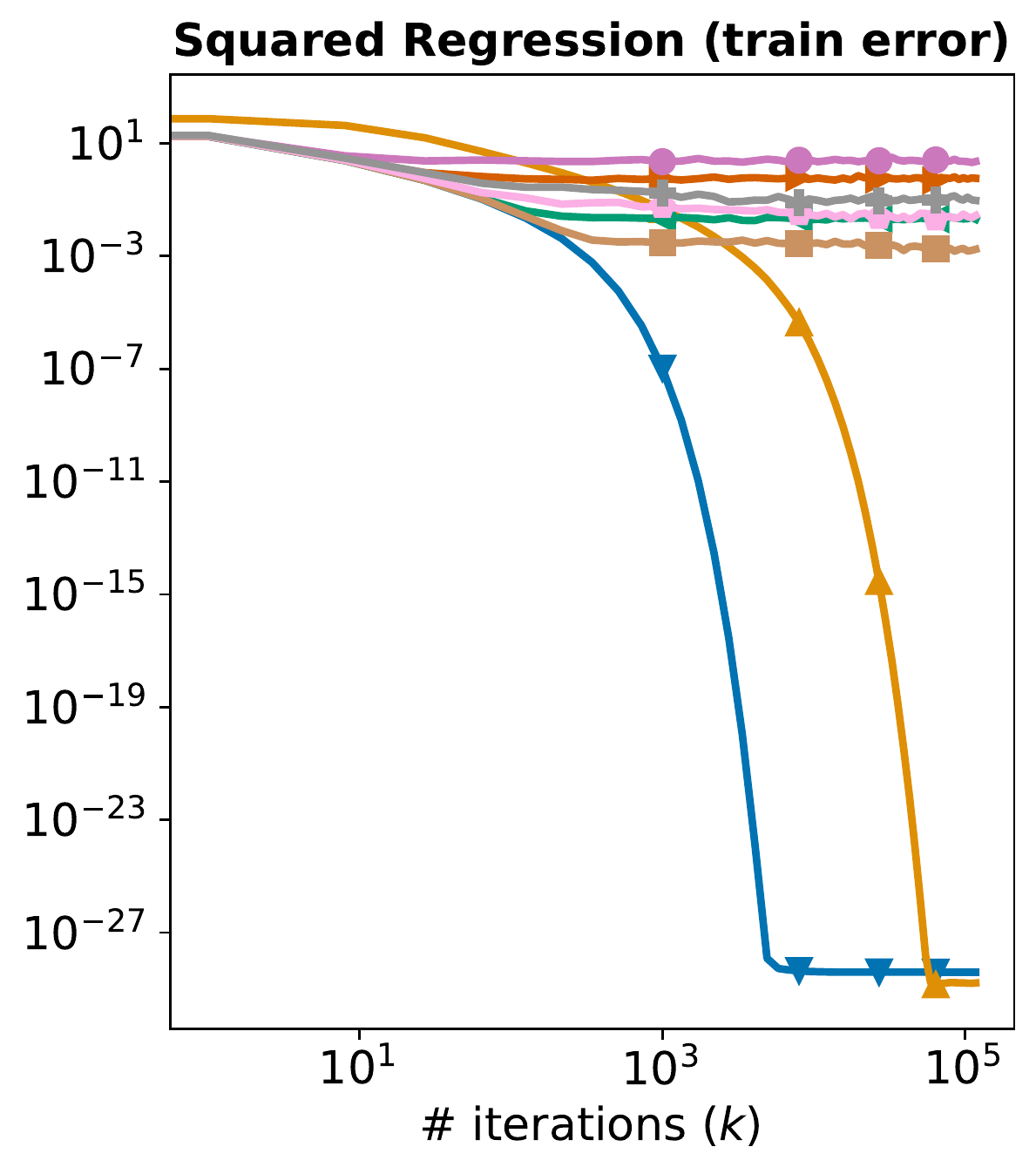}
    \includegraphics[height=0.275\textwidth]{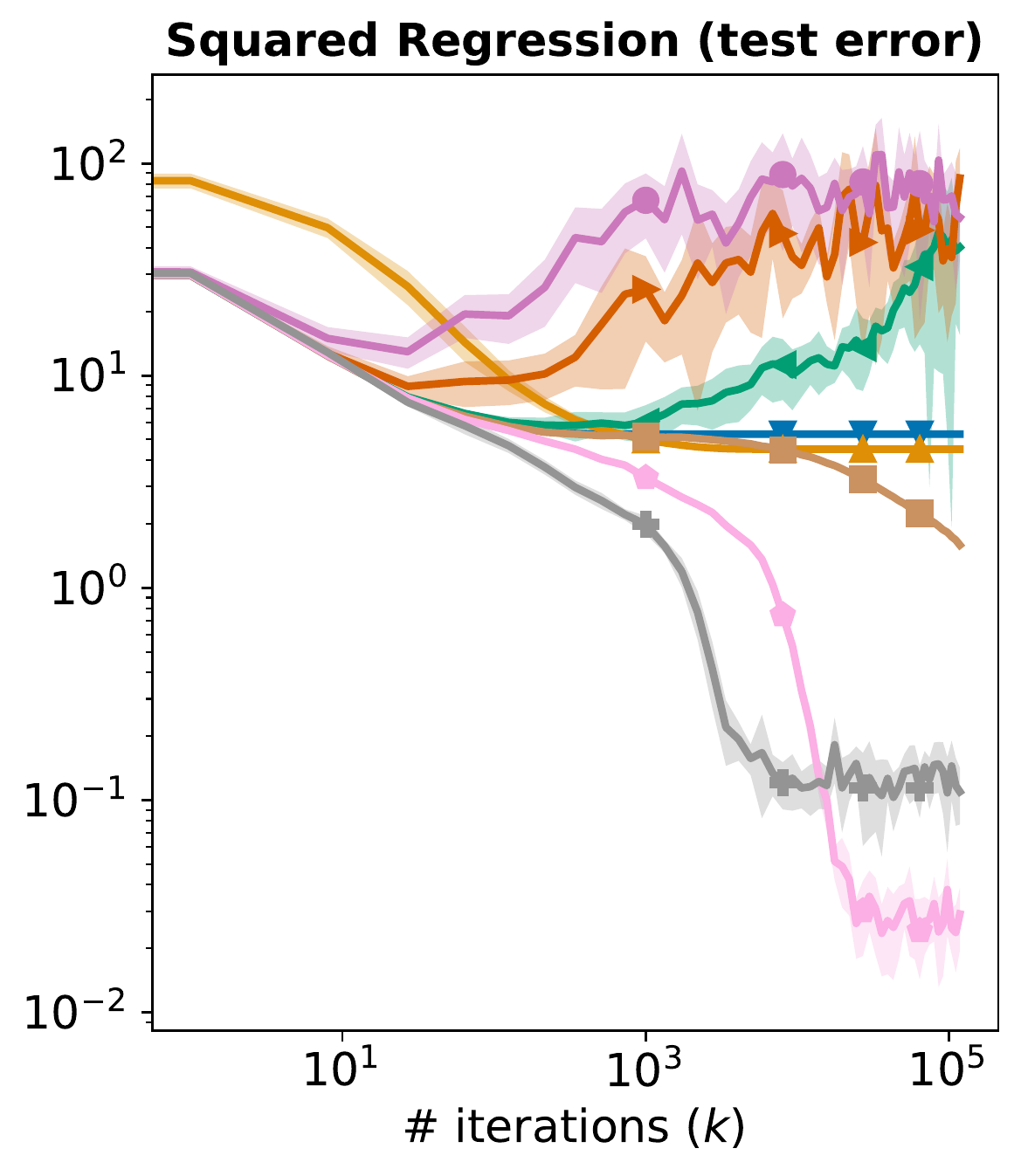}
    \includegraphics[height=0.275\textwidth]{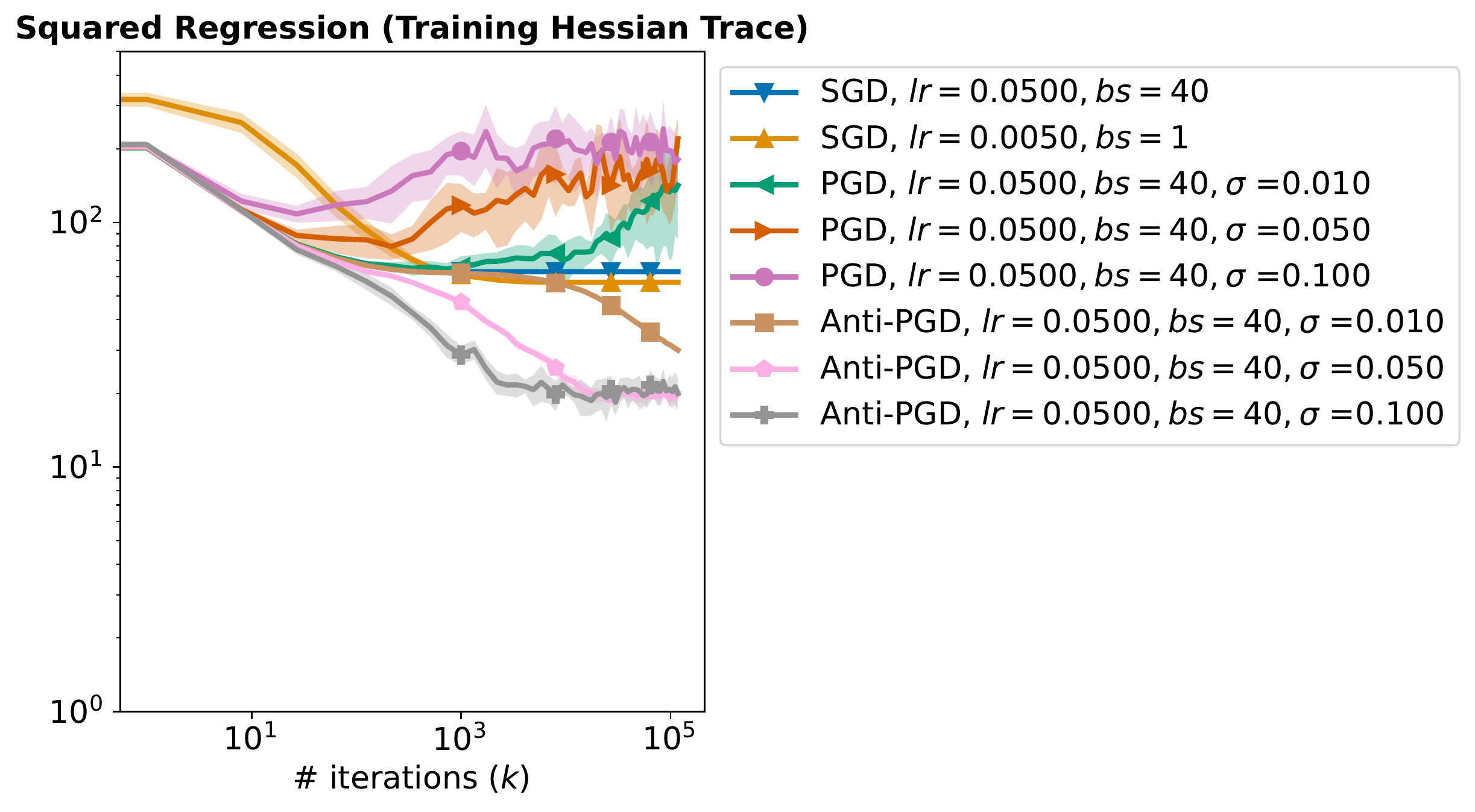}
    
    \vspace{-3mm}
    \caption{Performance of anti-PGD on \textbf{quadratically parametrized linear regression}, for \textbf{moderate learning rate} and different values of noise injection standard deviation. Plotted is also the error bar relative to 1 standard deviation~(10 runs).}
    \label{fig:square_tuning_mod}
\end{figure}

\begin{figure}[ht!]
    \centering

    \includegraphics[height=0.275\textwidth]{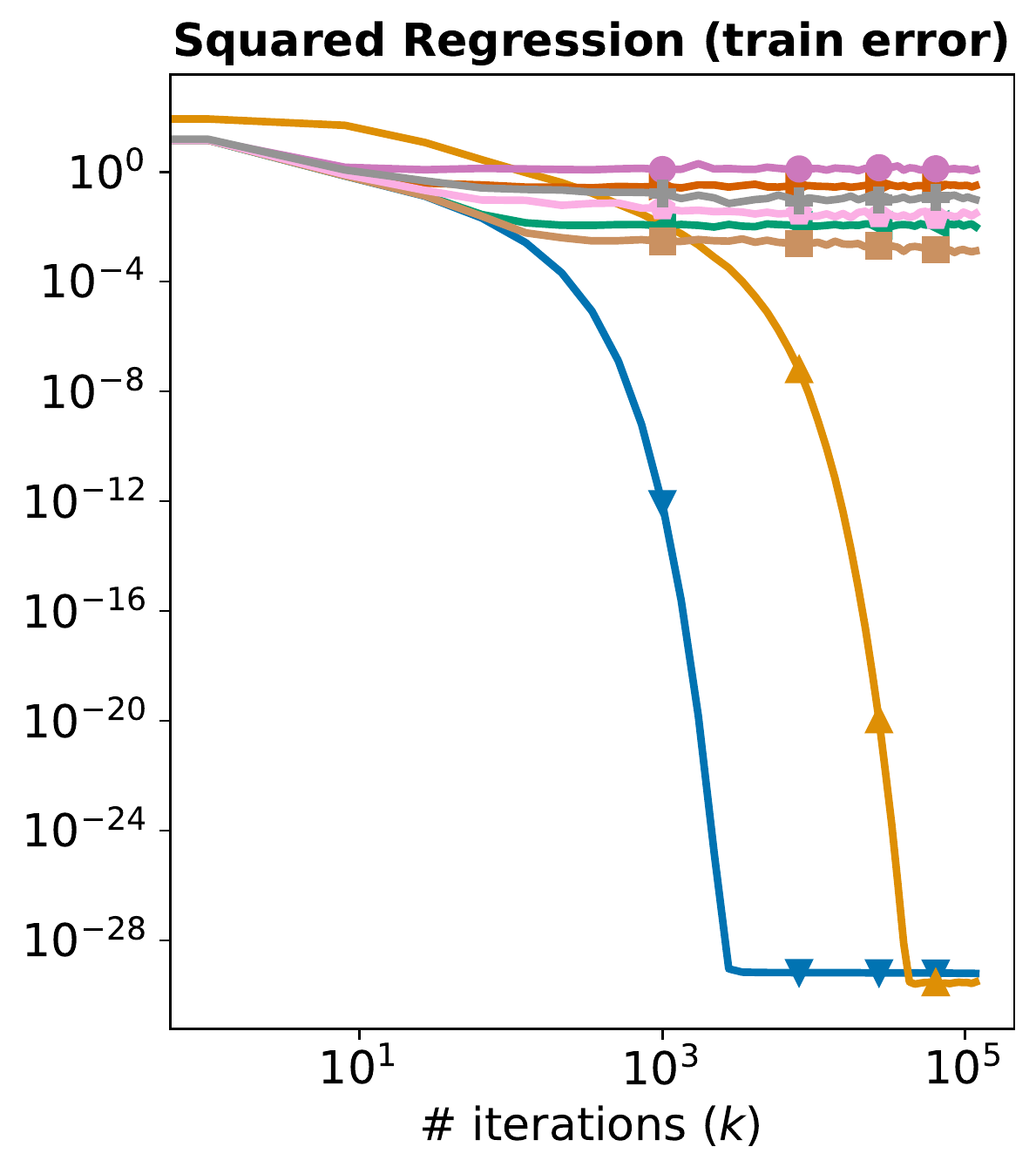}
    \includegraphics[height=0.275\textwidth]{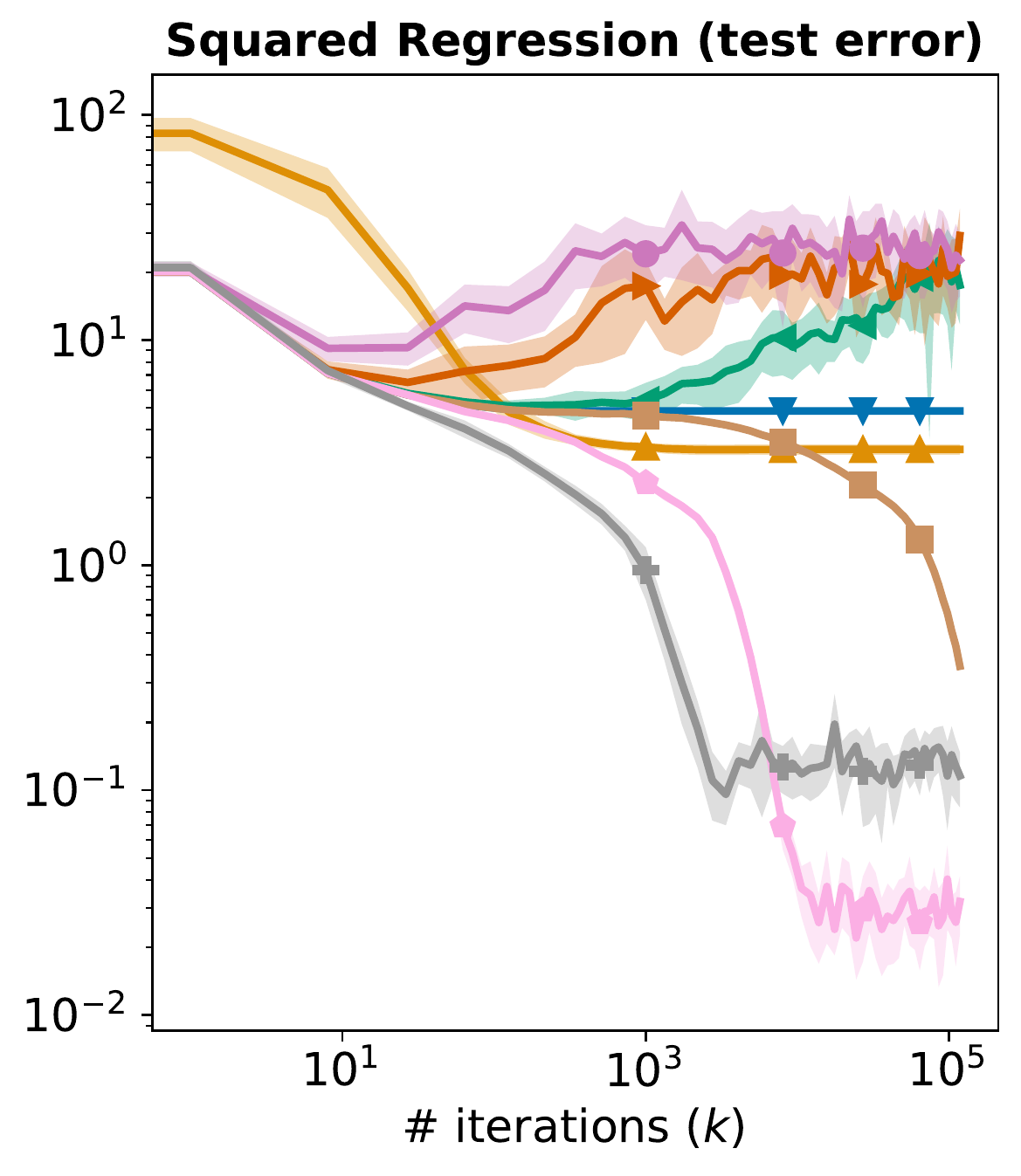}
    \includegraphics[height=0.275\textwidth]{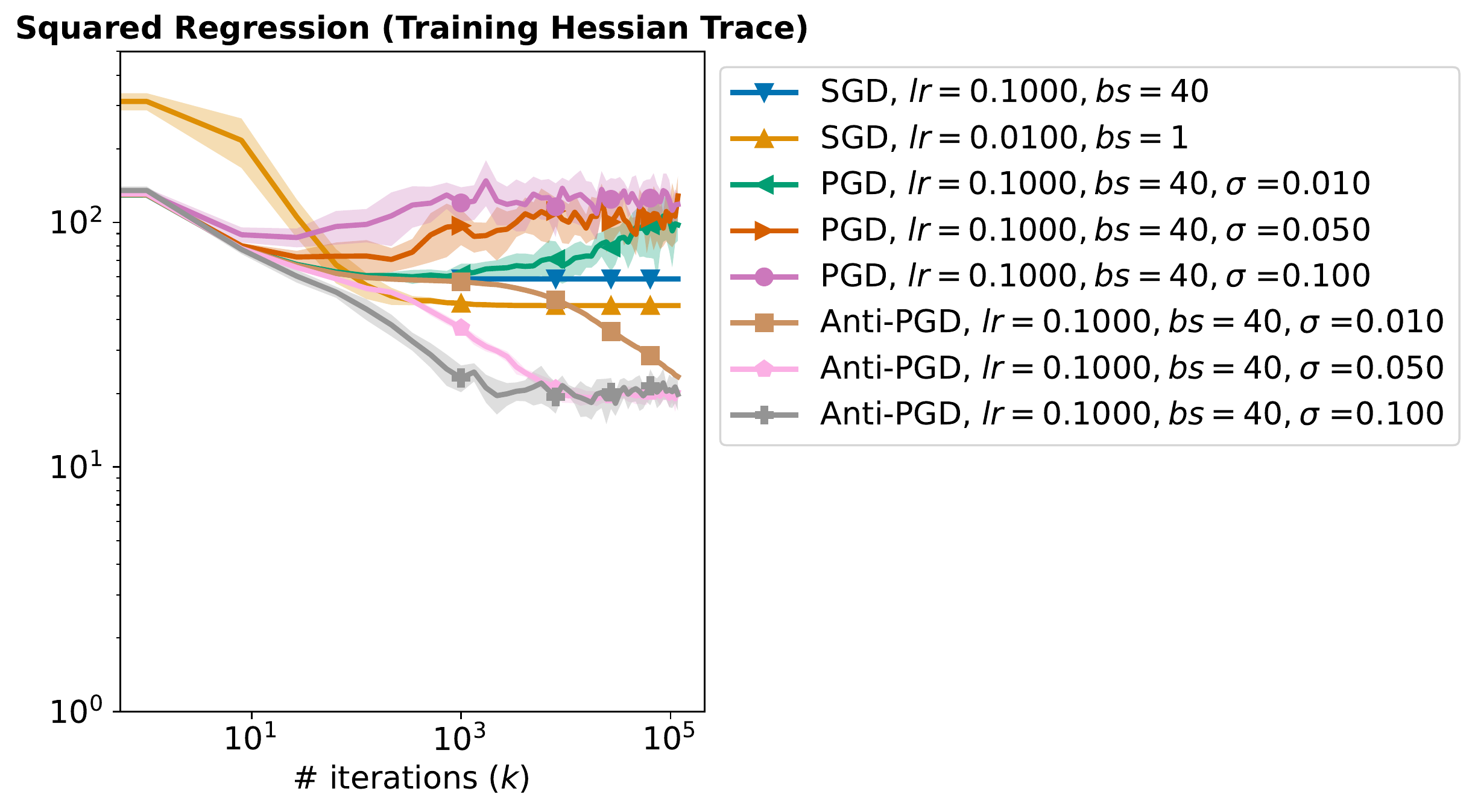}

    \vspace{-3mm}
    \caption{Performance of anti-PGD on \textbf{quadratically parametrized linear regression}, for \textbf{high learning rate} and different values of noise injection standard deviation. Plotted is also the error bar relative to 1 standard deviation~(10 runs).}
    \label{fig:square_tuning_high}
\end{figure}

\subsection{Matrix Sensing}
\label{sec:matrix_sensing}
\paragraph{Problem Definition.} This setting is inspired by the experiment of~\cite{blanc2020implicit} on label noise. Let $X^*$ be an unknown rank-$r$ symmetric positive semidefinite (PSD) matrix in $\mathbb{R}^{n\times n}$ that we aim
to recover. Assume this has unit $2$-norm. Let $A_1,\dots,A_M\in\mathbb{R}^{n\times n}$ be $M$ given (wlog) symmetric measurement matrices. We assume
that the label vector $y \in \mathbb{R}^M$ is generated by linear measurements $y_i = \langle A_i^\top, X^* \rangle = tr(A_i^T X^*).$
We want to minimize the loss
\begin{equation}
    L(U)=\frac{1}{M}\sum_{i=1}^M L_i(U),\quad  L_i(U) =\frac{1}{2}(y_i-\langle A_i, UU^\top\rangle)^2,
\end{equation}
where $U\in\mathbb{R}^{n\times n}$ in general achieves a good test accuracy if has small rank. 

\paragraph{Precise setting.} Our setting is similar to~\cite{blanc2020implicit}. We consider the case $n = 20$, and generate a random $X^* = V^*(V^*)\top$, of rank $5$, by picking $V\in\mathbb{R}^{n\times5}$ with standard Gaussian entries. Also all the $A_i\in\mathbb{R}^{n\times n}$ are sampled at random with standard Gaussian distributed independent entries. We consider learning from $100$ training examples~(corrupted by a small Gaussian noise). At test time, we evaluate the solution against 100 newly sampled measurements. We test three different learning rates in Figure~\ref{fig:sensing_tuning_low},~\ref{fig:sensing_tuning_mod},~\ref{fig:sensing_tuning_high}, and for each 3 values of noise injection variance.

\begin{figure}[ht!]
    \centering
    
    \includegraphics[height=0.28\textwidth]{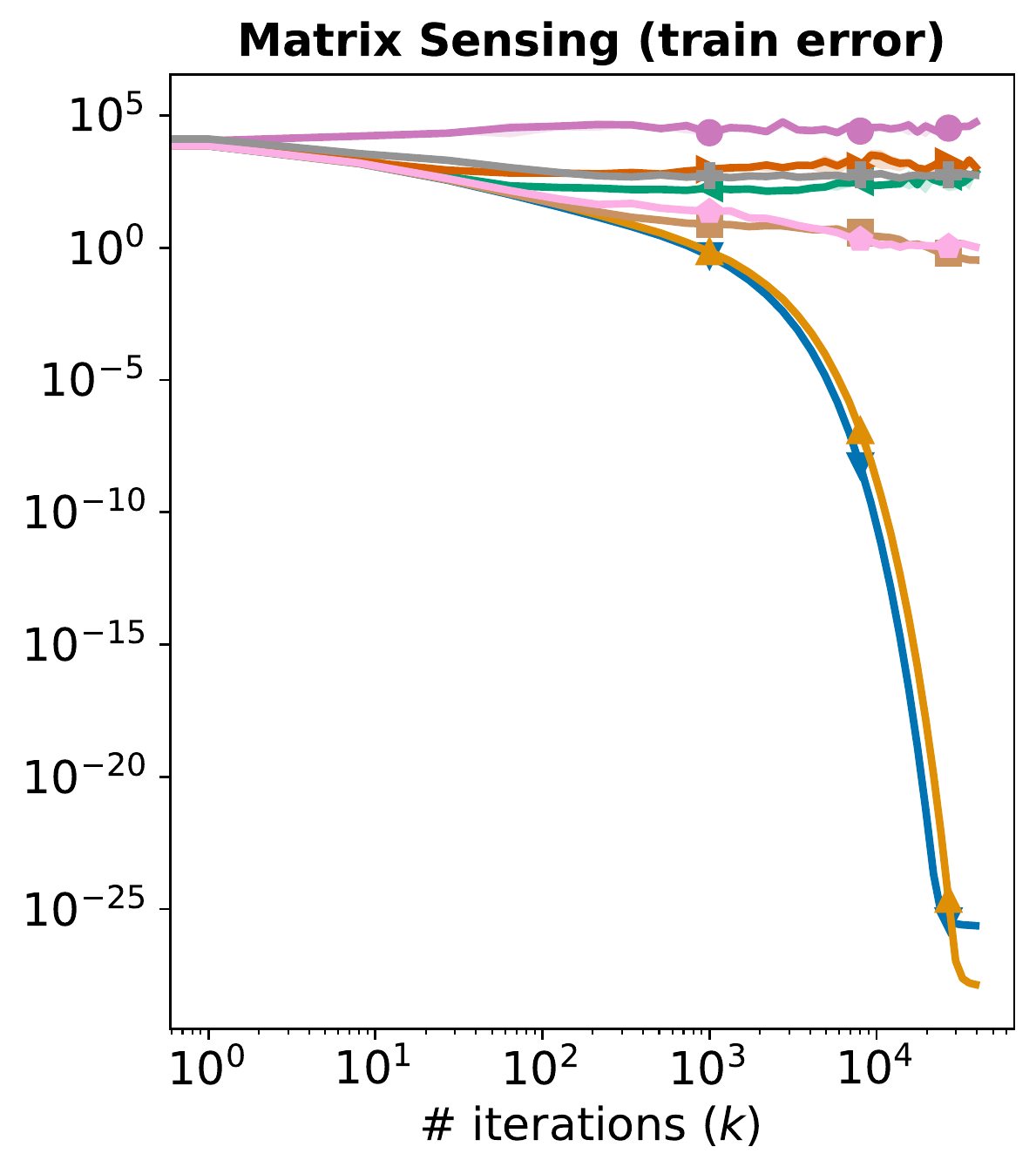}
    \includegraphics[height=0.28\textwidth]{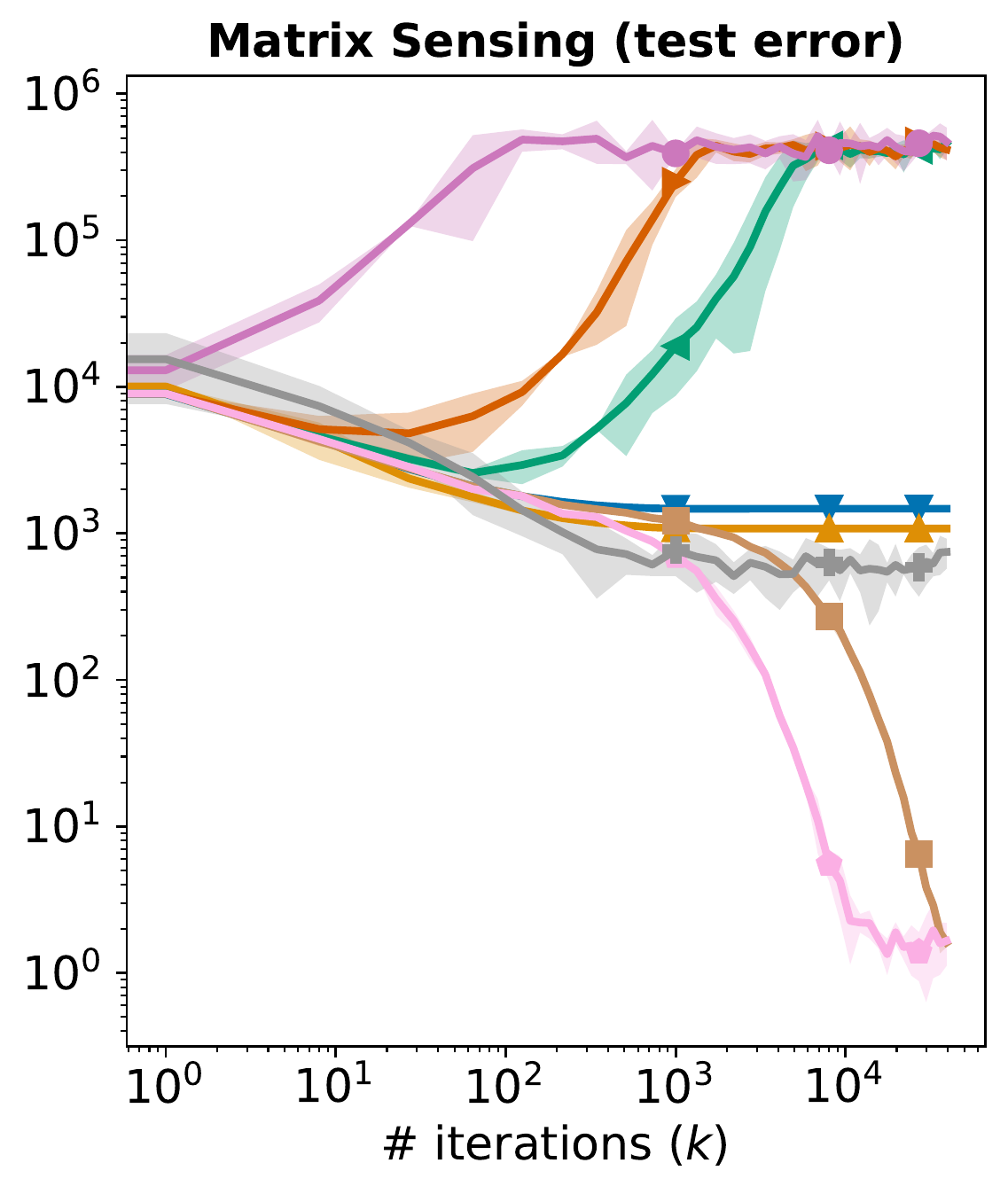}
    \includegraphics[height=0.28\textwidth]{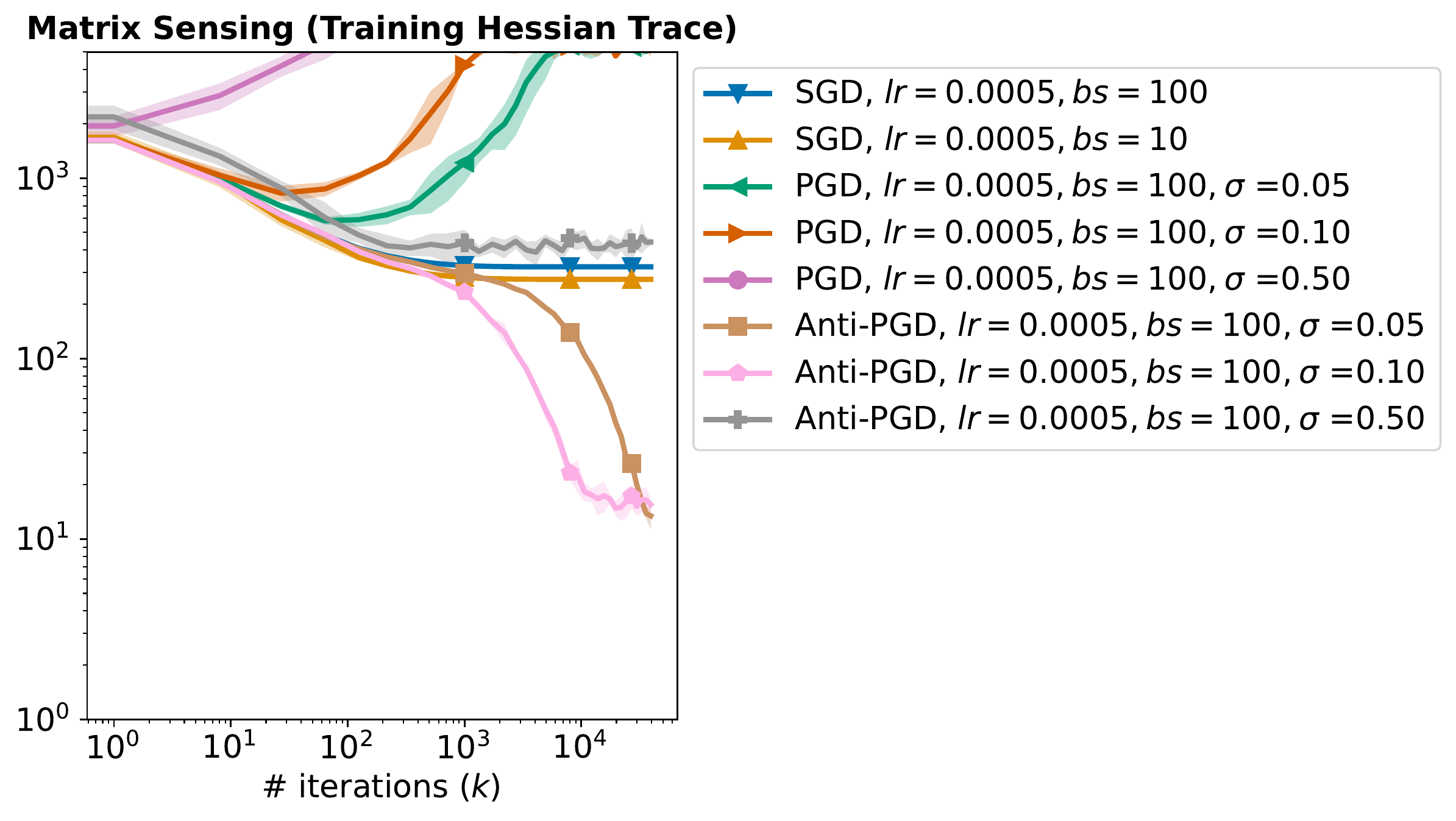}
    
    \vspace{-3mm}
    \caption{Performance of anti-PGD on \textbf{matrix sensing}, for \textbf{low learning rate} and different values of noise injection standard deviation. Plotted is also the error bar relative to 2 standard deviation~(5 runs). }
    
    \label{fig:sensing_tuning_low}
\end{figure}

\begin{figure}[ht!]
    \centering
    
    \includegraphics[height=0.28\textwidth]{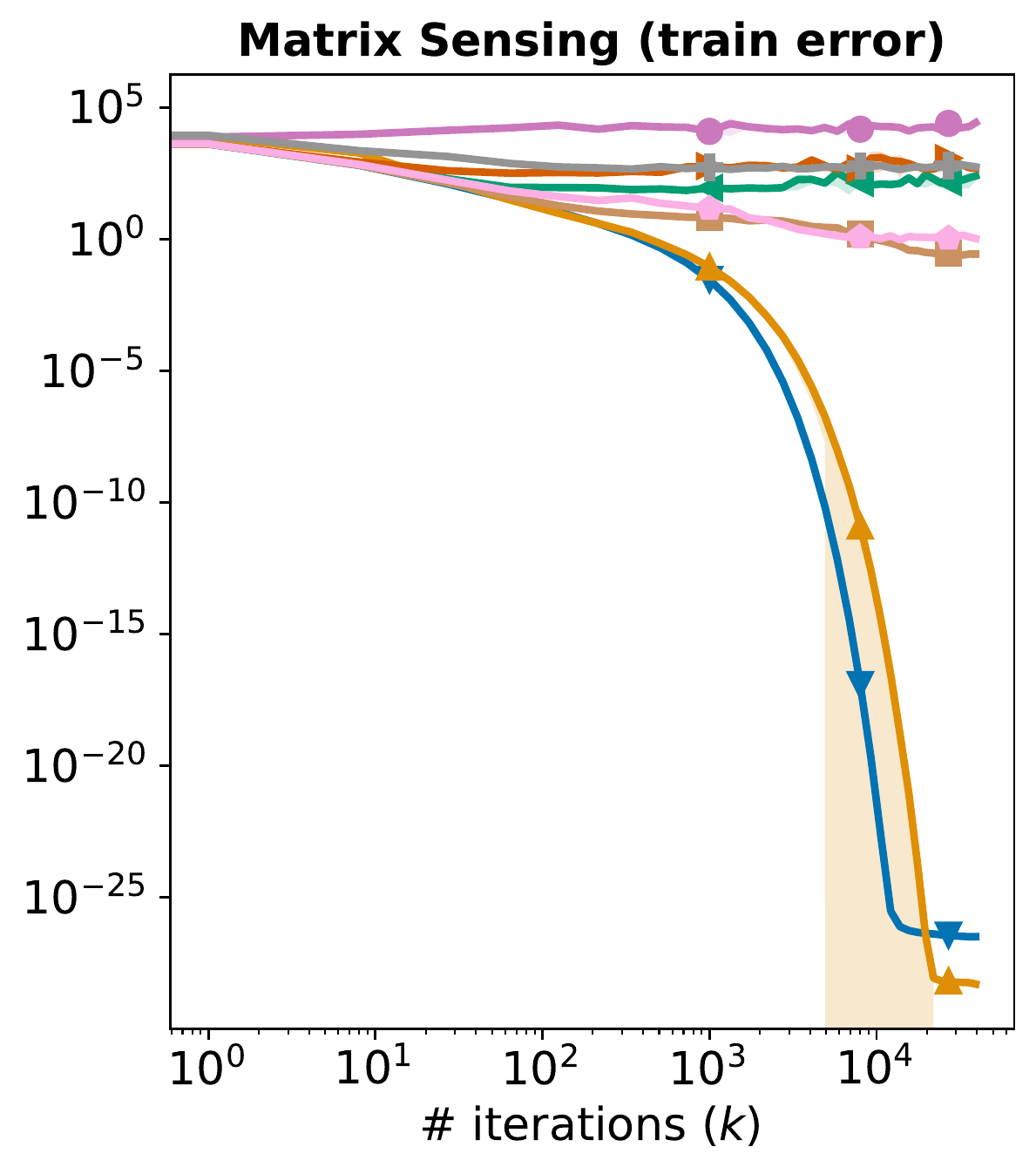}
    \includegraphics[height=0.28\textwidth]{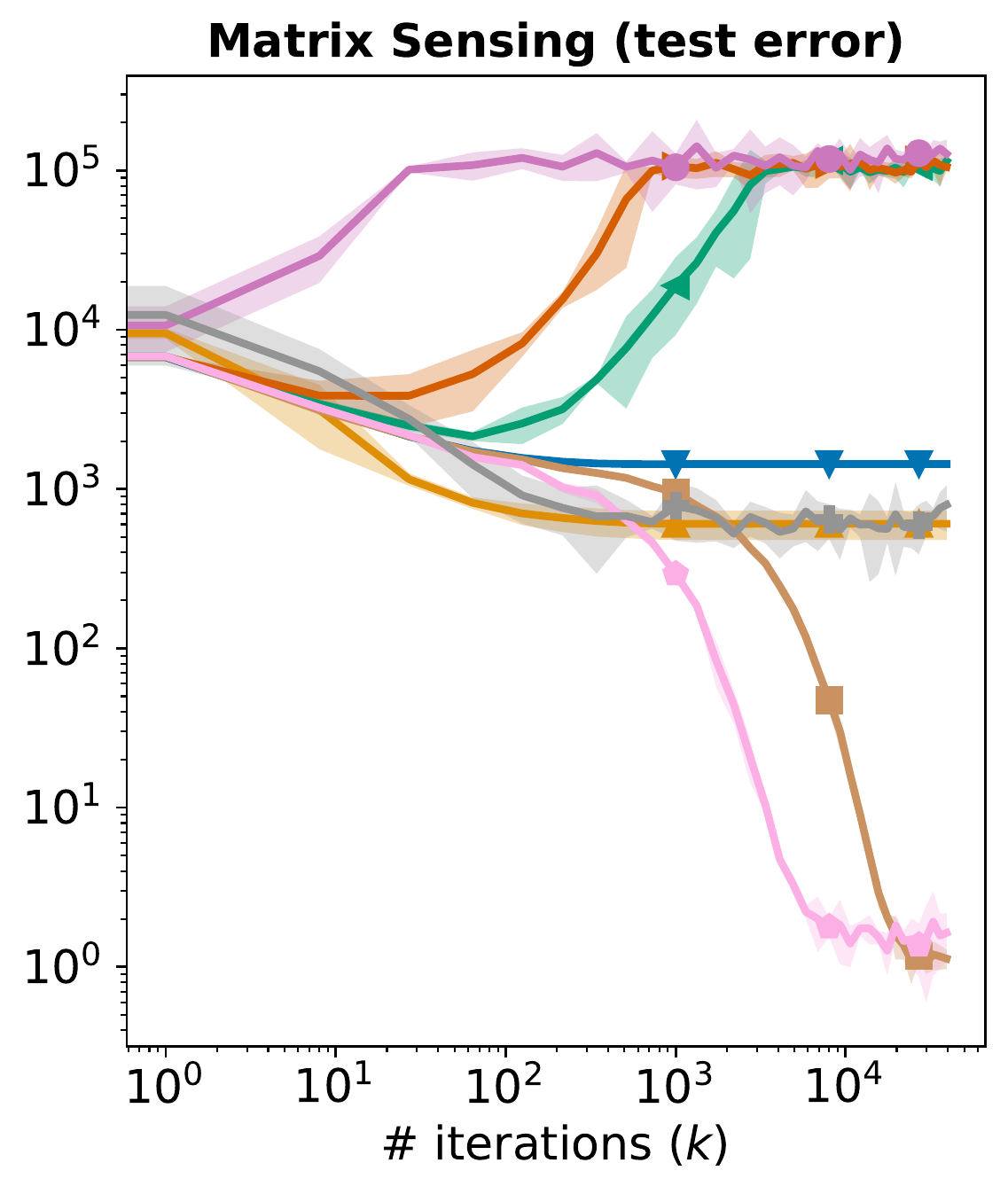}
    \includegraphics[height=0.28\textwidth]{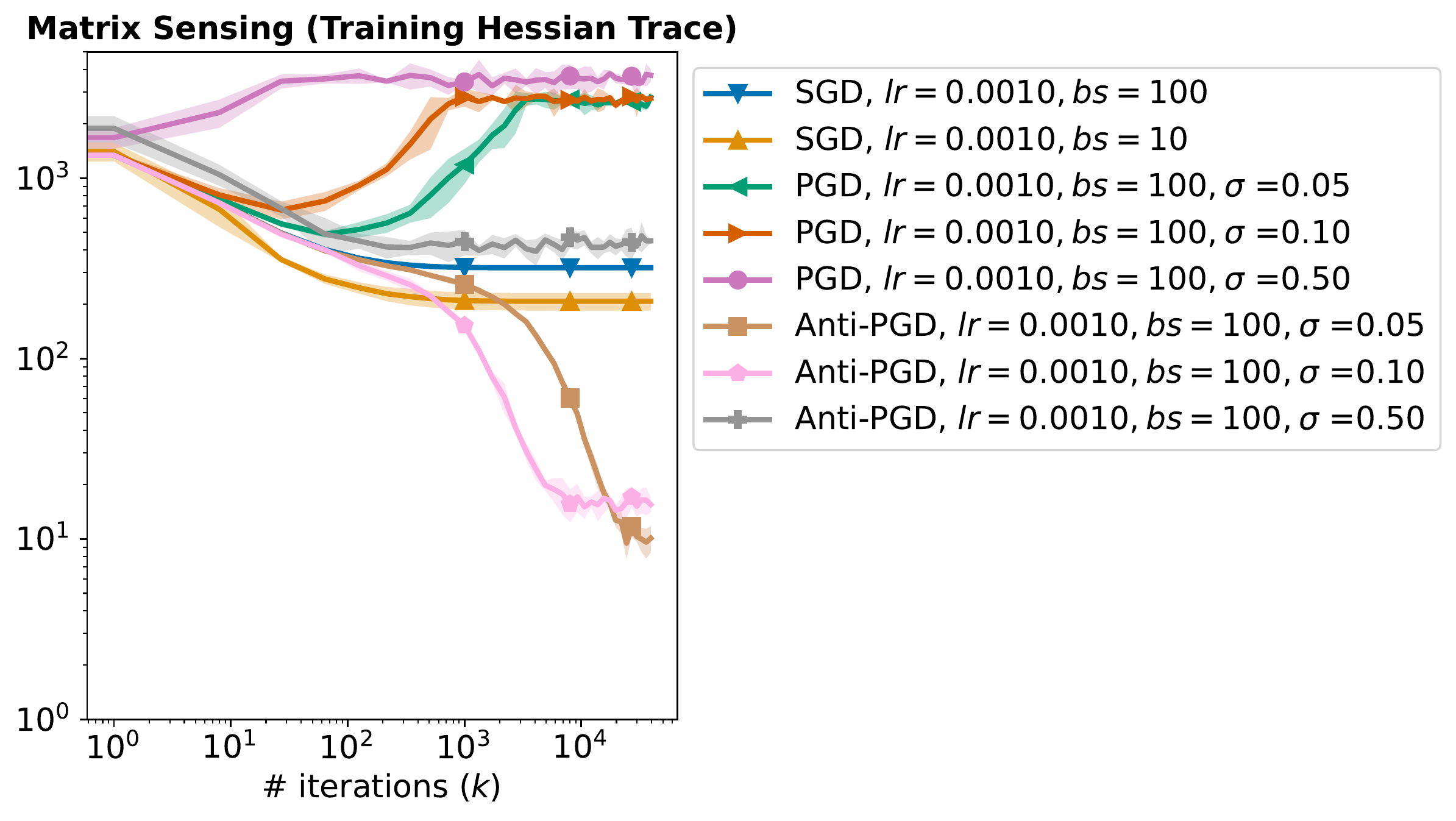}
    
    \vspace{-3mm}
    \caption{Performance of anti-PGD on \textbf{matrix sensing}, for \textbf{moderate learning rate} and different values of noise injection standard deviation. Plotted is also the error bar relative to 2 standard deviation~(5 runs).}
    \label{fig:sensing_tuning_mod}
\end{figure}

\begin{figure}[ht!]
    \centering

    \includegraphics[height=0.28\textwidth]{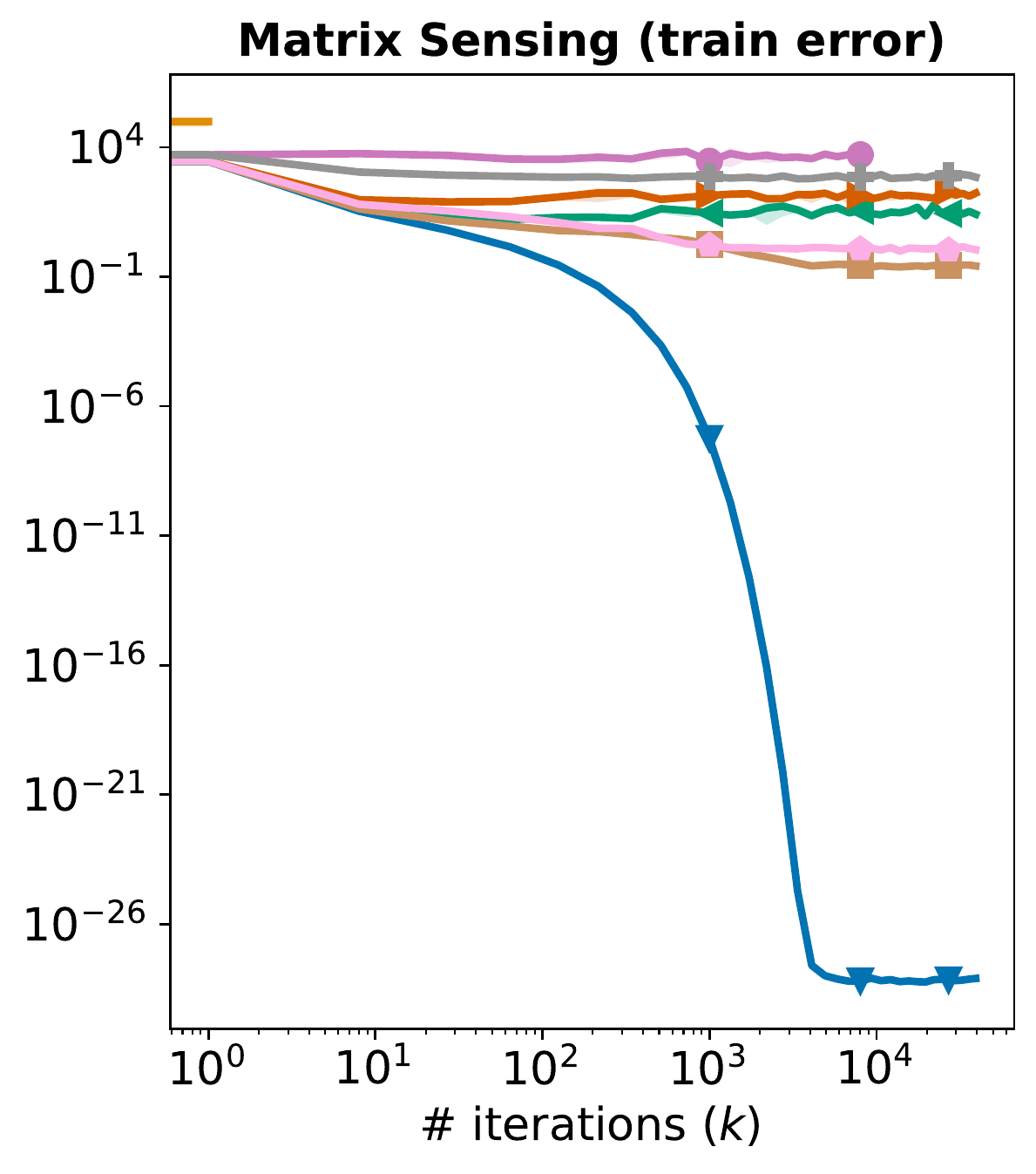}
    \includegraphics[height=0.28\textwidth]{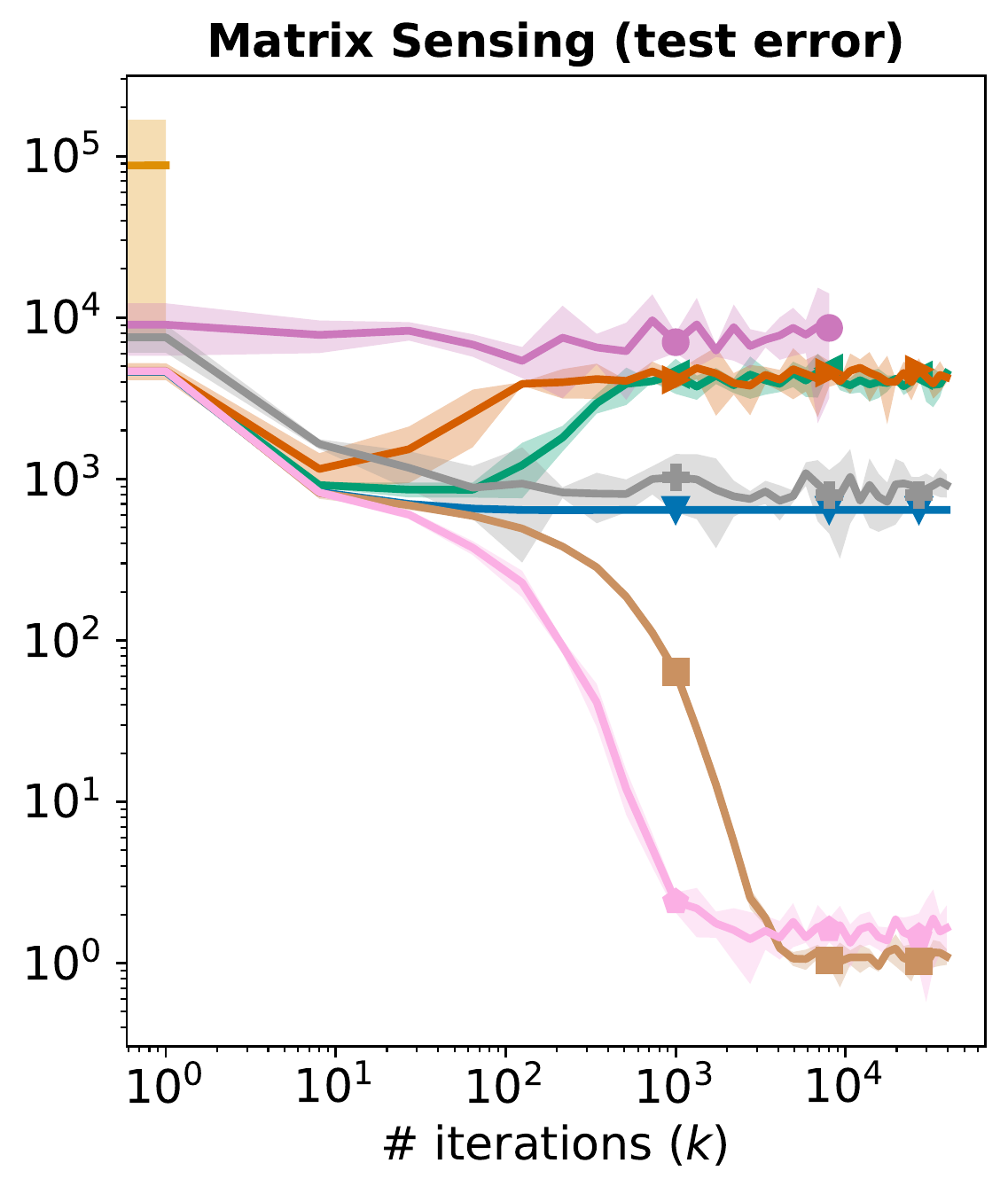}
    \includegraphics[height=0.28\textwidth]{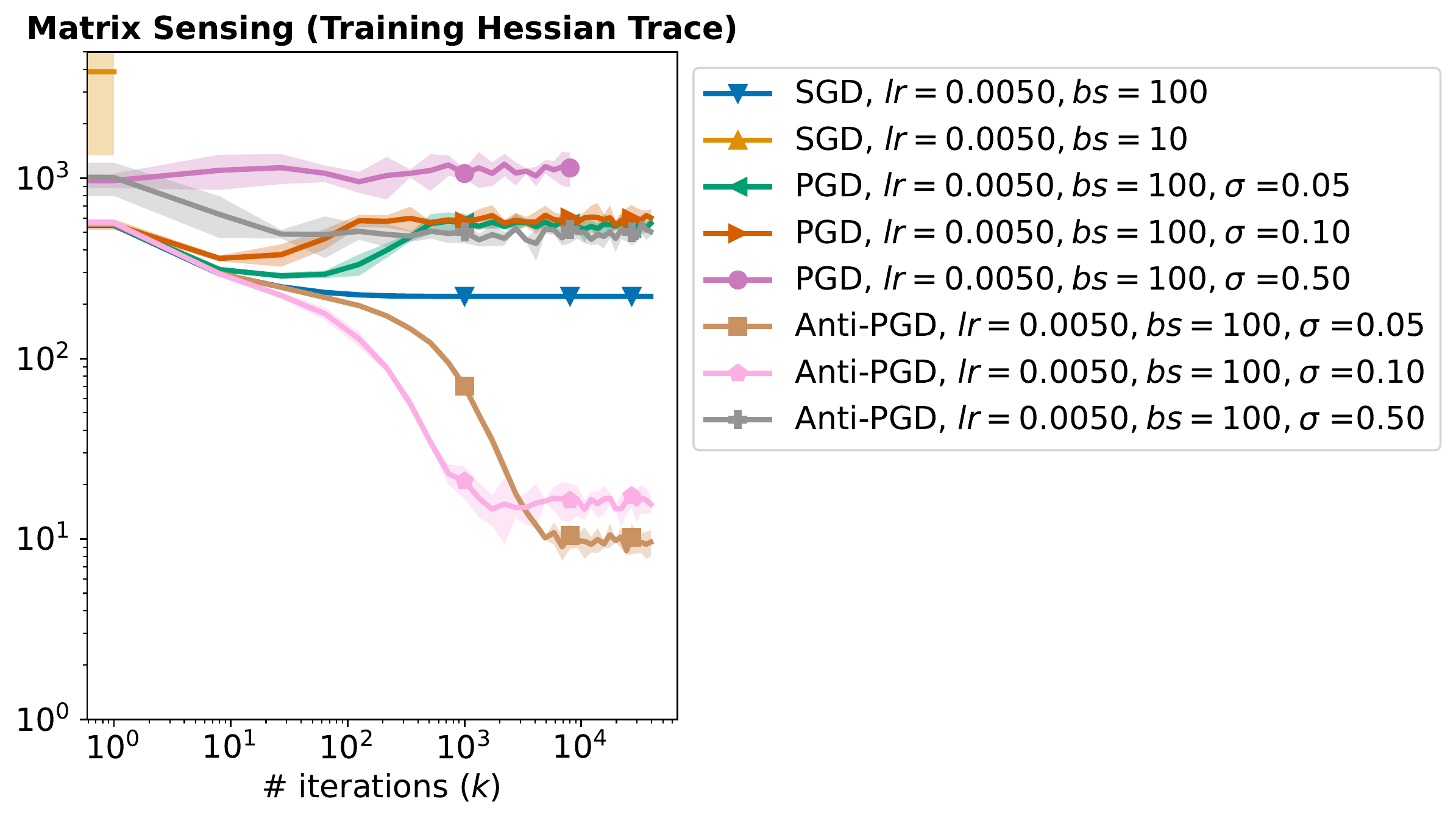}

    \vspace{-3mm}
    \caption{Performance of anti-PGD on \textbf{matrix sensing}, for \textbf{high learning rate} and different values of noise injection standard deviation. Plotted is also the error bar relative to 2 standard deviation~(5 runs). SGD with batch size 10 is unstable at this learning rate.}
    \label{fig:sensing_tuning_high}
\end{figure}

\paragraph{Findings.} We found that anti-PGD always provides the best test accuracy, and minimizes the trace of the Hessian as well. This finding is quite robust in terms of hyperparameter tuning. Further, we found that the performance is always drastically different from the one of PGD. An explanation of this phenomenon is provided in Theorem~\ref{thm:main_tube}, in the main paper. Mini-batch SGD improves the final test loss if the stepsize is small enough, but gets unstable for big stepsizes.

\subsection{CIFAR 10 ResNet 18}
\label{sec:cifar_app}
We use the implementation of ResNet18 provided by \url{https://github.com/kuangliu/pytorch-cifar}). Details about the corresponding experiments can be found in \S\ref{sec:exp}.

\paragraph{Takeaway: even after heavy tuning, AntiPGD performs better than standard noise injection.}

\begin{figure}[ht!]
    \centering

    \includegraphics[height=0.22\textwidth]{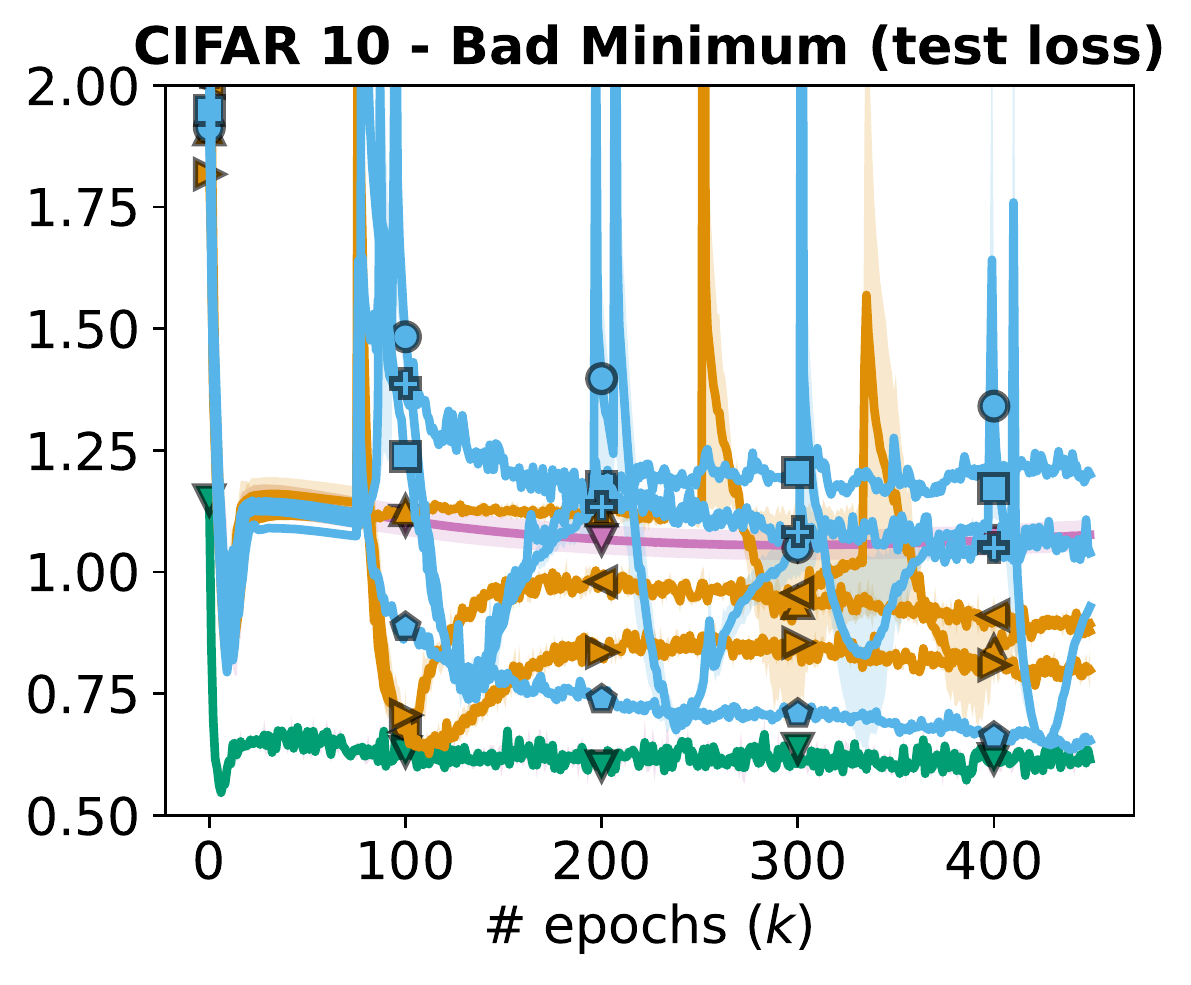}
    \includegraphics[height=0.22\textwidth]{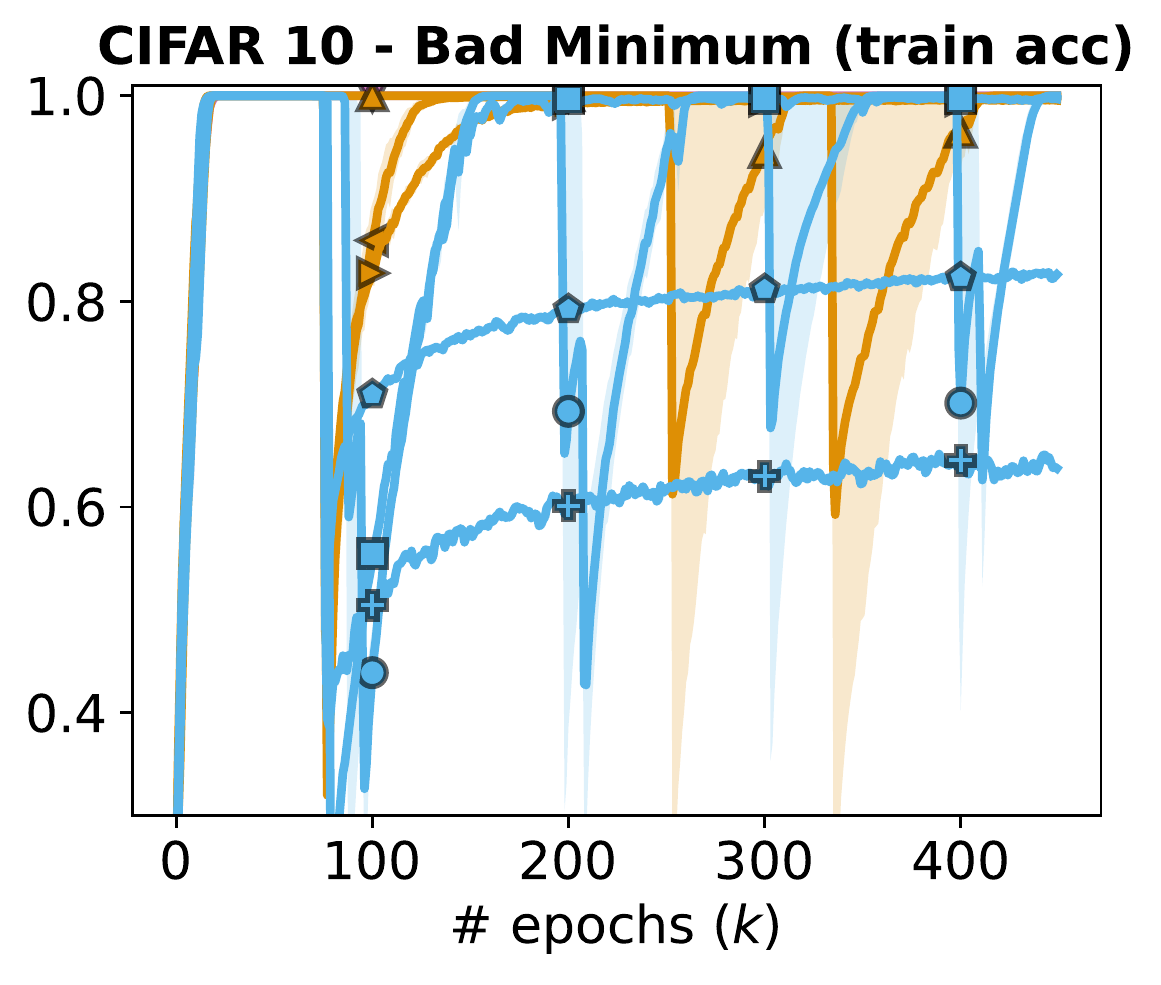}
    \includegraphics[height=0.22\textwidth]{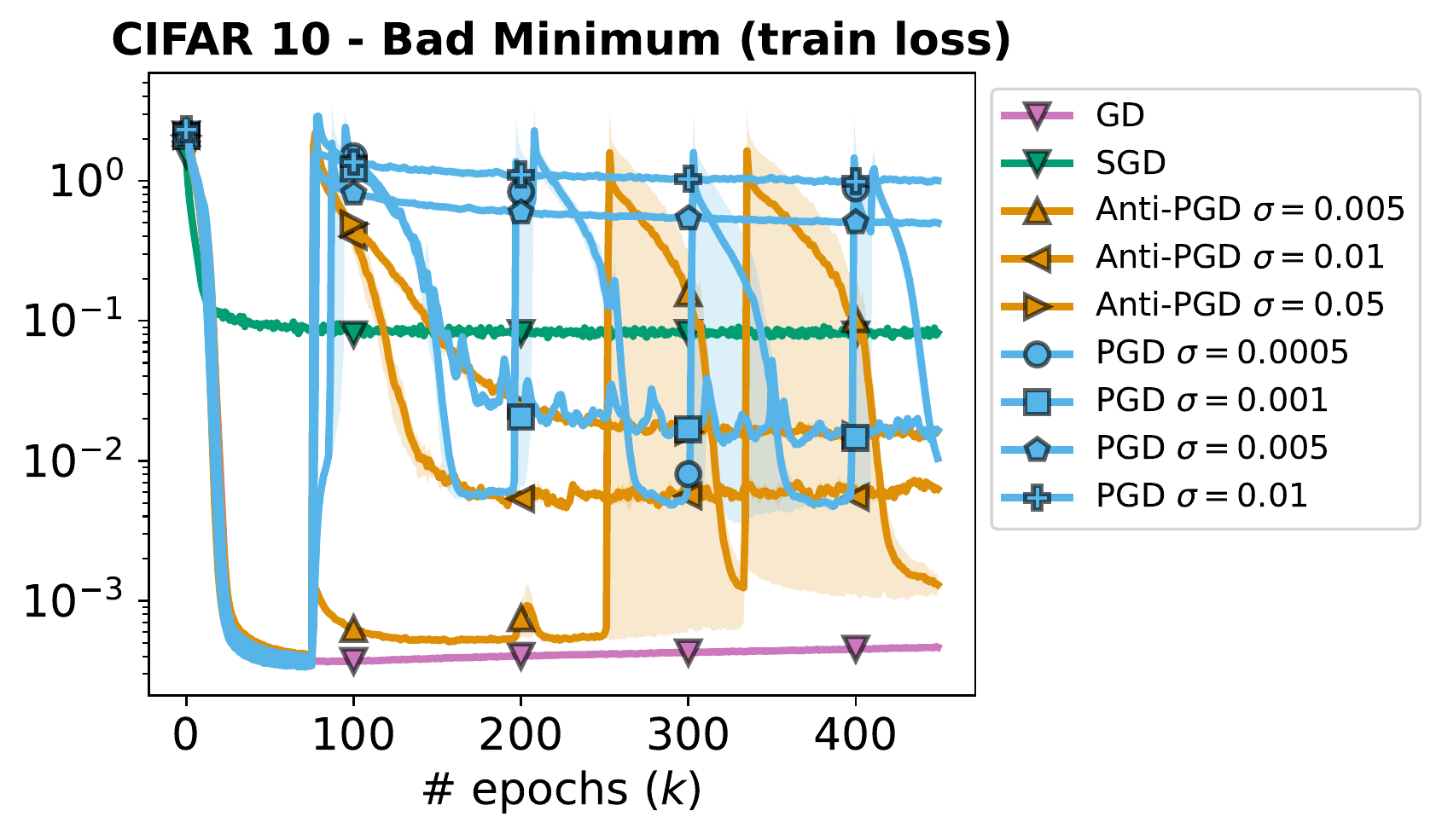}

    \vspace{-3mm}
    \caption{Additional plot for the experiment in Figure~\ref{fig:cifar-full-inj-init}.}
\end{figure}
\vspace{-3mm}
\begin{figure}[ht!]
    \centering
    \includegraphics[height=0.22\textwidth]{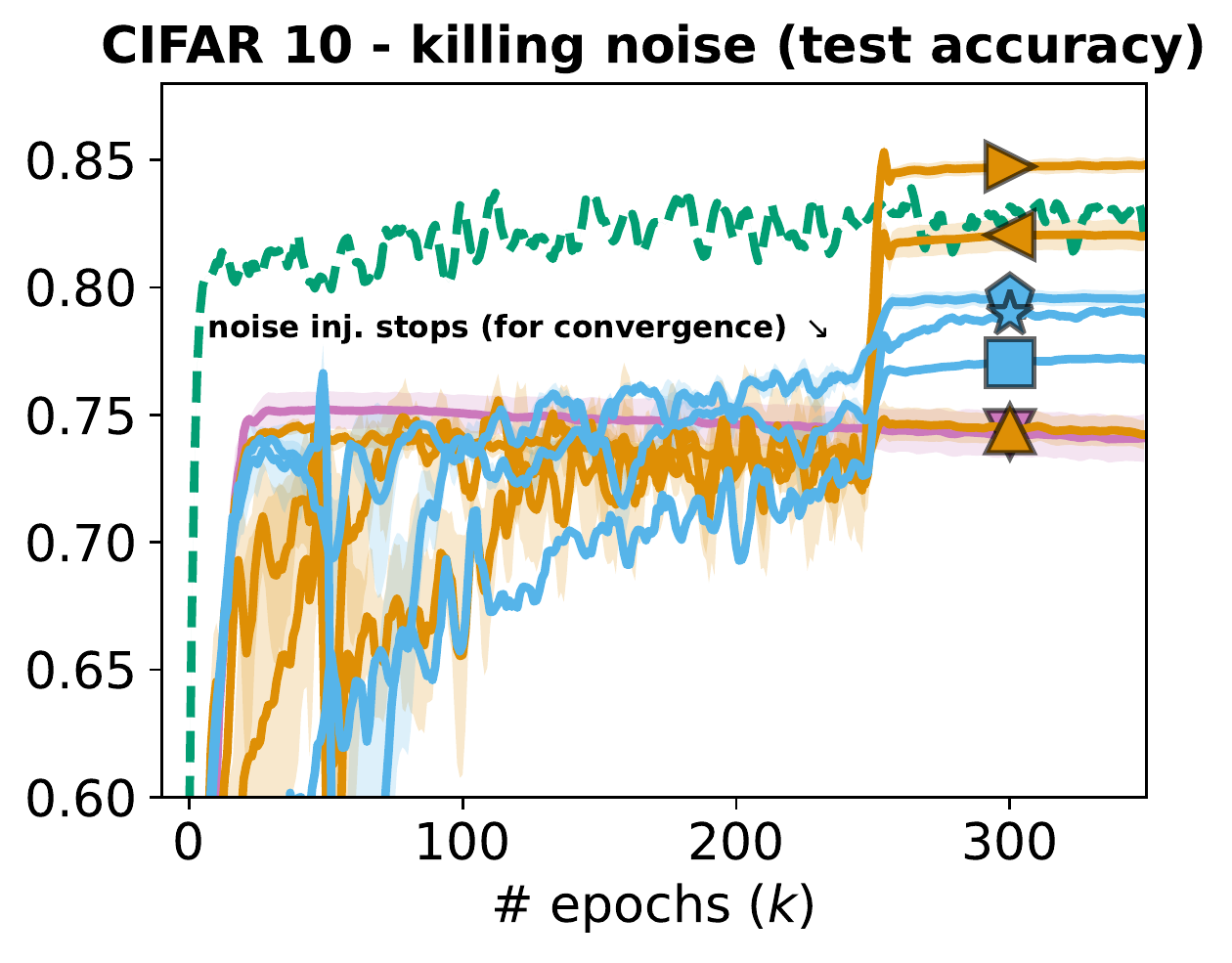}
    \includegraphics[height=0.22\textwidth]{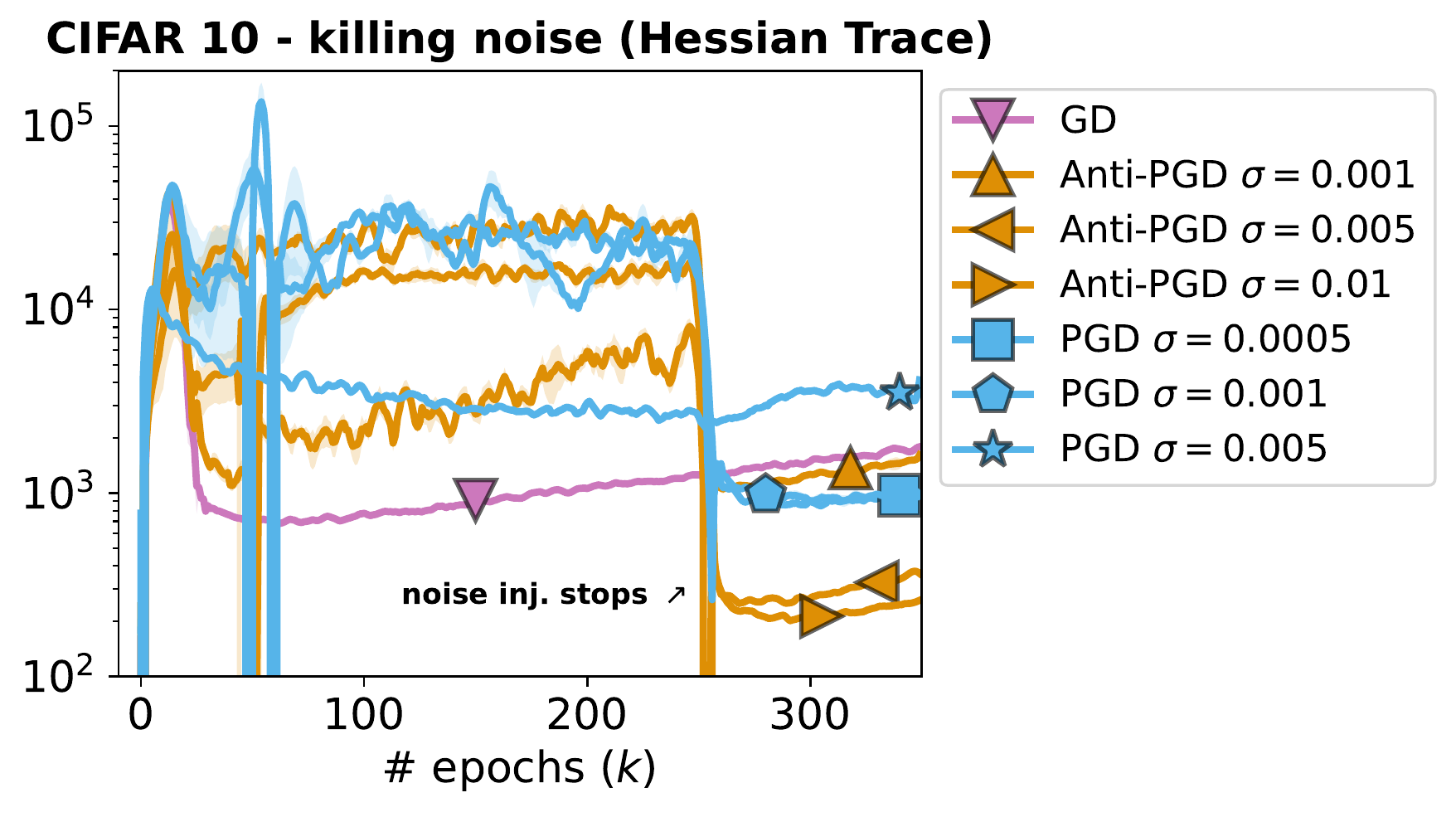}
    \includegraphics[height=0.22\textwidth]{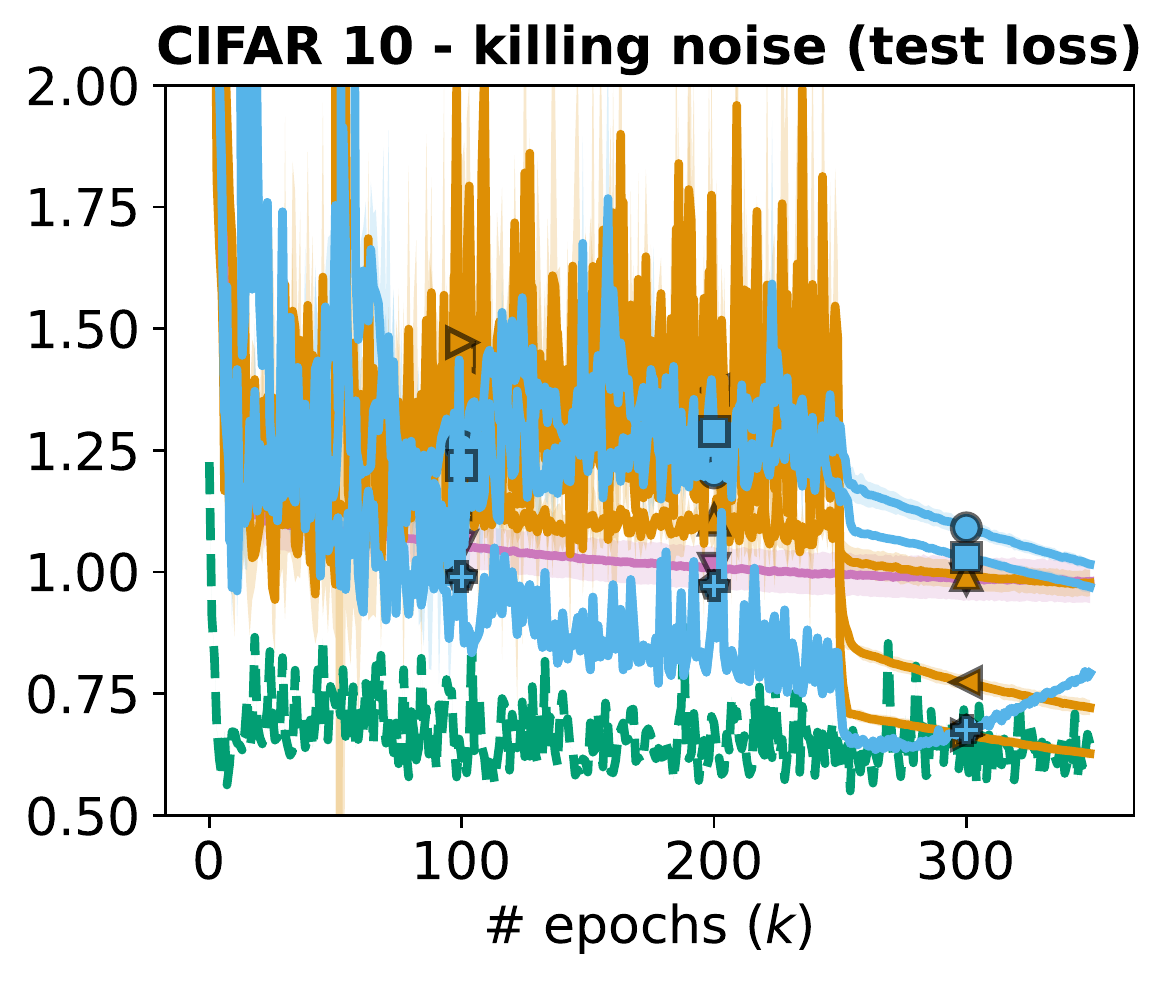}
    \includegraphics[height=0.22\textwidth]{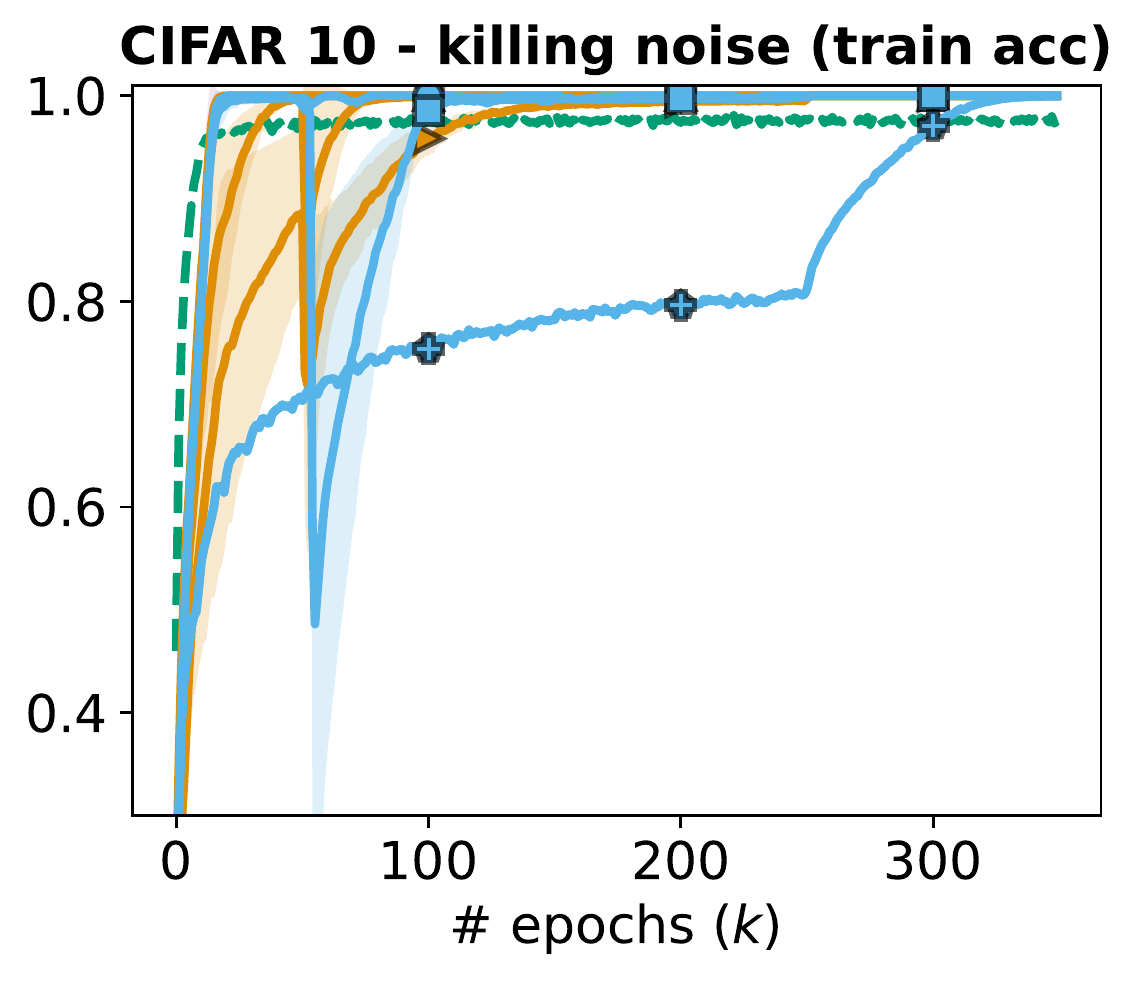}
    \includegraphics[height=0.22\textwidth]{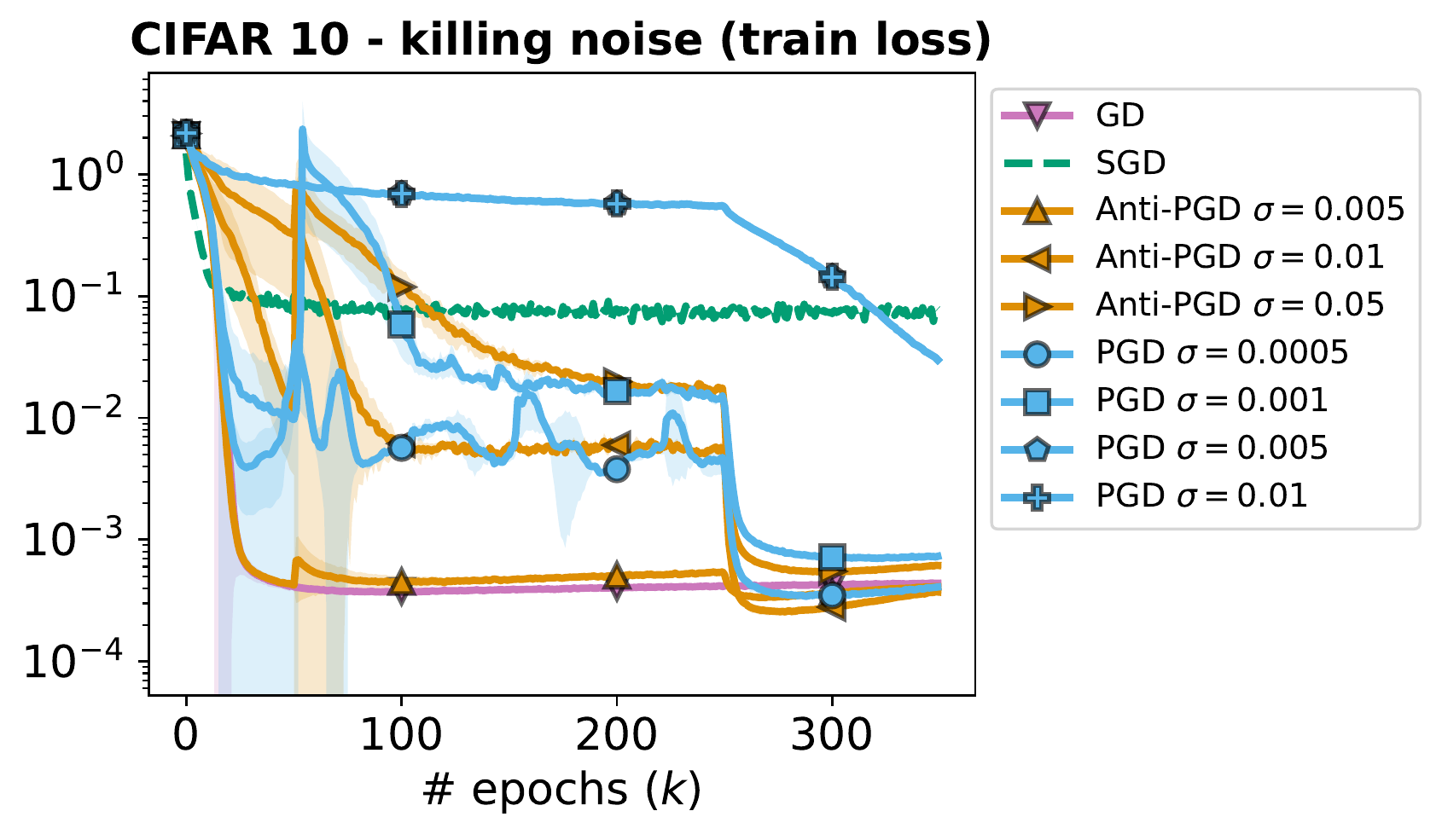}

    \vspace{-3mm}
    \caption{Additional plot for the experiment in Figure~\ref{fig:front_experiments}.}
\end{figure}
\vspace{-2mm}
\subsection{Performance of Anti-SGD}
We test the performance of anticorrelated noise when injected on top of mini-batch SGD. We consider the two non-toy settings of Fig.~\ref{fig:front_experiments} and report results in Fig.~\ref{fig:anti-sgd}. For matrix sensing, injecting anticorrelated noise to SGD gives a substantial improvement, with a slight edge over Anti-PGD. For the ResNet18 experiment, we compared batch-sizes of 128 or 1024: noise injection works best at moderate batch sizes. We hypothesize this is due to the high stochastic nature of SGD at low batch-sizes, which dominates over injected noise. Plots for the Hessian follow the same trend (highest test, lowest trace).

\vspace{-1mm}
\begin{figure}[ht]
    \centering
    \includegraphics[height=0.22\linewidth]{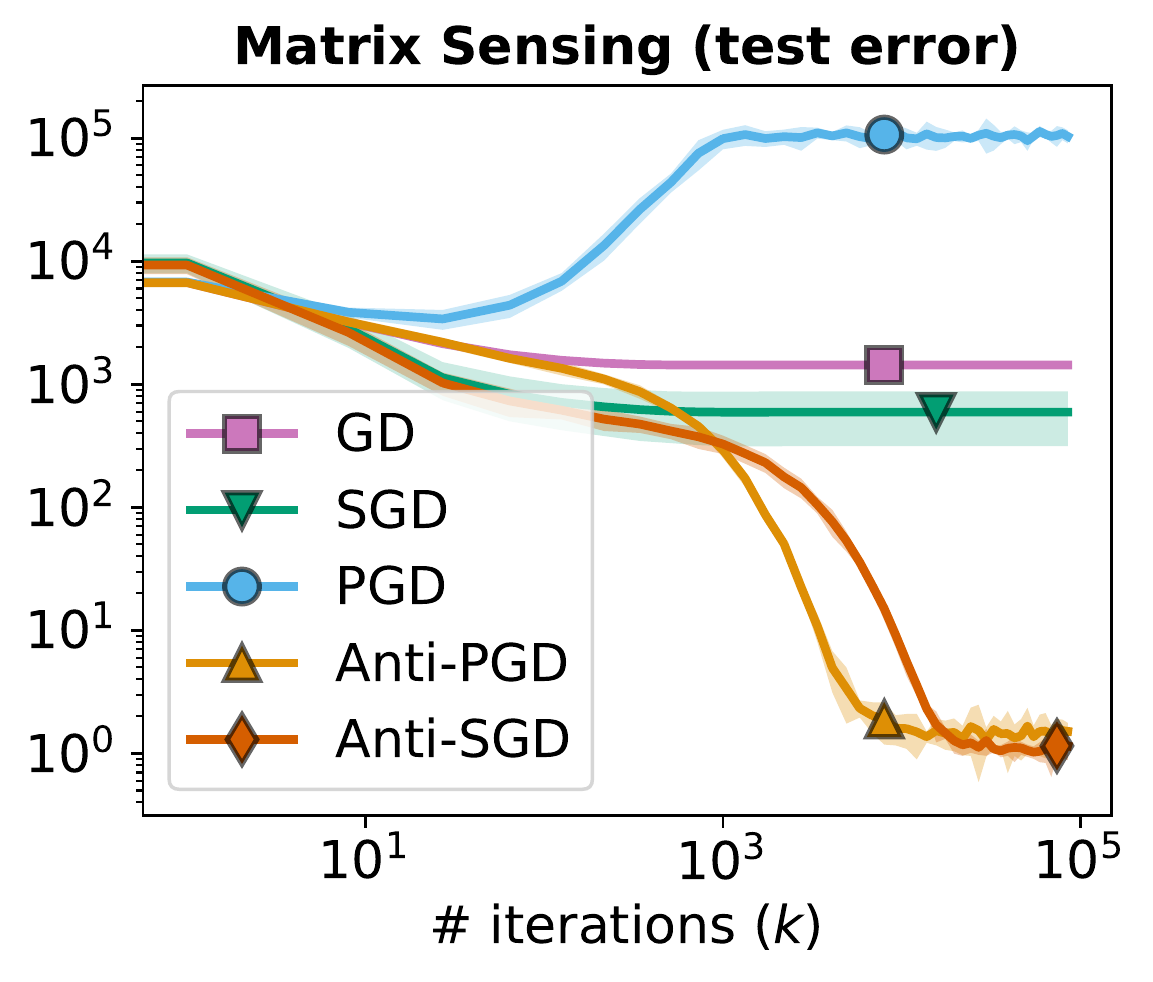}
    \includegraphics[height=0.22\linewidth]{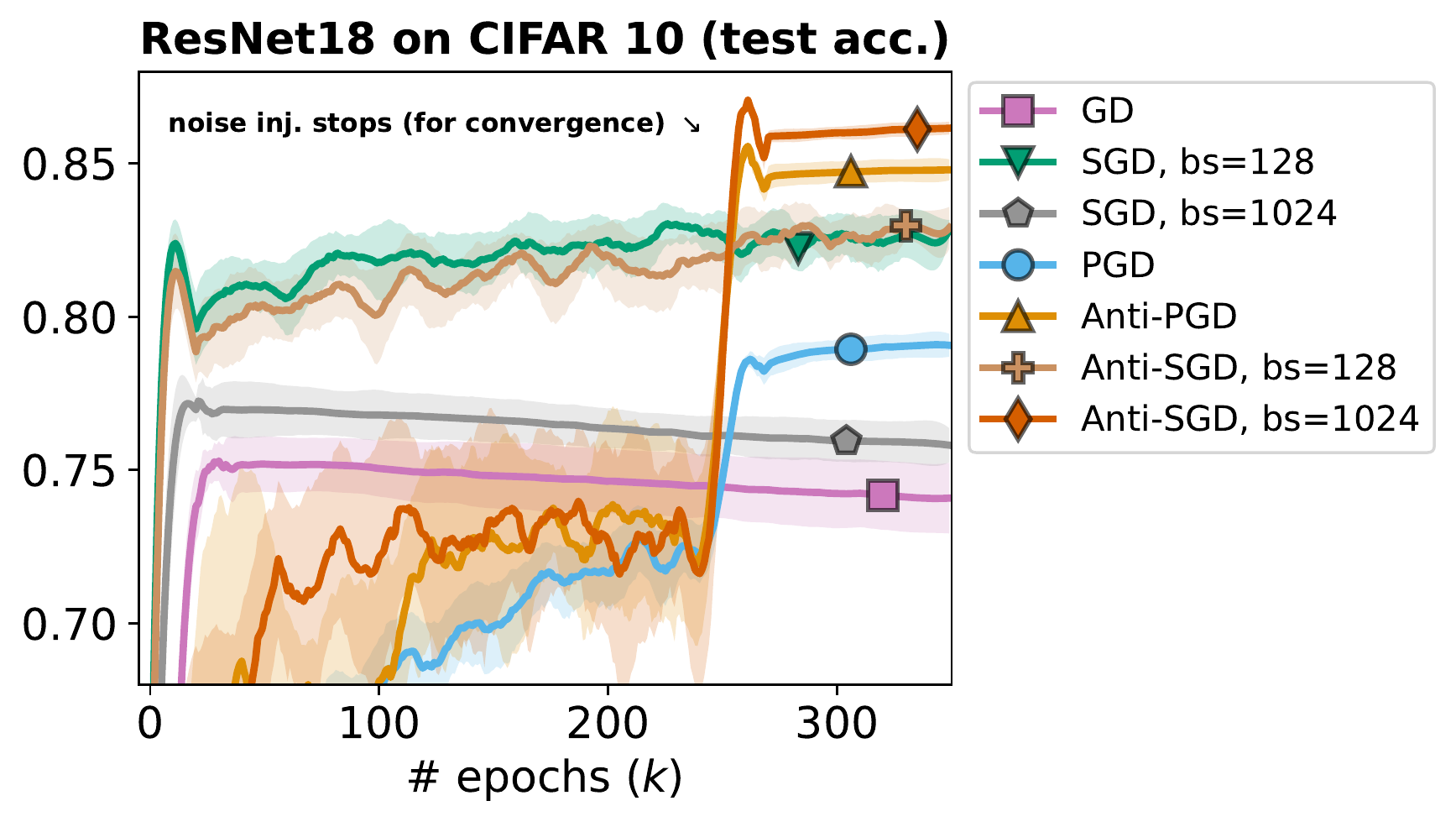}
\vspace{-3mm}
    \caption{\small Performance of Anti-SGD in the same settings as Figure~\ref{fig:front_experiments}}
    \label{fig:anti-sgd}
\end{figure}

\end{document}